\documentclass[10pt,twocolumn]{article}

\usepackage[T1]{fontenc}
\usepackage{mathptmx}
\usepackage[letterpaper,textwidth=177.53mm,textheight=672pt,columnsep=6.95mm,centering,headsep=12pt]{geometry}
\usepackage{amsmath,amssymb,amsthm}
\usepackage{graphicx}
\usepackage{booktabs}
\usepackage{multirow}
\usepackage{array}
\usepackage[table]{xcolor}
\usepackage{longtable}
\usepackage{algorithm}
\usepackage{algpseudocode}
\usepackage{hyperref}
\hypersetup{hidelinks}
\usepackage{ragged2e}
\usepackage{tabularx}
\usepackage{tabulary}
\usepackage{etoolbox}
\usepackage{placeins}
\usepackage{titlesec}
\usepackage{fancyhdr}

\renewcommand{\thesection}{\Roman{section}}

\titleformat{\section}{\normalfont\bfseries\normalsize\centering}{\thesection.}{0.5em}{\MakeUppercase}
\titleformat{\subsection}{\normalfont\itshape\normalsize}{\Alph{subsection}.}{0.5em}{}
\titleformat{\subsubsection}{\normalfont\itshape\normalsize}{\arabic{subsubsection})}{0.5em}{}
\titlespacing*{\section}{0pt}{10pt plus 2pt}{5pt plus 1pt}
\titlespacing*{\subsection}{0pt}{8pt plus 2pt}{3pt plus 1pt}
\titlespacing*{\subsubsection}{0pt}{6pt plus 2pt}{2pt plus 1pt}
\newcommand{\appendices}{\appendix\renewcommand{\thesection}{\Alph{section}}%
  \titleformat{\section}{\normalfont\bfseries\normalsize\centering}{APPENDIX \thesection}{0.5em}{\MakeUppercase}}
\newenvironment{keywords}{\par\noindent\textbf{Keywords:}\ }{\par}


\usepackage{needspace}
\let\oldsection\section
\renewcommand{\section}{\Needspace{5\baselineskip}\oldsection}
\let\oldsubsection\subsection
\renewcommand{\subsection}{\Needspace{4\baselineskip}\oldsubsection}
\let\oldsubsubsection\subsubsection
\renewcommand{\subsubsection}{\Needspace{4\baselineskip}\oldsubsubsection}

\newtheorem{lemma}{Lemma}
\newtheorem{proposition}{Proposition}
\newtheorem{corollary}{Corollary}
\newtheorem{assumption}{Assumption}
\newtheorem{definition}{Definition}
\newtheorem{remark}{Remark}

\newcommand{\citep}[1]{\cite{#1}}

\newcommand{\assets}{paper_assets}
\usepackage{refcount}
\newcommand{\apptables}[2]{\noindent{\footnotesize\textit{Tables of this appendix: \ifnum\getpagerefnumber{#1}=\getpagerefnumber{#2}\relax p.~\pageref{#1}\else pp.~\pageref{#1}--\pageref{#2}\fi.}}\par\smallskip}
\graphicspath{{\assets/}}

\newcommand{\xfigwdcol}{\xdef\xfigwd{\the\columnwidth}}
\newcommand{\xfigwdtext}{\xdef\xfigwd{\the\textwidth}}

\begin{document}

\twocolumn[
\begin{center}
{\LARGE\bfseries BRIDGE: An Acceptance Gate to Block False Improvements in Model Updates, with a Retrospective Clinical Evaluation\par}
\vspace{10pt}
{\large Dwi Sunaryono\par}
\vspace{3pt}
{\small Department of Informatics, Faculty of Intelligent Electrical and Informatics Technology,\\ Institut Teknologi Sepuluh Nopember, Surabaya 60111, Indonesia\\ Corresponding author: dwi@its.ac.id\par}
\vspace{4pt}
{\small Preprint, October 2026\par}
\end{center}
\vspace{4pt}
\begin{quote}\small
\noindent\textbf{Abstract.} Models are updated repeatedly, and an update that looks better on validation data but not on new data is a false improvement that can misguide clinical decisions. The naive rule accepts every positive validation gain and is optimistic when the same data choose and judge the update. BRIDGE is a model-agnostic acceptance gate that decides each update from predictions and the available labels, with fixed settings. Each fold of the data judges the update chosen on the other folds (cross-fitting) to reduce selection bias; acceptance requires the gain's lower bootstrap bound to exceed a margin. The four settings are (A) 11 new datasets, (B) synthetic tasks, (C) 12 retrospective clinical tasks, and (D) 39 supporting datasets. The false-improvement rate falls from 19.1\% (naive) to 1.1\% (BRIDGE) in (A), 20.7\% to 1.7\% (B), 20.0\% to 0\% (C), and 12.5\% to 0.5\% (D); every 95\% confidence interval of the reduction over datasets or tasks excludes zero (not over the clinical sources). An equal-size random draw from the naive acceptances contains more false improvements in (A)--(E) (post hoc, $p<0.001$). On the OpenML collection (E) of 14 new datasets, the false-improvement rate falls from 25.6\% to 1.0\%. BRIDGE removes 97\% to 100\% of the false improvements that the naive rule accepts, and its mean gain per deployed update is 1.8 to 3.1 times the naive rule's. The 95\% upper bound on the false-improvement rate is below 5.5\%: formal in (B) for independent tasks (3.56\%), and conditional on the real data tested (0.33\% to 5.40\%).

\vspace{4pt}
\end{quote}
\vspace{10pt}
]
\thispagestyle{plain}

\section{Introduction}
\label{sec:intro}

Machine-learning models are typically selected and approved on the basis of training and validation scores, but performance on those scores is often not reproduced on previously unseen data \cite{kaufman2012,varoquaux2022,recht2019,roberts2021}. Deployed models therefore frequently have to be revised \cite{sculley2015,paleyes2022}, and the changes in question include a threshold shift \cite{lipton2014}, an ensemble \cite{caruana2004}, feature selection \cite{guyon2003,ambroise2002}, and hyperparameter tuning \cite{bergstra2012,bischl2023}.

The naive validation-improvement rule accepts such a change whenever the validation score improves, during development or after deployment. Formally, let $b$ denote the baseline predictor, the model currently in use, and $c$ a candidate update. Let $\mathcal{M}$ be a performance metric evaluated on a validation set $\mathcal{D}_v$ that is held out from fitting, but not from the selection of the candidate. The update is accepted whenever the observed difference
\begin{equation}
\label{eq:acceptproblem}
\hat\Delta(c,b) \;=\; \mathcal{M}(c;\mathcal{D}_v) \;-\; \mathcal{M}(b;\mathcal{D}_v)
\end{equation}
is positive. The same validation data serve two purposes: choosing the candidate and then estimating the gain of the chosen candidate against the baseline. The twofold use spans classical model comparison \cite{varma2006,cawley2010,tibshirani2009,krstajic2014} and modern automated pipelines \cite{feurer2015,akiba2019}.

Because the validation data cannot judge the update chosen on that same data, in this study each accepted update is scored on separate data set aside from the start, the locked reference. The outcome is one of two: the change improves, when the selected metric is higher than the baseline metric, or the change is a false improvement (a harmful update in the pre-registration documents), when the metric is lower or equal. A tie is counted as a false improvement, because a change that brings no improvement still carries the cost and the risk of replacing the model (formalized in Definition~\ref{def:risk}). The risk quantity of interest \cite{feng2021biometrics} is the proportion of false improvements among those accepted.

The false-improvement rate among accepted updates is difficult to bound, because two sources make the validation gain optimistic. The first source of optimism is selection bias: a candidate is chosen as the best of many on the validation data and then judged on those same data, so its apparent improvement is larger than the true one. Maximizing over many candidates is a multiple-comparisons problem, and the resulting optimism is known as selection bias or as the winner's curse \cite{varma2006,cawley2010}. The bias grows with the number of candidates and the amount of tuning, as formalized in \eqref{eq:maxbias} (Section~\ref{sec:rel-selection}; \cite{ambroise2002,tibshirani2009}).

The second source of optimism is data leakage into the scores used for selection: Kapoor and Narayanan \cite{kapoor2023} found leakage in $294$ machine-learning papers across $17$ fields, including the use of test-set information in model or feature selection. The BRIDGE evaluation protocol closes that source: the locked reference is used only for judging, so leakage into the judging data is excluded by design and the optimism studied in this paper is that of selection.

Four directions of prior work address parts of this problem. First, estimating the performance of a selection procedure (nested cross-validation \cite{varma2006}, bias correction of the selected error rate \cite{tibshirani2009}) returns a performance estimate. Second, comparing two models or two training procedures \cite{mcnemar1947,dietterich1998,nadeau2003} assumes the compared pair is fixed before the results are seen. Third, governing a released model or an evaluation resource (medical-device approval policies \cite{feng2021biometrics,feng2021chil,fda2019}, no-regression certification \cite{zhang2026,payonly2026}, adaptive data analysis \cite{dwork2015,blumhardt2015}) needs data beyond the labels already available. Fourth, controlling the risk of a chosen policy (Learn then Test \cite{angelopoulos2021}) needs a separate calibration set.

\begin{figure}[t]
\centering
\xfigwdcol
\includegraphics[width=0.95\linewidth]{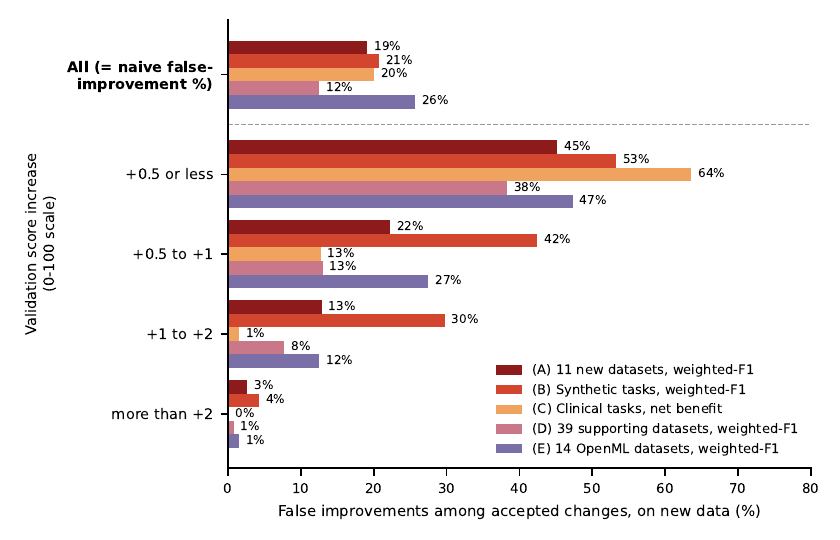}
\caption{False improvements accepted by the naive rule, grouped by validation gain, in collections (A)--(E).}
\label{fig:validation_gain}
\vspace{2pt}
\parbox{\linewidth}{\justifying\scriptsize \emph{Note.} The figure is descriptive. Current model: histogram gradient boosting (HGB) with the default candidate space (HGB, decision trees, and a support vector classifier); other learners, other candidate members, and a joint seven-family space (XGBoost, multilayer perceptron, logistic regression, naive Bayes, k-nearest neighbors) are tested separately (Tables~\ref{tab:members} and~\ref{tab:scope}). Data and comparison with BRIDGE: Sections~\ref{sec:exp-rules}--\ref{sec:exp-clinical}. Above $+5$ points none of the $726$ accepted changes is a false improvement in any setting, so that band is merged into ``more than $+2$''.}
\end{figure}

Despite these four directions, the naive rule remains in use in practice, and the danger is real. Across the five collections, (A) to (D) and the OpenML collection (E), $12.5\%$ to $25.6\%$ of the updates that the naive rule accepts are false improvements (Fig.~\ref{fig:validation_gain}, row ``All''). The danger concentrates where the evidence is weakest: for a validation gain of $+0.5$ points or less, $38\%$ to $64\%$ of the accepted changes are false improvements, whereas above $+2$ points $0\%$ to $4\%$ are (Fig.~\ref{fig:validation_gain}). The pattern is descriptive; it was seen after the study was designed, and it was not used to set the gate parameters, which were fixed beforehand.

If left alone, the problem accumulates, because each false improvement replaces a model that is already running, with its cost and its risk. Repeated updates cause a slow decline in performance, the effect called bio-creep \cite{feng2021biometrics,feng2021chil}, and add technical debt to the machine-learning system \cite{sculley2015,paleyes2022}. In the clinic, an update that makes a model look better in validation but not on new data is a false improvement, and a screening or diagnostic model updated on that basis can misguide decisions about patients \cite{roberts2021,varoquaux2022}. Regulators therefore ask for control over the changes made to a deployed model \cite{fda2019}.

The four directions leave two gaps, a methodological gap and an empirical gap. For the methodological gap, a fixed acceptance gate for one selected update combines cross-fitted selection, a bootstrap decision score, and a margin; its false-improvement rate is evaluated across tasks (Table~\ref{tab:related}). Honest holdouts and tests with re-selection also separate choosing from judging and serve as comparators, whereas the closest governance and calibration methods use further data \cite{feng2021biometrics,zhang2026} or bound the risk of a selected policy \cite{angelopoulos2021}. The empirical gap is the absence, in the literature reviewed in Section~\ref{sec:related}, of a large-scale pre-registered evaluation across change types, metrics, and domains, with a factorial component analysis of the acceptance gate.

Pre-registration means that the design, the hypotheses, and the analysis are fixed before the runs. The gap matters because model comparisons often depend on the seed and the data split \cite{bouthillier2021} and leakage and methodological failures are common \cite{kapoor2023,varoquaux2022}. In this study, pre-registration is internal: each plan was written and hashed before its run, without a third-party registry, so the order of plan and run rests on the recorded times (Appendix~\ref{app:prereg}). Pre-registration limits the freedom to choose the analysis after the results are seen \cite{nosek2018}, whereas seed variation is handled by drawing a new seed for every trial.

The paper therefore proposes BRIDGE, an acceptance gate that decides one update using only the labels already available. BRIDGE separates the data that choose an update from the data that judge it (cross-fitting) and accepts the update only when the lower bootstrap bound on its gain exceeds a minimum gain, the margin. Changes whose gain is small and uncertain---the group with the most false improvements in Fig.~\ref{fig:validation_gain}---are therefore rejected. On the collections tested (Section~\ref{sec:exp-settings}), the false-improvement rate among accepted updates falls to at most $1.7\%$ (Section~\ref{sec:exp-summary}).
 The contributions respond to the two gaps: C1 and C2 answer the methodological gap, C3 answers the empirical gap, and C4 makes both usable and checkable by others.
\begin{itemize}
\item \textbf{C1: BRIDGE, an acceptance gate with fixed settings.} BRIDGE blocks false improvements by accepting or rejecting one update after it has been chosen, using the current model's predictions, the candidate's predictions, and the labels already available, with no new data. What is new is the cross-fitting step, which separates the data that choose an update from the data that judge it, used together with the bootstrap lower bound and the margin. The settings are the seven frozen parameters (the fold count, the bootstrap count, the quantile, the margin, the bag count, the bag fraction, and the SEARCH size) and the candidate space. These settings are fixed once during development and used unchanged for every task, so no per-task tuning reintroduces selection bias. Binary and multiclass tasks use the same procedure (Section~\ref{sec:method}).
\item \textbf{C2: A finite-sample bound and a component analysis.} The finite-sample bound is Proposition~\ref{prop:bound}: a $95\%$ upper bound on the proportion of false improvements among those accepted, computed on the locked reference after each decision. The gate itself decides only from POOL, so the bound is an evaluation protocol. The bound applies to a finite number of tasks rather than only asymptotically, with the condition that the tasks are drawn independently from one population (Section~\ref{sec:method-risk}). The bound targets the false-improvement rate among accepted updates rather than a prediction error. The component analysis uses a factorial design over three components: the cross-fitted evidence, the bootstrap lower bound, and the margin. The design switches the three on and off in all eight combinations and also varies the size of the candidate space, to show which component lowers the false-improvement rate (Section~\ref{sec:exp-ablation}).
\item \textbf{C3: An internally pre-registered evaluation and a clinical application.} The evaluation is internally pre-registered, so the design, the pass criterion, the dataset list, and the seed rule are fixed before each run. Each file is recorded with a SHA-256 cryptographic hash \cite{nist2015} in a manifest that is kept in the code repository (chronology in Appendix~\ref{app:prereg}). A reader can therefore check that each reported result comes from the hashed code. All rules judge the same decisions in the four settings (A) to (D) (Section~\ref{sec:exp-settings}), each with a new seed and ground truth on the locked reference; the OpenML collection (E) is run through the released program as well. BRIDGE is compared with acceptance rules built on classical test statistics, an honest holdout, and a fixed margin. The conditions of use are tested for base learners, tuning, and imbalanced data. The clinical application applies the gate retrospectively to $12$ clinical tasks from $6$ sources, with net benefit \cite{vickers2006} as the decision metric (Section~\ref{sec:exp-clinical}).

\item \textbf{C4: A reproducible open-source implementation.} The open-source implementation is the package \texttt{bridge-gate} (\texttt{pip install bridge-gate}), which runs the paper's procedure as its default behavior and can be applied to a user's own data. The package reproduces every stored result row of $15$ runs, and the repository and the release archive are listed (Appendix~\ref{app:prereg}; Data Availability Statement).
\end{itemize}
\section{Literature Review}
\label{sec:related}

The related work is organized around three questions that match the three subsections: why validation-based model choice is biased, and how data leakage adds to that bias (Section~\ref{sec:rel-selection}); what classifier comparison tests assess (Section~\ref{sec:rel-tests}); and how prior work governs the approval of model updates (Section~\ref{sec:rel-gates}). Table~\ref{tab:related}, at the end of Section~\ref{sec:rel-gates}, compares the approaches closest to BRIDGE.

\subsection{Selection Bias and Data Leakage in Validation-Based Model Choice}
\label{sec:rel-selection}

Validation-based model choice is the common practice of selecting, from many candidates, the one with the highest validation score, and then reporting that score as evidence of improvement \cite{cawley2010}. This practice causes selection bias, although for a single model that is already fixed the validation score is unbiased. Once the best of many candidates is chosen, the score of the selected candidate is the highest of many noisy scores. On average, that score therefore exceeds the gain that the candidate would show on new data.

In symbols, writing $\hat\Delta(c,b)$ for the validation gain of candidate $c$ over the baseline $b$, with $\mathcal{C}$ the frozen candidate space (Section~\ref{sec:method-partition}),
\begin{equation}
\label{eq:maxbias}
\mathbb{E}\!\left[\max_{c\in\mathcal{C}}\hat\Delta(c,b)\right]
\;\ge\;
\max_{c\in\mathcal{C}}\mathbb{E}\!\left[\hat\Delta(c,b)\right],
\end{equation}
where all expectations are over the random validation-data draw. The gap between the two sides of \eqref{eq:maxbias} can grow with the number of candidates and the amount of tuning, so hyperparameter search itself becomes a form of fitting \cite{cawley2010}. Beyond this selection bias, a reported difference can also be lost in seed and partitioning variability, which replication studies find can be as large as the difference itself \cite{bouthillier2021}.

Three approaches reduce selection bias: (i) nested cross-validation \cite{varma2006} and bias correction of the minimum cross-validation error \cite{tibshirani2009} estimate the performance of the selection procedure rather than of the model deployed \cite{bates2024}; (ii) post-selection inference accounts for the mechanism that selected the object \cite{berk2013,taylor2015}; and (iii) cross-fitting separates the data used to choose from the data used to judge \cite{chernozhukov2018}. BRIDGE uses this principle within each round of the gate (Section~\ref{sec:method-gate}) and returns an accept/reject decision for one update.

Besides selection bias, the second source of optimism noted in Section~\ref{sec:intro} is data leakage. Group leakage, where rows from the same group or subject cross the split, and temporal leakage, where information from a future time enters an earlier stage, both make a validation score too optimistic \cite{saeb2017,arp2022,kapoor2023}. The paper closes that source with an explicit group and temporal audit before any split (Section~\ref{sec:exp-settings}).

\subsection{Statistical Tests for Comparing Classifiers}
\label{sec:rel-tests}

Statistical tests are the standard way to compare classifiers, and the tests differ in the objects being compared: (i) McNemar's test \cite{mcnemar1947} compares the predictions of two models that are already trained, on the same observations; (ii) two training procedures are compared by refitting on several data splits, using the $5\times2$ cross-validated (cv) paired $t$-test \cite{dietterich1998} or the corrected resampled $t$-test \cite{nadeau2003}; and (iii) two methods are compared across datasets by rank-based \cite{demsar2006} or Bayesian tests \cite{benavoli2017}. The first two kinds assume that the compared pair is fixed before the results are inspected. Neither therefore assesses the decision that follows when the best configuration is chosen among many on the same data. The third kind answers which method is better, not whether a single update should be accepted.

\subsection{Gates for Approving Model Updates}
\label{sec:rel-gates}

The work most directly related to this paper concerns gates, that is, rules that accept or reject a proposed model update before the update is used. Three approaches build such gates for model updates, and each approach has its own data needs. First, medical-device approval policies \cite{feng2021biometrics,feng2021chil} formalize bio-creep, the risk that approving too many changes slowly degrades performance, and test online on a single deployment stream. Second, slice-wise no-regression certification \cite{zhang2026} evaluates a candidate and its baseline on independent certification data, with fallback to the current model. Third, disagreement labeling \cite{payonly2026} uses unlabeled deployment data and requests labels only on the items where the candidate and baseline differ. All three need data beyond the labels already available, whereas BRIDGE uses only the labels already available.

Calibration methods also need a separate calibration set and work with policies, where a policy is a model together with its decision rule, for example a threshold. Learn then Test \cite{angelopoulos2021} uses the calibration set to certify that the risk of a chosen policy stays below a target level, with a multiple-testing procedure that controls the family-wise error across the candidate policies. Conformal risk control \cite{angelopoulos2024} uses it to calibrate a single threshold so that the expected loss of the resulting policy is bounded. BRIDGE instead answers a paired question, whether this update improves on the current model, from the cross-fit evidence (Section~\ref{sec:method-gate}).

Three further approaches answer other questions: (i) reusable-holdout and leaderboard mechanisms \cite{dwork2015,blumhardt2015} guard an evaluation resource reused across many queries rather than deciding one update; (ii) clinical model updating, such as recalibration and revision \cite{janssen2008,steyerberg2019}, changes the model to fit new data without an accept/reject gate; and (iii) sharpness-aware minimization \cite{foret2021}, invariant risk minimization \cite{arjovsky2019}, and group distributionally robust optimization \cite{sagawa2020} change how models are trained, not whether a post-training change is applied.

This paper asks whether a single update, chosen from many candidates, may be accepted using only the labels already available, and it evaluates how often the accepted updates are false improvements with an upper bound computed on a locked reference over many decisions. BRIDGE answers by separating the data that choose from the data that judge (Section~\ref{sec:method}). Table~\ref{tab:related} compares the closest approaches on what the method controls, whether it accounts for selection, what labels the decision needs, and what the bound covers.

Relative to the methods in Table~\ref{tab:related}, the closest comparison is with the calibration methods. Learn then Test and conformal risk control can take the calibration set from the same label budget, but these methods bound the risk of a selected policy, whereas the evaluation of BRIDGE bounds the share of false improvements over many decisions. The comparison with Learn then Test is conceptual, and no empirical advantage is claimed. BRIDGE chooses an update on one part of the data and judges it on another, and it uses the training labels that already exist. Its bound covers the acceptance procedure over a population of tasks rather than one deployed model.

\begin{table*}[!t]
\centering
\scriptsize
\setlength{\tabcolsep}{3pt}
\caption{Positioning of BRIDGE against related gates and methods.}
\label{tab:related}
\rowcolors{0}{}{}
\setlength{\tymin}{28pt}\begin{tabulary}{\linewidth}{@{}LLlLL@{}}
\toprule
Method & What it controls & \begin{tabular}[b]{@{}l@{}}Accounts for\\selection\end{tabular} & \begin{tabular}[b]{@{}l@{}}Labels needed\\at decision\end{tabular} & \begin{tabular}[b]{@{}l@{}}Bound\\covers\end{tabular} \\
\midrule
Nested CV and bias corrections \cite{varma2006,cawley2010,bates2024} & error-estimate optimism & yes (inside estimate) & existing labels only & selection procedure \\
McNemar / $5\times2$~cv / corrected $t$ \cite{mcnemar1947,dietterich1998,nadeau2003} & fixed-pair difference & no & existing labels only & prespecified comparison \\
Learn then Test \cite{angelopoulos2021} & policy risk & yes & reserved calibration set & selected policy \\
Update approval \cite{feng2021biometrics,feng2021chil} & bio-creep & yes (within policy) & labeled deployment stream & approval policy \\
No-regression certificates \cite{zhang2026,payonly2026} & per-slice regression & yes (within slices) & certification set (Zhang); disagreements only (Balachandran) & deployed model \\
BRIDGE & false-improvement rate among accepted updates & \textbf{yes} (per-fold separation) & \textbf{existing training labels only} & acceptance procedure, population-level \\
\bottomrule
\end{tabulary}

\vspace{2pt}
\parbox{\linewidth}{\justifying\scriptsize *Column meanings and reading: Section~\ref{sec:rel-gates}; empirical comparison with an honest holdout: Table~\ref{tab:holdout}.}
\end{table*}

\section{The BRIDGE Procedure and Experimental Design}
\label{sec:method}

The BRIDGE gate (a label, not an acronym) turns a candidate update into an accept/reject decision from data that are already labeled. The gate operates within a complete procedure (Fig.~\ref{fig:partition}) that runs from the split of the data to the deployment decision.

The proposed procedure (Fig.~\ref{fig:partition}) has three parts: the fixed settings, the outer layer, and the inner layer. The fixed settings were chosen during development, before the confirmatory runs. These settings comprise the model type, the candidate space, and the gate parameters (Table~\ref{tab:hyper}). The outer layer splits the data of a task once into the locked reference and TRAIN. The locked reference is used solely to score the experiments, and the rows of TRAIN are re-drawn in every trial. TRAIN has $2{,}430$ rows, or $1{,}000$ on small datasets, and is split into SEARCH ($630$ rows) and POOL ($1{,}800$ rows, three folds of $600$). The base models are fitted on SEARCH, and an update is selected and judged on POOL. Table~\ref{tab:hyperbasis} (Appendix~\ref{app:sens}) gives the basis for these sizes.

In the inner layer, the cross-fit, POOL is split into three folds, and in each round a candidate is selected on two folds and judged on the third. The pooled judgments decide whether the update is accepted. An accepted update is re-selected on the whole of POOL and deployed; otherwise the baseline is retained. Every judgment therefore rests on data that were not used to fit or to select the candidate being judged.

Each component of the gate addresses one source of error: (i) selection bias, which makes the gain of the best candidate optimistic when the data that chose it also judge it \eqref{eq:maxbias}, so cross-fitting judges each candidate on a fold that did not choose it; (ii) sampling noise, which can make a gain positive by chance, so the gate uses the bootstrap lower bound; and (iii) a gain too small to justify replacing the running model, so the margin sets the minimum improvement. The false-improvement rate is then evaluated with the finite-sample bound (Proposition~\ref{prop:bound}) and the KL-ratio bound on the locked reference (Section~\ref{sec:method-risk}).

Sections~\ref{sec:method-partition}--\ref{sec:method-risk} follow the figure: Section~\ref{sec:method-partition} presents the partition and the candidate space, and Section~\ref{sec:method-gate} presents the cross-fit and the bootstrap gate. Section~\ref{sec:method-metrics} presents the metrics and multiclass tasks, and Section~\ref{sec:method-risk} presents the upper bound on the false-improvement rate. Section~\ref{sec:exp-settings} then describes the experimental design, namely the data, the partition into trials, the comparators, and the evaluation metrics. Tables~\ref{tab:notation}, \ref{tab:gate}, and~\ref{tab:outcomes} list the notation, and Algorithms~\ref{alg:gate}--\ref{alg:crossfit} give the procedure step by step.

\subsection{Data Partition and Candidate Space}
\label{sec:method-partition}

The data partition, the outer layer of Fig.~\ref{fig:partition}, begins with the labeled examples of a task, $\{(x_i,y_i)\}_{i=1}^{N}$ with $y_i\in\{1,\dots,K\}$ for $K$ classes. A disjoint subset, the \emph{locked reference}, is set aside once and used only to score the experiments; it supplies the ground truth at evaluation time. The rows that remain once the reference has been set aside are eligible for training. Each trial samples its training rows and splits the rows into a \emph{SEARCH} set and a \emph{POOL} set with a fresh seed (Section~\ref{sec:exp-settings}),
\begin{equation}
\label{eq:partition}
\text{TRAIN} \;=\; \text{SEARCH} \;\sqcup\; \text{POOL},
\end{equation}
where SEARCH fits the base learners and POOL carries the evidence used both to select an update and to gate it (notation in Table~\ref{tab:notation}). Base-learner parameters and preprocessing are fitted only on SEARCH. On POOL nothing is fitted: the gate selects one threshold-and-weight configuration from the prespecified candidate space.

\begin{table}[!t]
\centering
\scriptsize
\providecolor{bghead}{HTML}{DBEFF8}\providecolor{bgzebra}{HTML}{F2F9FC}\providecolor{bgpass}{HTML}{D7F0D2}\providecolor{bgfail}{HTML}{FCDCD7}
\setlength{\tabcolsep}{3pt}
\caption{Frozen parameters of BRIDGE and of the evaluation protocol.}
\label{tab:hyper}
\setlength{\tymin}{28pt}\begin{tabulary}{\columnwidth}{@{}lL@{}}
\toprule
Parameter & Value \\
\midrule
Cross-fit folds $F$ & $3$ \\
Bootstrap replicates $N_{\mathrm{boot}}$ & $20{,}000$ \\
Lower quantile $q$ & $5\%$ \\
Margin $\varepsilon$ & $0.005$ \\
Base learner & HGB, library defaults with at most 300 iterations and early stopping \\
Bagged base models & $32$ bags of $80\%$ of SEARCH \\
Ensemble members & DT and SVC, library settings \\
Preprocessing & fitted on SEARCH only (WiDS exception: Table~\ref{tab:leakage}) \\
Candidate space & $173$ (binary); $14+8K$ for $K$ classes \\
Training rows & $2{,}430$ (SEARCH $630$, POOL $1{,}800$); $1{,}000$ when fewer rows \\
Locked reference & one pre-registered seed per task; size by a fixed rule \\
Trials & $100$ per dataset, $200$ per clinical task, $60$ per multiclass dataset, $1$ per task of $\mathcal{Q}$; one fresh seed per trial \\
\bottomrule
\end{tabulary}

\vspace{2pt}
\parbox{\linewidth}{\justifying\scriptsize *Details of every value and its basis (where it was set and where it was tested): Table~\ref{tab:hyperbasis} in Appendix~\ref{app:sens}.}
\end{table}

\begin{table*}[!t]
\centering\footnotesize
\setlength{\tabcolsep}{4pt}\renewcommand{\arraystretch}{1.1}
\caption{Data units, partition, and predictors.}
\label{tab:notation}
\begin{tabularx}{\linewidth}{@{}p{2.6cm}>{\raggedright\arraybackslash}X@{}}
\toprule
\textbf{Symbol / term} & \textbf{Definition} \\ \midrule
Task & One dataset, or one draw from the synthetic population $\mathcal{Q}$. \\
Trial & One re-randomization of the training rows, the SEARCH/POOL split, and the bag draw under a fresh seed; produces one decision; trials per task: $100$ (real), $200$ (clinical), $60$ (multiclass), $1$ ($\mathcal{Q}$). \\
Decision; index $j$ & Accept or reject in one trial; $j$ runs over trials on real data and over tasks on $\mathcal{Q}$; $J$ is the number of trials or tasks. \\
Source & Clinical: data provider (six sources; the four National Health and Nutrition Examination Survey (NHANES) tasks form one source, the four Behavioral Risk Factor Surveillance System (BRFSS) tasks another). Non-clinical: the dataset. \\
\rowcolor{gray!10} Locked reference $R_j$ & Rows set aside once per task with a pre-registered seed; used only as the ground truth to judge the deployed update. Size: Appendix~\ref{app:datasets}. \\
\rowcolor{gray!10} TRAIN $=$ SEARCH $\sqcup$ POOL & SEARCH fits the base models; POOL holds the evidence for selection and gating. Size: Appendix~\ref{app:datasets}. \\
\rowcolor{gray!10} Development datasets & The $18$ datasets on which $\mathcal{C}$ and the gate were frozen; a subset of the $34$ confirmation datasets in (D). \\
\rowcolor{gray!10} New data (unseen during development) & Data held out from designing and freezing the gate (Section~\ref{sec:exp-settings}). \\
\rowcolor{gray!20} $b$ (baseline) & Default base learner (histogram gradient-boosted classifier, HGB) averaged over the kept bags at threshold $0.5$ (Section~\ref{sec:method-partition}). \\
\rowcolor{gray!20} $\bar p_\ell(x)$ & Mean predicted probability of base model $\ell$ over its retained bags ($80\%$ subsamples of SEARCH); bag-eligibility rule: Section~\ref{sec:method-partition}. \\
\rowcolor{gray!20} $c$ (candidate), $\mathcal{C}$ & A configuration $f_{w,t}$; $\mathcal{C}$ is the $173$ configurations of the binary default space (threshold-only, ensembles with decision trees (DT), ensembles with DT and a support vector classifier (SVC)); other protocols use separate prespecified spaces; basis: Table~\ref{tab:hyperbasis}. \\
\rowcolor{gray!20} Update & A proposed model change: a candidate from $\mathcal{C}$ selected on POOL that the gate may accept or reject; accepted update: the deployed configuration, re-selected on the whole of POOL. \\
\rowcolor{gray!30} Frozen & Gate parameters ($F$, $N_{\mathrm{boot}}$, $q$, $\varepsilon$, bag count) are fixed before any evaluation run and never tuned per dataset; values: Table~\ref{tab:hyper}. \\
\bottomrule
\end{tabularx}

\vspace{2pt}
\parbox{\linewidth}{\justifying\scriptsize *Shading marks the four groups in order: \emph{data units} (white), \emph{data partition} (light gray), \emph{models and updates} (gray), and \emph{protocol} (darker gray).}
\end{table*}
\rowcolors{0}{}{}

\paragraph{Base learners.}
The default base learner is a histogram gradient-boosted classifier (HGB) \cite{friedman2001,ke2017,pedregosa2011} with the settings fixed before any evaluation and listed in Table~\ref{tab:hyperbasis}. The base learner uses a fixed project configuration: the library defaults, except at most $300$ boosting iterations with early stopping after $10$ iterations without improvement (Table~\ref{tab:hyperbasis}). This fixed learner separates the effect of the gate from the effect of tuning. Tuning is treated separately as an update that can be proposed to the gate (Section~\ref{sec:exp-base}).

All base models are bagged to reduce the variability of the predicted probabilities. Each bag is an $80\%$ subsample of SEARCH, and $32$ bags are drawn \cite{breiman1996,buhlmann2002} (the basis and the tests of $16$ and $64$ bags are in Table~\ref{tab:hyperbasis}, Appendix~\ref{app:sens}). A bag is retained provided that every class has enough examples for the internal validation split of the base model. The bagged score $\bar p_\ell(x)$ of base model $\ell$ is the average over the retained bags. Appendix~\ref{app:perdataset} gives the minimum class count per bag and the few trials skipped under this rule.

\paragraph{Baseline and candidates.}
The \emph{baseline} $b$ is the base learner alone at the fixed threshold $0.5$, the usual default of a probabilistic classifier. A \emph{candidate} $c$ is a weighted, thresholded combination of $L$ base models,
\begin{equation}
\label{eq:predictor}
f_{w,t}(x) \;=\; \mathbb{1}\!\left[\textstyle\sum_{\ell=1}^{L} w_\ell\,\bar p_\ell(x) \;\ge\; t\right],
\end{equation}
where $w$ lies on the simplex ($w_\ell\ge 0$, $\sum_\ell w_\ell=1$) and $t\in(0,1)$ is the decision threshold; the baseline is recovered at $L=1$, $w_1=1$, $t=0.5$.

In the binary default, $\mathcal{C}$ contains three families: other thresholds for the HGB base learner; mixtures of HGB with decision trees (DT) \cite{breiman1984}; and mixtures of HGB, DT, and a support vector classifier (SVC) \cite{cortes1995}. Each family is combined with a grid of thresholds, because moving the decision threshold is a standard way to trade errors between the classes \cite{lipton2014}. Moving the threshold also handles class imbalance without reweighting the training data. The mixtures follow weighted ensemble averaging \cite{caruana2004}. The grids are coarse by design to keep the space small, because the selection bias of \eqref{eq:maxbias} grows with the number of candidates.

The default space has $173$ configurations. Appendix~\ref{app:sens} gives the grids, the count, and the development history of the space, and Table~\ref{tab:hyperbasis} gives the basis of each setting. Section~\ref{sec:exp-ablation} tests whether this coarse choice matters. In that section, the same grids are refined to $1{,}825$ and $11{,}249$ configurations (grids in Table~\ref{tab:kspace}). The multiclass space follows the same idea, with class scales in place of thresholds. Section~\ref{sec:method-metrics} describes this space together with the multiclass metrics.

The three members, HGB, DT, and SVC, were fixed stepwise during development and represent three ways of learning: boosted trees, a single tree, and a kernel. During development, SVC replaced an earlier k-nearest neighbors (KNN) member \cite{cover1967}; Table~\ref{tab:devruns} (Appendix~\ref{app:prereg}) lists the development runs behind this choice. The base model throughout is the HGB classifier; DT and SVC enter only as candidate members. Two other member pairs are evaluated as conditions of use in Section~\ref{sec:exp-base}: logistic regression \cite{cox1958} with naive Bayes \cite{hand2001} and a multilayer perceptron (MLP) \cite{rumelhart1986} with KNN. A joint space of seven member families (DT, SVC, logistic regression, naive Bayes, MLP, KNN, and XGBoost) is evaluated there as well. XGBoost \cite{chen2016} and an MLP are also evaluated there as alternative base learners. No member is tuned, and Appendix~\ref{app:sens} gives the member settings.

With the space fixed in this way, the only adaptation per task on POOL is the choice of one configuration from $\mathcal{C}$, or no change. The space and the gate parameters were frozen during development on $18$ datasets (dataset list in Table~\ref{tab:splits}) before any evaluation. The confirmation datasets of (D), including those $18$, are therefore re-evaluated with fresh per-trial seeds, whereas (A) contains task families not used during development (Section~\ref{sec:exp-settings}).

\begin{figure*}[t]
\centering
\xfigwdtext
\includegraphics[width=0.95\textwidth]{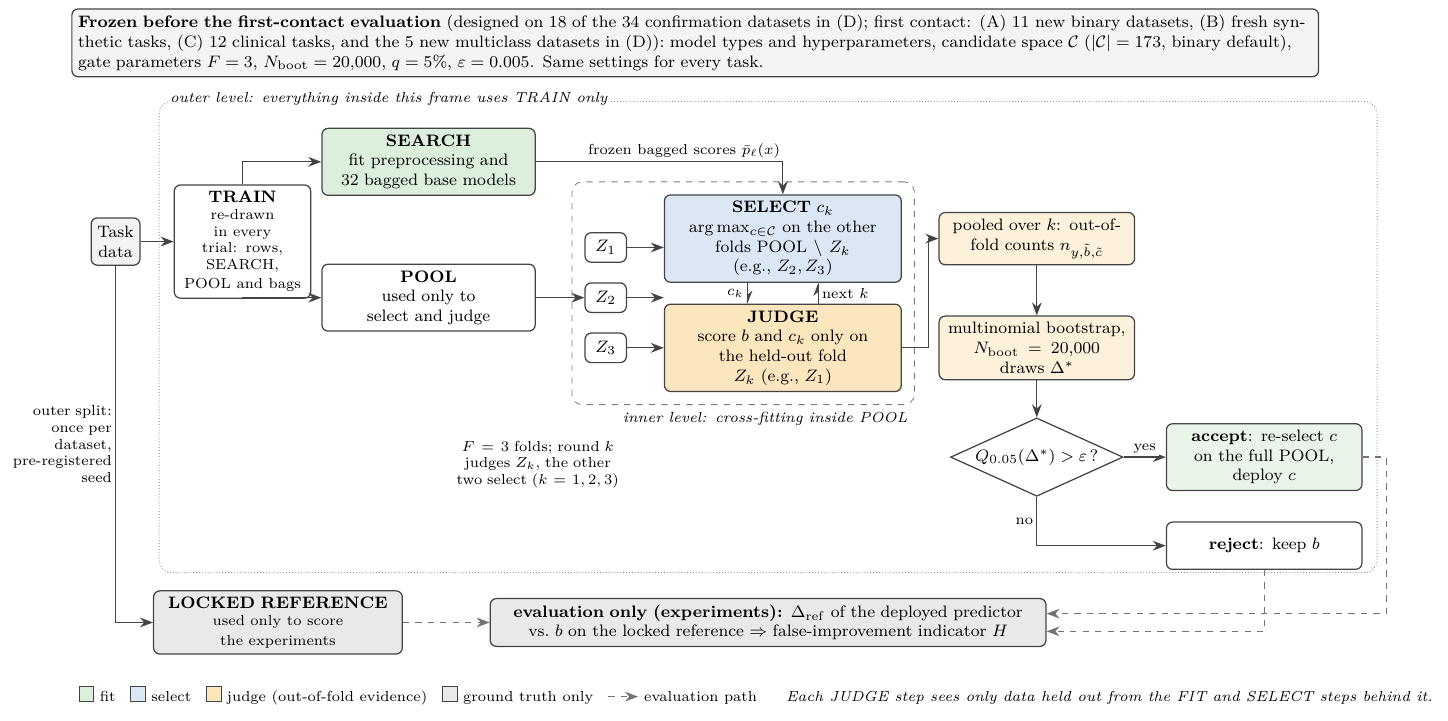}
\caption{Two-level separation: outer train/reference split, inner three-fold cross-fit within POOL; colors mark subset roles.}
\label{fig:partition}
\end{figure*}

\subsection{Cross-Fitted Selection and Bootstrap Gate}
\label{sec:method-gate}

\paragraph{Cross-fitted selection.}
POOL is partitioned into $F=3$ stratified folds $Z_1,\dots,Z_F$, with other values of $F$ tested in Appendix~\ref{app:sens}. In round $k$, the fold $Z_k$ is the scoring fold and its complement $\text{POOL}\setminus Z_k$ is the selection data. A candidate $c_k$ is chosen on $\text{POOL}\setminus Z_k$ by maximizing the metric of interest subject to strictly improving on the baseline, and the baseline and $c_k$ are then evaluated on $Z_k$. When no candidate strictly improves on $\text{POOL}\setminus Z_k$, $c_k$ is the baseline $b$ (the fallback). Algorithm~\ref{alg:gate} gives the steps of the gate. Algorithm~\ref{alg:crossfit} gives the selection rule, including the tie rule.

Repeating over the $F$ folds and pooling the held-out predictions yields the cross-fitted triple counts
\begin{equation}
\label{eq:counts}
n_{y,\tilde b,\tilde c} \;=\; \bigl|\{\, i \in \text{POOL} : y_i = y,\; \hat b_i = \tilde b,\; \hat c_i = \tilde c \,\}\bigr| ,
\end{equation}
where $y$, $\tilde b$, and $\tilde c$ range over the classes, and $\hat b_i$ and $\hat c_i$ are the predicted labels of the baseline and of the candidate. The label $\hat c_i$ records the prediction of the candidate $c_k$ that was selected for the fold containing observation $i$. The pooled counts therefore combine the predictions of the per-fold candidates, which need not be the same configuration. Section~\ref{sec:method-metrics} gives the coding of the labels for binary and multiclass tasks.

The baseline confusion table $n^{\mathrm{base}}_{y,\hat y}=\sum_{\tilde c} n_{y,\hat y,\tilde c}$ and the candidate confusion table $n^{\mathrm{cand}}_{y,\hat y}=\sum_{\tilde b} n_{y,\tilde b,\hat y}$ are both derived from these counts. A metric expressible in these counts is computed from the two tables, and the metric difference is defined as
\begin{equation}
\label{eq:deltametric}
\Delta \;=\; \mathcal{M}\!\left(\bigl\{n^{\mathrm{cand}}_{y,\hat y}\bigr\}\right) \;-\; \mathcal{M}\!\left(\bigl\{n^{\mathrm{base}}_{y,\hat y}\bigr\}\right).
\end{equation} Because the configuration evaluated on fold $k$ was selected without fold $k$, the counts are out-of-sample with respect to selection. The separation is made explicit in Table~\ref{tab:leakage} (Appendix~\ref{app:sens}): within the cross-fit, the per-fold candidate is assessed on the fold that was not used to select it.

\paragraph{Bootstrap gate.} The gate resamples the pooled counts $N_{\mathrm{boot}}=20{,}000$ times from the multinomial distribution defined by the observed proportions and the observed total. The resamples form a bootstrap distribution $\{\Delta^{*}_{(r)}\}_{r=1}^{N_{\mathrm{boot}}}$ \cite{efron1993}; Table~\ref{tab:gate} lists the quantities of the gate. The resampling does not repeat the candidate selection, so the lower bootstrap quantile $Q_{0.05}(\Delta^{*})$, at the quantile level $q=0.05$, is a decision statistic rather than a confidence bound. BRIDGE accepts the update exactly when that lower quantile exceeds a fixed margin,
\begin{equation}
\label{eq:accept}
\text{accept} \iff Q_{0.05}\!\left(\Delta^{*}\right) \;>\; \varepsilon, \qquad \varepsilon = 0.005,
\end{equation}
and otherwise retains the baseline.

The five gate parameters of \eqref{eq:counts}--\eqref{eq:accept} have been fixed since development and are the same for every dataset: $F=3$, $N_{\mathrm{boot}}=20{,}000$, quantile level $5\%$, margin $\varepsilon=0.005$, and bag count $32$ (Table~\ref{tab:hyper}); the basis for each value is given in Table~\ref{tab:hyperbasis} (Appendix~\ref{app:sens}). Together with the bag fraction and the SEARCH size, these values make up the seven frozen parameters (Section~\ref{sec:limitations}). The sensitivity of the criterion to these values is reported in Appendix~\ref{app:sens}. The margin $\varepsilon=0.005$ is the operational minimum improvement fixed during development ($0.5$ percentage points of weighted $F_1$, $0.005$ on the $0$--$1$ scale), and its basis is also given in Table~\ref{tab:hyperbasis}. For clinical decisions, benefit is assessed separately with net benefit in (C).

Once the gate accepts, the configuration to deploy is re-selected on the whole of POOL, so the deployed configuration may differ from the per-round candidates on which the decision was based. The bootstrap quantile on the cross-fit evidence is an empirical decision score, not an interval for the final configuration. Whether that score transfers to the final configuration is assessed empirically on the locked reference, through the false-improvement rate of the deployed configuration and Fig.~\ref{fig:crossfit}. In that figure, the cross-fit estimate of the final configuration's gain is slightly pessimistic, by $0.0034$ on average. Among the updates that BRIDGE accepts, the reference gain of the final configuration reaches the lower bound in $80\%$ to $96\%$ of cases. In the cases that fall short, the mean shortfall is $0.0026$ to $0.0089$, and most of these updates still improve on the reference (post hoc; Appendix~\ref{app:results}). The false-improvement rate and its bounds are then measured for that deployed configuration on the locked reference (Definition~\ref{def:risk}). The two layers of the design resemble nested cross-validation, and Remark~\ref{rem:nestedcv} states how the two differ.

\begin{table*}[!t]
\centering\footnotesize
\setlength{\tabcolsep}{4pt}\renewcommand{\arraystretch}{1.1}
\caption{Quantities of the bootstrap gate.}
\label{tab:gate}
\rowcolors{0}{}{}
\setlength{\tymin}{28pt}\begin{tabulary}{\linewidth}{@{}lL@{}}
\toprule\textbf{Symbol / term} & \textbf{Definition} \\ \midrule
$n_{y,\tilde b,\tilde c}$ & Cross-fitted triple counts: label, baseline prediction, and candidate prediction on each held-out fold, pooled over the $F=3$ folds. \\
$\Delta$, $\Delta^{*}$ & Metric difference (candidate minus baseline) from the pooled counts; $\Delta^{*}$ is one of $N_{\mathrm{boot}}=20{,}000$ multinomial bootstrap draws. \\
Accept rule & Accept if and only if $Q_{0.05}(\Delta^{*})>\varepsilon$; otherwise retain $b$. \\
Margin $\varepsilon$ & $0.005$ on the scale of the selected metric, not rescaled; for net benefit \eqref{eq:netbenefit}, five additional net true positives per $1{,}000$ examples. \\
$\mathcal{M}$ & Any higher-is-better metric computed from a confusion table (weighted $F_1$, balanced accuracy, macro-$F_1$, $F_2$, and, in the binary clinical protocol, net benefit); $\mathcal{M}$ is the metric, distinct from the accepted count $M=\sum_j A_j$ (Section~\ref{sec:method-risk}). \\
\bottomrule
\end{tabulary}
\end{table*}

\begin{remark}[Relation to nested cross-validation]
\label{rem:nestedcv}
Nested cross-validation estimates the expected performance of a selection procedure (Section~\ref{sec:rel-selection}), whereas the gate answers a different question. The outer layer of the design is the locked train/reference split, and the inner layer is the cross-fit within POOL (Fig.~\ref{fig:partition}; data used for each choice in Table~\ref{tab:leakage}, Appendix~\ref{app:sens}). The cross-fitted counts yield the decision statistic, whereas the population bound for the false-improvement rate among accepted updates comes from Proposition~\ref{prop:bound} (Section~\ref{sec:method-risk}); that bound holds on average over decisions from the task population, not for each decision. For this construction the cross-fit within POOL takes the role of the inner layer, so no additional nested layer is needed.
\end{remark}

\begin{algorithm}[t]
\caption{BRIDGE gate}
\label{alg:gate}
\begin{algorithmic}[1]
\Require SEARCH and POOL with labels $y$; candidate space $\mathcal{C}$; metric $\mathcal{M}$; bag count, fold count $F$, replicate count $N_{\mathrm{boot}}$, quantile level $q$, and margin $\varepsilon$ (values in Table~\ref{tab:hyper})
\Ensure deploy candidate $c$ or retain baseline $b$
\State fit the bagged base models on SEARCH and store the bag predictions on POOL, including the baseline predictions $\hat b$
\State $\{Z_k\}_{k=1}^{F} \gets$ stratified $F$-fold split of POOL; initialize counts $n_{y,\tilde b,\tilde c}\gets 0$
\For{$k=1$ to $F$}
    \State $c_k \gets$ \Call{Select}{$\text{POOL}\setminus Z_k,\ \mathcal{C},\ \mathcal{M},\ b$} \Comment{must improve on the baseline}
    \State accumulate counts $n_{y,\tilde b,\tilde c}$ on $Z_k$ for $(\hat b,\hat c_k)$
\EndFor
\State $\{\Delta^{*}_{(r)}\}_{r=1}^{N_{\mathrm{boot}}} \gets$ resample the pooled counts $N_{\mathrm{boot}}$ times from the multinomial distribution and compute \eqref{eq:deltametric} for each draw
\If{$Q_{q}(\Delta^{*}) > \varepsilon$}
    \State $c \gets$ \Call{Select}{$\text{POOL},\ \mathcal{C},\ \mathcal{M},\ b$}
    \State \Return deploy $c$
\Else
    \State \Return retain baseline $b$
\EndIf
\end{algorithmic}
\end{algorithm}

\begin{algorithm}[t]
\caption{Candidate selection on a subset (used by Algorithm~\ref{alg:gate})}
\label{alg:crossfit}
\begin{algorithmic}[1]
\Function{Select}{$Z,\mathcal{C},\mathcal{M},b$}
    \State score every $c\in\mathcal{C}$ on $Z$ using the stored bag predictions
    \State \Return $\arg\max_{c} \mathcal{M}(c)$ subject to $\mathcal{M}(c) > \mathcal{M}(b)+10^{-12}$; else $b$; candidates within $10^{-12}$ of the best take the first in list order
\EndFunction
\end{algorithmic}
\end{algorithm}

\subsection{Generalization to Metrics and Multiclass Tasks}
\label{sec:method-metrics}

The same gate generalizes to any metric computed from the confusion table and to multiclass tasks with $K$ classes, without changing the decision rule. The counts, the metric difference $\Delta$, and the bootstrap are unchanged, so the gate applies to metrics defined on the confusion structure of \eqref{eq:counts} whose direction is higher-is-better; a lower-is-better metric would need a sign transform, which this implementation does not define. Labels and predictions are coded in $\{0,1\}$ in the binary experiments, a relabeling of the two classes. For $K$ classes the labels and predictions range over $[K]=\{1,\dots,K\}$ as in Section~\ref{sec:method-partition}. In the binary case, applying \eqref{eq:deltametric} to the baseline and candidate confusion tables gives
\begin{equation}
\label{eq:binarymetric}
\Delta \;=\; \mathcal{M}\!\left(\{n^{\mathrm{cand}}_{y,\hat y}\}_{y,\hat y\in\{0,1\}}\right) \;-\; \mathcal{M}\!\left(\{n^{\mathrm{base}}_{y,\hat y}\}_{y,\hat y\in\{0,1\}}\right).
\end{equation}

Multiclass tasks use a bagged HGB base learner with class-scale and simplex candidates. The class-scale family scales the score of one class at a time. The simplex family averages HGB, DT, and SVC, which are the same members as in the binary space. The alternative member pairs are also evaluated on the five multiclass datasets of (D) (Section~\ref{sec:exp-base}): logistic regression with naive Bayes and an MLP with KNN. The joint space of the seven member families is evaluated there as well. The default multiclass space has $14+8K$ candidates for $K$ classes, and the joint space has $42+32K$. Appendix~\ref{app:sens} gives the composition, the member settings, and the tie rules, and Table~\ref{tab:hyperbasis} gives the basis of each setting. The baseline is the argmax of the bagged HGB, and the same gate procedure uses multiclass weighted-$F_1$. For $K$ classes, the construction uses the confusion tables over $y,\hat y\in[K]=\{1,\dots,K\}$,
\begin{equation}
\label{eq:multiclass}
\Delta \;=\; \mathcal{M}\!\left(\{n^{\mathrm{cand}}_{y,\hat y}\}_{y,\hat y\in[K]}\right) \;-\; \mathcal{M}\!\left(\{n^{\mathrm{base}}_{y,\hat y}\}_{y,\hat y\in[K]}\right).
\end{equation}

For binary and multiclass tasks alike, the evidence is stored as counts. The number of resampled categories therefore grows as $K^3$. The count-based representation limits the implementation to metrics computed from hard predicted labels.

Common metrics of this kind, namely weighted $F_1$, balanced accuracy, macro-$F_1$, $F_2$, sensitivity, and specificity \cite{vanrijsbergen1979,brodersen2010,sokolova2009}, are computed from hard labels. Equations~\eqref{eq:binarymetric} and~\eqref{eq:multiclass} therefore cover these metrics. Net benefit at a threshold probability $p_t$ is used in the binary clinical protocol,
\begin{equation}
\label{eq:netbenefit}
\mathrm{NB}(p_t) \;=\; \frac{\mathrm{TP}}{n_{\mathrm{eval}}} \;-\; \frac{\mathrm{FP}}{n_{\mathrm{eval}}}\cdot\frac{p_t}{1-p_t},
\end{equation}
where TP and FP are the numbers of true and false positives on the data being assessed, and $n_{\mathrm{eval}}$ is the number of examples in those data. The clinical decision threshold $p_t$ is distinct from the candidate prediction threshold $t$ of \eqref{eq:predictor}. The definition of net benefit follows \cite{vickers2006}. 

Because the deployed model decides at one threshold, the gate computes the decision metric at that threshold. By contrast, the area under the receiver operating characteristic curve (AUC) \cite{fawcett2006} scores the ranking of scores across all thresholds. AUC is therefore unchanged when an update changes the threshold alone, as many candidate updates do. The gate therefore uses a threshold metric for both settings. For clinical decisions the metric is net benefit, read as a decision curve \cite{vickers2006}, and for rare classes the metrics are the minority-sensitive ones, balanced accuracy and macro-$F_1$ (Section~\ref{sec:exp-base}).

\subsection{Risk Among Accepted Updates: Upper Bound and Scope}
\label{sec:method-risk}

Whether an accepted update is a false improvement can be seen on the locked reference alone, so the risk of concern is a proportion over many decisions. This risk is the share of false improvements among the updates that the gate accepts (quantities in Table~\ref{tab:outcomes}). The subsection first states when this share can be bounded (Assumption~\ref{ass:pop}) and then defines the share (Definition~\ref{def:risk}). The subsection then gives an upper bound that holds for a finite number of tasks (Proposition~\ref{prop:bound}). The scope of the bound is the average over tasks drawn from one population, not a single decision. Real data, where the assumption does not hold, are treated in Remark~\ref{rem:real}.

\begin{assumption}[Task population]
\label{ass:pop}
(i) Tasks $j=1,\dots,T$ are drawn i.i.d.\ from a fixed task population $\mathcal{Q}$; (ii) the batch construction and the training/reference split of each task are part of that draw; and (iii) the gate and candidate space are frozen before any evaluation data are seen. The proofs use only that the decision pairs $(A_j,H_j)$ are i.i.d.\ across tasks.
\end{assumption}

\begin{definition}[Risk among accepted updates]
\label{def:risk}
For task $j$, let $A_j=\mathbb{1}\{\text{update accepted}\}$ and $H_j=\mathbb{1}\{\Delta_{\mathrm{ref},j} \le 0\}$, where $\Delta_{\mathrm{ref},j}$ is the metric difference between the deployed configuration and the baseline of task $j$ measured on its locked reference (the update is a false improvement when it does not improve). The quantity of interest is the proportion of false improvements among those accepted, defined when the acceptance probability is positive ($\mathbb{E}[A_j]>0$),
\begin{equation}
\label{eq:rq}
r_{\mathcal{Q}} \;=\; \frac{\mathbb{E}[A_j H_j]}{\mathbb{E}[A_j]} .
\end{equation}
\end{definition}

The bound uses two counts over the $T$ tasks: the number $M$ of accepted updates and the number $S$ of false improvements among those accepted. Given $M$, the count $S$ behaves like the number of false improvements in $M$ independent draws (Lemma~\ref{lem:binom}). A binomial tail bound therefore applies (Proposition~\ref{prop:bound}).

\begin{table*}[!t]
\centering\footnotesize
\setlength{\tabcolsep}{4pt}\renewcommand{\arraystretch}{1.1}
\caption{Outcome and risk quantities.}
\label{tab:outcomes}
\setlength{\tymin}{28pt}\begin{tabulary}{\linewidth}{@{}lL@{}}
\toprule\textbf{Symbol / term} & \textbf{Definition} \\ \midrule
$\Delta_{\mathrm{ref},j}$ & Metric difference between the deployed configuration $c_j$ and $b_j$ on the locked reference. \\
False improvement, $H_j$ & $H_j=\mathbf{1}\{\Delta_{\mathrm{ref},j}\le 0\}$: the deployed configuration does not improve the selected metric on the locked reference; ties count as false improvements. \\
Clearly worse / reference-uncertain & From the $95\%$ Bayesian-bootstrap interval of $\Delta_{\mathrm{ref},j}$ \cite{rubin1981}: a Dirichlet posterior over the confusion cells with pseudocount $0.5$ (the Jeffreys prior) on the binary and clinical protocols and over the positive cells on multiclass, $20{,}000$ draws, quantiles $2.5\%$ and $97.5\%$; entirely below $0$ = clearly worse; entirely above $0$ = clearly improving; otherwise reference-uncertain. \\
$A_j$; false-improvement rate $\hat r$ & $A_j=1$ if decision $j$ accepts; $\hat r=\sum_j A_jH_j/\sum_j A_j$, defined when $\sum_j A_j>0$. \\
Yield; total gain & Accepted updates per $100$ decisions; $\sum_j A_j\Delta_{\mathrm{ref},j}$. \\
$r_{\mathcal{Q}}$ & False-improvement rate among accepted updates, $\mathbb{E}[A_jH_j]/\mathbb{E}[A_j]$, on the i.i.d.\ task population $\mathcal{Q}$. \\
Proposition~\ref{prop:bound} bound $U$ & KL upper confidence bound on $r_{\mathcal{Q}}$ at $\delta=0.05$; valid on $\mathcal{Q}$ only. \\
KL-ratio bound & Real data: a $95\%$ upper bound on the false-improvement rate for these datasets (not for new datasets), $U_{\mathrm{KL}}(\bar W;\,J)/L_{\mathrm{KL}}(\bar A;\,J)$, where $U_{\mathrm{KL}}$ and $L_{\mathrm{KL}}$ are the one-sided Kullback--Leibler (KL) upper and lower bounds at level $\alpha/2$ (Section~\ref{sec:method-risk}) and $W_j=A_jH_j$; conditional on the observed collection. \\
Leave-one-source-out (LOCO) & Maximum KL-ratio bound after leaving out one source (clinical) or one dataset (non-clinical), to check that no single source or dataset drives the result. \\
Conservative bound & The same bound (Proposition~\ref{prop:bound} on $\mathcal{Q}$, KL-ratio on the real data) with every accepted update that is not clearly improving counted as a false improvement. \\
Safer & A lower false-improvement rate (equivalently, a lower risk among accepted updates). \\
\bottomrule
\end{tabulary}
\end{table*}

The gate makes each decision from POOL alone. The bound in this subsection belongs to the evaluation protocol: it is computed on the locked reference after the decision has been taken, so it measures risk rather than steering the gate. Proposition~\ref{prop:bound} applies to any frozen gate, and the contribution is the gate together with its evaluation, not a theorem for a single deployment.

\begin{lemma}[Conditional binomial]
\label{lem:binom}
Let $M=\sum_{j=1}^{T} A_j$ and $S=\sum_{j=1}^{T} A_j H_j$. Conditionally on $M=m$, $S\sim\mathrm{Binomial}(m,r_{\mathcal{Q}})$.
\end{lemma}

\begin{proposition}[Finite-sample upper bound]
\label{prop:bound}
Under Assumption~\ref{ass:pop}, with $\mathrm{kl}(x\|u)=x\log(x/u)+(1-x)\log((1-x)/(1-u))$, define
\begin{equation}
\label{eq:U}
U \;=\; \sup\Bigl\{\, u \in \bigl[\tfrac{S}{M},1\bigr] : M\,\mathrm{kl}\!\left(\tfrac{S}{M}\,\big\|\,u\right) \le \log\tfrac{1}{\delta} \,\Bigr\},
\end{equation}
and set $U=1$ when $M=0$. Then $\Pr\!\left(r_{\mathcal{Q}} \le U\right) \ge 1-\delta$, where $\delta\in(0,1)$ is the error level ($0.05$ in the evaluation).
\end{proposition}

The bound $U$ exceeds the true false-improvement rate with probability at least $95\%$, and it is computed from the $M$ accepted updates and the $S$ false improvements on the reference. The bound inverts an exact binomial tail inequality rather than a limit theorem, so the bound holds for any finite number of tasks, not merely asymptotically.

\begin{proof}[Proof sketch]
Conditionally on $M=m$, Lemma~\ref{lem:binom} gives $S\sim\mathrm{Binomial}(m,r_{\mathcal{Q}})$; inverting the lower Chernoff--Kullback--Leibler tail of a binomial \cite{boucheron2013,langford2005} gives a one-sided upper bound of coverage $\ge 1-\delta$ for every $m$, and averaging over the (stochastic) $M$ preserves coverage. Full proofs of Lemma~\ref{lem:binom}, Proposition~\ref{prop:bound} and Corollary~\ref{cor:cons} are in Appendix~\ref{app:proofs}.
\end{proof}

The bound counts an update as a false improvement exactly when the reference shows no improvement. A stricter reading also counts as a false improvement every accepted update whose reference result is uncertain. The bound remains valid under that reading (Corollary~\ref{cor:cons}).

\begin{corollary}[Conservative variant]
\label{cor:cons}
Counting every accepted update of reference-uncertain status as a false improvement yields a valid but more conservative bound; this variant is reported alongside the Proposition~\ref{prop:bound} bound in Table~\ref{tab:chronology}.
\end{corollary}

\rowcolors{0}{}{}

\begin{remark}[Real data]
\label{rem:real}
Proposition~\ref{prop:bound} does not apply to real benchmark tasks. The trials of a task share one locked reference, whose error enters every trial of that task, and the tasks form a fixed collection rather than an i.i.d.\ sample. Instead, the false-improvement rate is reported through the KL-ratio bound, which is a ratio of one-sided Kullback--Leibler bounds \cite{langford2005},
\begin{equation}
\label{eq:Uhat}
\hat U = \frac{U_{\mathrm{KL}}(\bar W;\,J,\alpha/2)}{L_{\mathrm{KL}}(\bar A;\,J,\alpha/2)},
\end{equation}
clamped to $1$ and set to $1$ when the lower denominator is not positive. The quantities $\bar A$ and $\bar W$ are the sample means over the $J$ trials of the acceptance indicator $A_j$ and of $W_j=A_jH_j$. The bounds $U_{\mathrm{KL}}$ and $L_{\mathrm{KL}}$ are one-sided Kullback--Leibler bounds \cite{langford2005,boucheron2013}, with $\alpha=0.05$ split as $\alpha/2$ across the two bounds by the union bound. The bound $\hat U$ is conditional on the observed collection. This bound is reported for (A), (C), (D), and the OpenML collection (E) in Section~\ref{sec:experiments}. LOCO and a dataset-level bootstrap bound serve as sensitivity analyses \cite{fieldwelsh2007}, and the box in Section~\ref{sec:limitations} collects the scope of each claim (the estimand $r_{\mathcal{E}}$ is defined in Appendix~\ref{app:proofs}). The scope of the method depends on the input that the gate requires as well as on the bound, as stated in Remark~\ref{rem:interface}.
\end{remark}

\begin{remark}[Interface and evidence scope]
\label{rem:interface}
The gate consumes the baseline predictions, the candidate predictions, and the observed labels on POOL, together with a metric of the family defined in Section~\ref{sec:method-metrics}. Because it uses only these predictions and labels, its interface does not depend on the model type, and the labels it needs are the POOL labels already available at the decision, not new labels. The empirical evidence covers tabular learners (trees, kernels, boosting, and a two-hidden-layer multilayer perceptron, Section~\ref{sec:exp-base}).
\end{remark}

\subsection{Experimental Design: Data, Partition, Comparators, and Evaluation}
\label{sec:exp-settings}

\paragraph{Data.} The evidence rests on $76$ real datasets or clinical tasks and on a set of synthetic tasks. The evidence is of three kinds: (i) a first evaluation on new data in (A), (B), and (E); (ii) a re-test on data seen before in (C) and (D), where $18$ datasets of (D) were used during development and (C) repeats a failed run; and (iii) post hoc analyses (T3, Table~\ref{tab:falseimp}, and the source-level checks). The data form five collections, and Section~\ref{sec:experiments} assigns a distinct role to each collection. (A) The $11$ new real datasets supply the main comparison of acceptance rules (Section~\ref{sec:exp-rules}). No development run used these datasets. (B) The synthetic population $\mathcal{Q}$ supplies $1{,}400$ independent tasks with reference-measured gain. These tasks use the task seeds of the generator that no pilot run had used (seed allocation and generator in Appendix~\ref{app:datasets}). These tasks test the formal bound. A second, fresh batch of $1{,}400$ tasks repeats that test (Section~\ref{sec:exp-prop1}). (C) The $12$ retrospective clinical tasks, drawn from six sources, support the clinical application (Section~\ref{sec:exp-clinical}). (D) The $39$ supporting datasets are $34$ binary confirmation datasets, $18$ of which were used during development and are re-evaluated with fresh seeds, and $5$ multiclass datasets. These datasets form the basis of the conditions-of-use tests. The gate is also run once on all $39$ datasets together (Section~\ref{sec:exp-base}). (E) The OpenML collection of $14$ new datasets \cite{vanschoren2013} passes through the released program under the same protocol and serves as supporting evidence (Section~\ref{sec:exp-summary}). In total, (A), (C), (D), and (E) contain $11+12+39+14=76$ real data units. Table~\ref{tab:splits} lists each unit with its domain, source, size, minority rate, and reference size (Appendix~\ref{app:datasets}). Table~\ref{tab:openmlE} gives the per-dataset decisions of (E) (Appendix~\ref{app:settingE}).

Datasets enter under the pre-specified inclusion and exclusion rules in Appendix~\ref{app:datasets}, and a group and temporal audit, de-duplication, and filtering are applied in the loaders before any split, so that rows from the same group or subject do not cross the split. Appendix~\ref{app:settingE} reports two further collections. The first is the pre-registered extension (E2), which adds $19$ OpenML datasets, some of which carry curation flags (origin not verifiable, artificial data, or flattened images). The second is a later supplement with one dataset, Volkert.

\paragraph{Partition and trials.} A trial is one run of the procedure on one task under a fresh seed, and each seed re-draws the partition. The seed fixes the training rows, the SEARCH/POOL split, the bag draws, the folds, and the bootstrap draw. Each trial ends in one decision: accept or reject the update.

Each dataset receives $100$ trials, the first count tested that passed the power check of Table~\ref{tab:criterion} in development. That check gave $\mathrm{LOCO}_0$ of $18.2\%$ at $50$ trials and $0.8\%$ at $100$ (development runs in Table~\ref{tab:devruns}). Each clinical task has $200$ trials. This count was raised after the clinical run with $100$ trials failed the per-source LOCO condition (run history in Appendix~\ref{app:prereg}). Each multiclass dataset has $60$ trials, the count pre-registered for the multiclass confirmation (basis of each setting in Table~\ref{tab:hyperbasis}). On (B), each task carries one trial, because every task is a new draw from $\mathcal{Q}$. These counts give $1{,}100$ decisions in (A) ($11\times100$), $1{,}400$ in (B), $2{,}400$ in (C) ($12\times200$), and $3{,}700$ in (D) ($34\times100+5\times60$). The four main settings therefore contain $8{,}600$ decisions, and (E) adds $1{,}280$ ($11$ binary datasets $\times100+3$ multiclass datasets $\times60$).

The locked reference and the random state of the learner stay fixed across trials. The reference is set aside once per task with a pre-registered seed, and the conclusion is unchanged under two further reference seeds (run-level numbers in Table~\ref{tab:citedruns}). The reference, the outer split, and the SEARCH/POOL split are stratified by class. Because trials from the same dataset are correlated, the confidence intervals treat the dataset rather than the trial as the unit \cite{fieldwelsh2007}; Appendix~\ref{app:datasets} lists the training rows and reference sizes per dataset. The experiments ran in Python~3.12 with scikit-learn~1.9.1 and XGBoost~3.4.1 on three CPU machines. A replay on a second machine reproduces the stored trials bit-identically, and Appendix~\ref{app:prereg} lists the exceptions and the full software versions.

\paragraph{Comparators.}
All rules judge the same decisions, with the same candidates, the same base model, the same POOL counts, the same reference, and the same fixed learner configuration for every rule, so the only difference is the accept/reject rule. The test-based rules apply the statistics of the classical tests to fixed predictions on POOL splits. Because the base models are not refitted in each split, the $5\times2$~cv and corrected $t$ statistics follow the original formulas \cite{dietterich1998,nadeau2003} but not the original retraining procedure. The comparator set comprises the naive rule, an honest holdout, a fixed validation margin, and a random control. The same set also includes three acceptance rules built on classical test statistics: the McNemar test, the $5\times2$~cv test, and the corrected resampled $t$-test, also with per-split re-selection. Table~\ref{tab:criterion} defines each comparator and states the pre-registered pass criterion.

\begin{table*}[!t]
\centering\footnotesize
\setlength{\tabcolsep}{4pt}\renewcommand{\arraystretch}{1.1}
\caption{Comparator rules and the pre-registered criterion.}
\label{tab:criterion}
\rowcolors{0}{}{}
\setlength{\tymin}{28pt}\begin{tabulary}{\linewidth}{@{}lL@{}}
\toprule\textbf{Symbol / term} & \textbf{Definition} \\ \midrule
Naive rule & Accept whenever the configuration selected on the whole of POOL improves the validation metric, i.e.,\ \eqref{eq:acceptproblem} with a positive gain. \\
McNemar & Exact one-sided binomial test on discordant pairs (candidate correct and baseline wrong, or the reverse) on POOL; tests per-example accuracy. \\
$5\times2$~cv & Dietterich's paired $t$ statistic, $5$ d.f., on the fixed predictions of the final configuration over five $2$-fold splits of POOL, without refitting. \\
Corrected $t$ & Ten $2/3$--$1/3$ splits of POOL, Nadeau--Bengio variance correction, $9$ d.f., final configuration; with re-selection: the candidate is re-chosen inside each split. \\
Fixed margin $t$ & Accept when the validation gain on POOL is at least $t$; $t^{*}=0.05$ pre-registered from development, other $t$ post hoc. \\
Random control & Accept at random with BRIDGE's yield. \\
Dataset-level bootstrap & $11$ datasets resampled with replacement $10{,}000$ times; $95$th percentile of the pooled false-improvement rate. \\
Pre-registered criterion & Real data: pass if and only if (i) KL-ratio bound $<10\%$, (ii) LOCO $<10\%$, (iii) yield $\ge 5$ per $100$, (iv) at least $3$ datasets (clinical: tasks) with yield $\ge 5$ per $100$, and (v) on the non-clinical runs, the power check $\mathrm{LOCO}_0<10\%$, the LOCO bound recomputed with zero false improvements. On $\mathcal{Q}$: Proposition 1 bound $<10\%$ and yield $\ge 5$ per $100$. \\
\bottomrule
\end{tabulary}
\end{table*}

\paragraph{Evaluation metrics.} Each decision $j$ is evaluated on the locked reference $R_j$, by comparing the deployed configuration $c_j$ with the baseline $b_j$. This comparison produces the metric difference $\Delta_{\mathrm{ref},j}$ and the false-improvement indicator $H_j$ of Definition~\ref{def:risk}. The main quantity is the false-improvement rate among accepted updates, $\hat r=\left(\sum_j A_j H_j\right)/\left(\sum_j A_j\right)$. Its upper bound is $U$ \eqref{eq:U} on $\mathcal{Q}$ and $\hat U$ \eqref{eq:Uhat} on the real data. Two further bounds check the real-data result: LOCO and a dataset-level bootstrap interval. LOCO is the maximum $\hat U$ obtained when one dataset (one source for the clinical tasks) is left out in turn. This bound indicates whether a single dataset drives the result. The dataset-level bootstrap interval takes datasets rather than trials as the sampling unit.

The main metric is weighted-$F_1$, fixed in advance as a choice of the study. That choice is not a claim that weighted-$F_1$ is better than balanced accuracy or macro-$F_1$. The comparison is repeated with those two metrics, and Table~\ref{tab:general} reports the results by metric, size, and class count. For the clinical tasks the comparison is also repeated with two clinically oriented metrics (results in Section~\ref{sec:exp-clinical}). The first is $F_2$ \cite{vanrijsbergen1979}, which weighs sensitivity twice as heavily as precision. The second is net benefit \eqref{eq:netbenefit}, which credits true positives against false positives at the clinical threshold $p_t$. Further quantities are the yield (the number of accepted updates per $100$ decisions) and the total gain $\sum_j A_j \Delta_{\mathrm{ref},j}$. For the clinical tasks, the reported quantities also include the clinical consequences of the accepted updates, namely the flagged patients and the extra true and false positives.

\paragraph{Statistical analysis.} The pre-registered criterion of Table~\ref{tab:criterion} provides the primary decision of every run. This criterion is a decision rule on an upper bound, and it does not depend on the comparators. A run passes when the $95\%$ upper bound on the false-improvement rate lies below $10\%$. The bound has error level $\delta=0.05$. The bound component of the criterion is equivalent to rejecting $H_0\!: r\ge0.10$ in favor of $r<0.10$ in a one-sided test at level $0.05$. A pass requires every condition: the bound, LOCO, yield, the number of datasets, and $\mathrm{LOCO}_0$. $\mathrm{LOCO}_0$ is a check of precision sufficiency in the no-event scenario, not a general power proof. The quantity $r$ is $r_{\mathcal{Q}}$ on (B) (Proposition~\ref{prop:bound}) and the false-improvement rate of the observed collection on the real data (Remark~\ref{rem:real}). Three hypothesis tests support the criterion, and each test uses the significance level $\alpha=0.05$. The first two tests resample datasets rather than trials. The third test is exact, and it is computed on all decisions and again within each dataset.

The first test, T1, compares the false-improvement rate of BRIDGE with that of a comparator rule $k$ on the same decisions,
\begin{equation}
H_0\colon r_{\mathrm{B}}-r_k=0 \quad\text{versus}\quad H_1\colon r_{\mathrm{B}}-r_k\neq 0,
\label{eq:hdiff}
\end{equation}
where $r_{\mathrm{B}}$ and $r_k$ denote the false-improvement rates among the updates accepted by BRIDGE and by rule $k$, respectively. The test is a paired cluster bootstrap \cite{efron1993,fieldwelsh2007} with $10{,}000$ resamples of the datasets (of the tasks in (B) and (C)). Both rules are evaluated on the same resample. The $95\%$ confidence interval (CI) is the percentile interval of $\hat r_{\mathrm{B}}-\hat r_k$. The two-sided $p$ is twice the smaller share of resampled differences on either side of zero. Within one collection, the $p$-values of all comparators are adjusted by the Holm procedure \cite{holm1979}, which controls the family-wise error rate at $\alpha$.

The second test, T2, asks whether a variant of the gate matches the default gate in false-improvement rate. Two one-sided tests (TOST) \cite{schuirmann1987} are applied to the hypotheses
\begin{equation}
H_0\colon \lvert r_{\mathrm{B}}-r_k\rvert\ge\eta \quad\text{versus}\quad H_1\colon \lvert r_{\mathrm{B}}-r_k\rvert<\eta,
\label{eq:hequiv}
\end{equation}
using the equivalence margin $\eta=2$ percentage points, the same cluster bootstrap, and the Holm adjustment over the variants compared in one run. The equivalence margin $\eta$ was fixed in the pre-registration of the first equivalence test (run MINI-44) before any equivalence result was computed (Appendix~\ref{app:prereg}). Every run is labeled MINI-$n$ and listed with its pre-registration in Table~\ref{tab:chronology}.

The third test, T3, asks whether BRIDGE attains a lower false-improvement rate by selecting updates rather than by accepting fewer updates. The null hypothesis $H_0$ of this test states that the $n_{\mathrm{B}}$ updates accepted by BRIDGE form a random draw of the same size from the $n_{\mathrm{N}}$ updates accepted by the naive rule. Of these $n_{\mathrm{N}}$ updates, $h_{\mathrm{N}}$ are false improvements. Under $H_0$ the number of false improvements $h$ among the draw follows a hypergeometric distribution, and the one-sided $p$ is the probability of at most the observed count $h_{\mathrm{B}}$,
\begin{equation}
p=\Pr\left(h\le h_{\mathrm{B}}\right),\qquad h\sim\mathrm{Hypergeometric}\left(n_{\mathrm{N}},\,h_{\mathrm{N}},\,n_{\mathrm{B}}\right).
\label{eq:hrandom}
\end{equation}

A paired $t$-test on per-trial differences is a common choice for comparing classifiers \cite{dietterich1998}, but it does not fit these comparisons for two reasons. First, the false-improvement rate is a ratio over accepted updates, and the number of accepted updates differs between rules. Second, the trials of one dataset are not independent. T1 was pre-registered for (C) and (D) and was applied afterwards to the stored runs of (A), (B), and (E). T2 was specified before its computation, and T3 is a post hoc analysis of the stored runs.

\section{Results}
\label{sec:experiments}

Every confirmatory number reported in this section comes from a pre-registered run. The code was hashed before the run, and the seeds were drawn after pre-registration (chronology in Appendix~\ref{app:prereg}). A run passes when the pre-registered criterion of Section~\ref{sec:exp-settings} is met (Table~\ref{tab:criterion}). This criterion is a $95\%$ upper bound on the false-improvement rate below $10\%$. The $10\%$ level was fixed in the first pre-registered run and stayed unchanged in every later run. The run list in Table~\ref{tab:chronology} (Appendix~\ref{app:prereg}) records the criterion of each run. In every setting, the level lies below the false-improvement rate of the naive rule, which ranges from $12.5\%$ to $25.6\%$ (Fig.~\ref{fig:validation_gain}).

\subsection{Summary of Main Results}
\label{sec:exp-summary}

\begin{table}[!t]
\centering
\footnotesize
\providecolor{bghead}{HTML}{DBEFF8}\providecolor{bgzebra}{HTML}{F2F9FC}\providecolor{bgpass}{HTML}{D7F0D2}\providecolor{bgfail}{HTML}{FCDCD7}
\setlength{\tabcolsep}{1.5pt}
\caption{Summary of main results: the naive rule against BRIDGE, (A)--(E).}
\label{tab:headline}
\begin{tabular}{@{}lrrrrr@{}}
\toprule
Setting & \begin{tabular}[b]{@{}c@{}}Naive\\false \% $\downarrow$\end{tabular} & \begin{tabular}[b]{@{}c@{}}BRIDGE\\false \% $\downarrow$\end{tabular} & \begin{tabular}[b]{@{}c@{}}Reduction\\(pp) $\uparrow$\end{tabular} & \begin{tabular}[b]{@{}c@{}}False\\removed \% $\uparrow$\end{tabular} & \begin{tabular}[b]{@{}c@{}}Gain per\\update $\uparrow$\end{tabular} \tabularnewline
\midrule
(A) New datasets & 19.1 & 1.1 & 18.0 & 99 & 2.5$\times$ \\
(B) Synthetic & 20.7 & 1.7 & 19.0 & 97 & 2.1$\times$ \\
(C) Clinical & 20.0 & 0.0 & 20.0 & 100 & 1.8$\times$ \\
(D) Supporting & 12.5 & 0.5 & 12.0 & 99 & 2.1$\times$ \\
(E) OpenML$^{\dagger}$ & 25.6 & 1.0 & 24.6 & 99 & 3.1$\times$ \\
\bottomrule
\end{tabular}

\vspace{2pt}
\parbox{\linewidth}{\justifying\scriptsize *False \%: false improvements among the accepted updates, pooled over all decisions of a setting (one fresh seed per decision). Reduction: naive rule minus BRIDGE, in percentage points; all $95\%$ intervals lie above zero. $^{\dagger}$(E): $14$ additional OpenML datasets, run through the released program under the same protocol and reported as supporting evidence; the class-count subgroups, the extension (E2), and the dataset amendment appear in Appendix~\ref{app:settingE}. False removed: the false improvements of the naive rule that BRIDGE rejects (Table~\ref{tab:selectivity}). Gain per update: BRIDGE relative to the naive rule. Definitions, intervals, and bounds: Table~\ref{tab:summary}. Every setting: current model HGB with the fixed configuration; candidates from the default space built from HGB, DT and SVC ($173$ binary, $14+8K$ multiclass; Section~\ref{sec:method-partition}); other learners and ensemble members: Table~\ref{tab:scope}.  $\downarrow$/$\uparrow$ = smaller/larger is better.}
\end{table}

Table~\ref{tab:headline} summarizes the main results. In every collection, BRIDGE lowers the false-improvement rate among accepted updates by $12.0$ to $24.6$ percentage points relative to the naive rule, to at most $1.7\%$. Every $95\%$ interval of the reduction over datasets or tasks lies above zero. The paired test \eqref{eq:hdiff} gives $p<0.04$ in every collection, after the Holm correction wherever several rules are compared. BRIDGE removes $97\%$ to $100\%$ of the false improvements that the naive rule accepts (Table~\ref{tab:selectivity}). The mean reference gain per update that BRIDGE deploys is $1.8$ to $3.1$ times that of the naive rule (Table~\ref{tab:selectivity}). All five collections meet the pre-registered criterion. Table~\ref{tab:summary} (Appendix~\ref{app:results}) lists the intervals, bounds, and per-collection counts.

The subsections that follow report the hypothesis tests (Section~\ref{sec:exp-tests}), the comparison of acceptance rules (Section~\ref{sec:exp-rules}), the formal bound (Section~\ref{sec:exp-prop1}), the clinical application (Section~\ref{sec:exp-clinical}), the component analysis (Section~\ref{sec:exp-ablation}), and the conditions of use (Section~\ref{sec:exp-base}). The OpenML collection (E) provides supporting evidence outside the four main settings, and Appendix~\ref{app:settingE} reports its class-count subgroups, the extension (E2), and the amendment of the (E2) dataset list.

\subsection{Hypothesis Tests}
\label{sec:exp-tests}

The hypothesis tests T1 to T3 of Section~\ref{sec:exp-settings} check whether the differences of Section~\ref{sec:exp-summary} could arise by chance. Questions 1 to 3 ask whether BRIDGE accepts a smaller share of false improvements than the naive rule, the test-statistic-based acceptance rules, and the rules that separate choosing from judging (the corrected $t$-test with per-split re-selection and the two holdouts). Question 4 asks whether the margin lowers the false-improvement rate, comparing BRIDGE with BRIDGE without the margin. These four questions use test T1, whose null hypothesis \eqref{eq:hdiff} states equal false-improvement rates.

Question 5 uses test T2, whose null hypothesis \eqref{eq:hequiv} states that the false-improvement rate at a margin between $0.0025$ and $0.01$ differs from that at the default margin $0.005$ by at least $2$ percentage points. Question 6 uses test T3, whose null hypothesis \eqref{eq:hrandom} states that the updates accepted by BRIDGE are a random draw from the acceptances of the naive rule. Table~\ref{tab:tests} gives the result and decision for each question, and Table~\ref{tab:testsdetail} (Appendix~\ref{app:results}) gives the estimates, intervals, and $p$-values.

\begin{table*}[!t]
\centering
\footnotesize
\setlength{\tabcolsep}{4pt}
\renewcommand{\arraystretch}{1.2}
\caption{Hypothesis tests: results and decisions on questions 1--6.}
\label{tab:tests}
\begin{tabular}{@{}l>{\raggedright\arraybackslash}p{1.3cm}>{\raggedright\arraybackslash}p{6.1cm}>{\raggedright\arraybackslash}p{8.6cm}@{}}
\toprule
No. & Test ($H_0$) & Result & Decision on $H_0$ ($\alpha=0.05$) and meaning \\
\midrule
1 & T1 \eqref{eq:hdiff} & BRIDGE lower by $12.0$ to $24.6$ pp ($p<0.04$) & Rejected in (A)--(E): BRIDGE accepts fewer false improvements than the naive rule in every collection \\
2 & T1 \eqref{eq:hdiff} & BRIDGE lower by $2.0$ to $33.1$ pp wherever $H_0$ is rejected & Rejected in (A), (B), (D), and against the $5\times2$~cv test in (C): BRIDGE accepts fewer false improvements \\
3 & T1 \eqref{eq:hdiff} & These rules have $0.0\%$ to $3.0\%$ false improvements; BRIDGE is $1.0$ pp lower than the re-selection test in (D) & Not rejected (except that test in (D)): these rules share the principle of BRIDGE, separating the data that choose from the data that judge; BRIDGE accepts $1.1$ to $1.9$ times as many updates as the holdouts (Section~\ref{sec:exp-rules}) \\
4 & T1 \eqref{eq:hdiff} & BRIDGE lower by $1.1$ to $1.8$ pp in (A), (B), (D); no difference in (C) & Rejected in (A), (B), (D): the margin lowers the false-improvement rate; in (C) both versions accept almost no false improvement \\
5 & T2 \eqref{eq:hequiv} & All $12$ comparisons within $\pm2$ pp ($p\le0.0498$) & Rejected in (A)--(D): rejecting this $H_0$ means that the margins are equivalent, so the false-improvement rate is stable across margins in this range \\
6 & T3 \eqref{eq:hrandom}, post hoc & Random draw $49$ to $240$ false improvements, BRIDGE $0$ to $8$ ($p<10^{-23}$) & Rejected in (A)--(E): BRIDGE picks fewer false improvements than a random choice, also within single datasets in (D) and (E) \\
\bottomrule
\end{tabular}

\vspace{2pt}
\parbox{\textwidth}{\justifying\scriptsize Questions 1--6: Section~\ref{sec:exp-tests}. T1--T3 and the null hypotheses: Section~\ref{sec:exp-settings}. The test, not BRIDGE, decides on $H_0$. Rejecting $H_0$ means that the data show the effect named in the question (in question 5, that the margins give equivalent results). Not rejecting means that no effect is shown, which is not proof of equality. In (C), the pre-registered unit is the task ($12$ clusters, $p=0.004$); over the six sources ($6$ clusters), the difference does not differ significantly ($p=0.17$). pp = percentage points. $p$: Holm-adjusted wherever several rules are compared within one collection. Estimates, $95\%$ intervals, and $p$-values of every comparison in every collection: Table~\ref{tab:testsdetail} (Appendix~\ref{app:results}).}
\end{table*}

For questions 1 and 2, the naive rule and the test-statistic-based acceptance rules accept a larger share of false improvements than BRIDGE. The null hypothesis of equal rates is rejected against the naive rule in all five collections, against these test-based rules in (A), (B), and (D), and against the $5\times2$~cv test in (C). For question 3, the rules that also separate choosing from judging have a low false-improvement rate as well ($0.0\%$ to $3.0\%$), consistent with the source named in Section~\ref{sec:intro}: choosing and judging on the same data. Test T1 detects a difference in (D), where BRIDGE is $1.0$ percentage point lower. Equal rates are a separate question from a difference test, so Section~\ref{sec:exp-rules} compares these rules by the pass criterion and by the number of accepted updates.

For question 4, the margin lowers the false-improvement rate by $1.1$ to $1.8$ percentage points in (A), (B), and (D), within the $2$-point equivalence margin of question 5. On the binary datasets of (E), it lowers the rate by $7.7$ points, from $8.9\%$ without the margin to $1.2\%$ with it (descriptive; Section~\ref{sec:exp-ablation}). For question 5, a margin between $0.0025$ and $0.01$ changes the rate by less than $2$ percentage points.

For question 6, a random draw of the same size from the naive acceptances contains $49$ to $240$ false improvements, against $0$ to $8$ for BRIDGE, so $H_0$ is rejected in every collection. Within single datasets, $H_0$ is also rejected in (D) and (E), so BRIDGE selects among the trials of one dataset there as well. In (A) and (C), the selection is shown across datasets ($p=0.15$ and $0.45$ within datasets; (B) has one trial per task). This analysis is post hoc and uses the stored runs.

\subsection{Comparison of Acceptance Rules}
\label{sec:exp-rules}

This subsection compares BRIDGE with the naive rule, three acceptance rules built on classical test statistics, the corrected resampled $t$-test with per-split re-selection, and an honest holdout. The decisive comparators are the rules that also separate selection from evaluation, the same-POOL cell, and the fixed margin; the naive rule and the test-based rules are not. Table~\ref{tab:falseimp} gives the operating point of each rule, and a fixed margin chosen on six settings and applied to the seventh passes six of the seven (Appendix~\ref{app:perdataset}). All rules are judged on the same decisions by the false-improvement rate among the updates that each rule accepts. The test-based rules and the holdout were run on the four main settings, (A) to (D). The OpenML collection (E) was run through the released program, which records the decisions of the naive rule and of BRIDGE. The collection (E) therefore enters this comparison through the naive rule and the random-selection control. Table~\ref{tab:rulesall} gives the false-improvement rate of every rule in every collection. Appendix~\ref{app:results} provides the detailed counts and bounds.

\begin{table}[!t]
\centering
\footnotesize
\setlength{\tabcolsep}{3pt}
\caption{Acceptance rules on collections (A)--(E): false-improvement rate of the accepted updates (\%).}
\label{tab:rulesall}
\begin{tabular}{@{}lrrrrr@{}}
\toprule
Rule & (A) & (B) & (C) & (D) & (E) \\
\midrule
Naive (validation improves)  & 19.1$^{*}$\,$\times$ & 20.7$^{*}$\,$\times$ & 20.0$^{*}$\,$\times$ & 12.5$^{*}$\,$\times$ & 25.6$^{*}$\,$\times$ \\
McNemar, same POOL  & 12.0$^{*}$\,$\times$ & 34.8$^{*}$\,$\times$ & 0.0$^{\circ}$\,$\times$ & 6.7$^{*}$\,$\times$ & -- \\
McNemar, cross-fit  & 8.9\,$\times$ & 26.0\,$\times$ & 0.0$^{\circ}$\,$\times$ & 2.5$^{*}$\,$\checkmark$ & -- \\
$5\times2$~cv paired $t$  & 6.3$^{*}$\,$\times$ & 5.6$^{*}$\,$\checkmark$ & 2.5$^{*}$\,$\checkmark$ & 3.6$^{*}$\,$\checkmark$ & -- \\
Corrected resampled $t$  & 6.3$^{*}$\,$\times$ & 5.7$^{*}$\,$\checkmark$ & 2.6$^{\circ}$\,$\checkmark$ & 4.2$^{*}$\,$\checkmark$ & -- \\
Corrected $t$, re-selection  & 2.3$^{\circ}$\,$\checkmark$ & 3.0$^{\circ}$\,$\checkmark$ & 0.4$^{\circ}$\,$\checkmark$ & 1.5$^{*}$\,$\checkmark$ & -- \\
BRIDGE without margin  & 2.8$^{*}$\,$\checkmark$ & 3.5$^{*}$\,$\checkmark$ & 0.6$^{\circ}$\,$\checkmark$ & 1.6$^{*}$\,$\checkmark$ & -- \\
Holdout 2/3 + 1/3  & 1.3 & 1.9\,$\checkmark$ & 0.0$^{\circ}$\,$\checkmark$ & 0.3$^{\circ}$\,$\checkmark$ & -- \\
Holdout 1/2 + 1/2  & 3.2 & 2.1\,$\checkmark$ & 0.0$^{\circ}$\,$\checkmark$ & 0.8$^{\circ}$\,$\checkmark$ & -- \\
\textbf{BRIDGE} & 1.1\,$\checkmark$ & 1.7\,$\checkmark$ & 0.0\,$\checkmark$ & 0.5\,$\checkmark$ & 1.0\,$\checkmark$ \\
\bottomrule
\end{tabular}

\vspace{2pt}
\parbox{\linewidth}{\justifying\scriptsize $\checkmark$/$\times$: pre-registered criterion met/not met (Table~\ref{tab:criterion}). $^{*}$: BRIDGE has a significantly lower false-improvement rate than the rule; $^{\circ}$: no significant difference (paired test \eqref{eq:hdiff}, $\alpha=0.05$, Holm over the comparators of one collection; in (E), a single comparison, $p<0.0002$). No superscript: paired test not run. No mark: criterion not evaluated, because LOCO is not computed for the holdouts in (A). --: not run, because the program used for (E) records the decisions of the naive rule and BRIDGE. Accepted counts, bounds, LOCO, yield, and updates above the margin of each rule: Table~\ref{tab:rules}; for (E): Table~\ref{tab:summary}.}
\end{table}

\begin{table*}[!t]
\centering
\footnotesize
\setlength{\tabcolsep}{3.5pt}
\caption{Strongest comparators: false-improvement rate, yield, and updates above the margin, (A)--(D).}
\label{tab:strongest}
\begin{tabular}{@{}lrrrrrrrrrrrr@{}}
\toprule
 & \multicolumn{3}{c}{(A) 11 new datasets} & \multicolumn{3}{c}{(B) Synthetic} & \multicolumn{3}{c}{(C) Clinical} & \multicolumn{3}{c}{(D) Supporting} \\
\cmidrule(lr){2-4}\cmidrule(lr){5-7}\cmidrule(lr){8-10}\cmidrule(l){11-13}
Rule & False \% $\downarrow$ & Yield $\uparrow$ & ${>}\varepsilon$ $\uparrow$ & False \% $\downarrow$ & Yield $\uparrow$ & ${>}\varepsilon$ $\uparrow$ & False \% $\downarrow$ & Yield $\uparrow$ & ${>}\varepsilon$ $\uparrow$ & False \% $\downarrow$ & Yield $\uparrow$ & ${>}\varepsilon$ $\uparrow$ \\
\midrule
Corrected $t$, re-selection & 2.3$^{\circ}$\,$\checkmark$ & 31.7 & 29.8 & 3.0$^{\circ}$\,$\checkmark$ & 37.9 & 35.9 & 0.4$^{\circ}$\,$\checkmark$ & 58.4 & 56.4 & 1.5$^{*}$\,$\checkmark$ & 39.6 & 35.2 \\
Holdout 2/3 + 1/3 & 1.3 & 13.9 & 13.4 & 1.9\,$\checkmark$ & 18.6 & 18.0 & 0.0$^{\circ}$\,$\checkmark$ & 41.1 & 41.0 & 0.3$^{\circ}$\,$\checkmark$ & 19.5 & 18.8 \\
Holdout 1/2 + 1/2 & 3.2 & 17.3 & 16.2 & 2.1\,$\checkmark$ & 24.2 & 23.3 & 0.0$^{\circ}$\,$\checkmark$ & 46.0 & 46.0 & 0.8$^{\circ}$\,$\checkmark$ & 22.3 & 21.2 \\
Fixed margin, dataset held out$^{\dagger}$ & 5.8\,$\times$ & 31.5 & 28.7 & 4.2\,$\checkmark$ & 49.4 & 45.4 & -- & -- & -- & -- & -- & -- \\
\textbf{BRIDGE} & 1.1\,$\checkmark$ & 25.8 & 24.9 & 1.7\,$\checkmark$ & 33.9 & 32.8 & 0.0\,$\checkmark$ & 50.0 & 49.9 & 0.5\,$\checkmark$ & 30.1 & 29.1 \\
\bottomrule
\end{tabular}

\vspace{2pt}
\parbox{\textwidth}{\justifying\scriptsize False \%: false improvements among the accepted updates. Yield: accepted updates per $100$ decisions. ${>}\varepsilon$: accepted updates whose reference gain exceeds the margin $\varepsilon=0.005$, per $100$ decisions (post hoc; tests in Table~\ref{tab:falseimp}). $\checkmark$/$\times$: pre-registered criterion met/not met; no mark: LOCO not computed for the holdouts in (A). $^{*}$: BRIDGE has a significantly lower false-improvement rate; $^{\circ}$: no significant difference (paired cluster bootstrap, Holm over the comparators of a collection, Table~\ref{tab:rulesall}). $^{\dagger}$Post hoc: for each dataset (in (B), each block of $140$ tasks), the margin is chosen from the pre-registered grid $0.005$--$0.050$ on all other data, by the rule of the leave-one-setting-out analysis, and applied to the held-out dataset; the margin $0.020$ is chosen for $43$ of $48$ held-out units. No test is applied: the comparison is descriptive, because each held-out dataset's margin is chosen on data that include the other evaluated datasets, so the decisions are not independent of that choice. Not available for (C) and (D). Cost: BRIDGE re-scores stored predictions (about $0.03$~s per run; Table~\ref{tab:complexity}).}
\end{table*}

\begin{figure}[t]
\centering
\xfigwdcol
\includegraphics[width=0.95\linewidth]{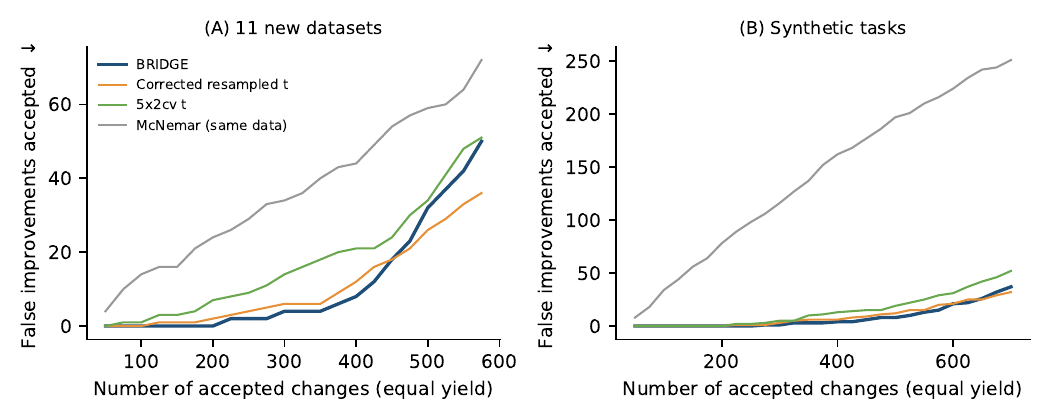}
\caption{Acceptance rules compared at matched acceptance counts on (A) and (B).}
\label{fig:yield}
\vspace{2pt}
\parbox{\linewidth}{\justifying\scriptsize \emph{Note.} The curves rank the decisions by each rule's own statistic ($-p$ for the tests, the bootstrap lower bound for BRIDGE), keep the top $N$ (ties broken by stored decision order), and count the false improvements among the selected trials; the metric is weighted-$F_1$. Runs on (A) and (B) store the decision statistic of every rule, so the curves use those collections. Runs on (C) and (D) store the accept/reject decision of each rule, and the program used for (E) records the decisions of the naive rule and BRIDGE. Current model: HGB with the default candidate space (HGB, DT, SVC). Other learners, other candidate members, and a joint seven-family space are tested separately (Tables~\ref{tab:members} and~\ref{tab:scope}).}
\end{figure}

On the $11$ new datasets of (A), the test-based rules fail the criterion, whereas BRIDGE meets it. The counts and bounds are in Table~\ref{tab:rules}, and Table~\ref{tab:rulesall} gives the false-improvement rates. These tests assume that the compared pair is fixed before the data are seen. In update acceptance, however, the candidate is chosen among many configurations on the same data that then judge it (Section~\ref{sec:rel-tests}). The failure is specific to (A). In (C) and (D), the $5\times2$~cv and corrected $t$-tests also meet the criterion. BRIDGE differs from the naive rule mainly on small validation gains. Of the $459$ changes of (A) with a validation gain of at most $0.01$, BRIDGE accepts none, whereas the naive rule accepts all $459$ and is wrong $164$ times (Table~\ref{tab:rules}).

The corrected $t$-test with per-split re-selection is the closest comparator, because it also separates choosing from judging. On (A), this test meets the criterion with a false-improvement rate of $2.3\%$, against $1.1\%$ for BRIDGE. In (A) and (B), the difference is not significant; in (A) it is $-1.2$ pp ($95\%$ CI $[-6.3;\,1.3]$, Holm-adjusted $p=0.32$; paired differences in Table~\ref{tab:ci}). In (A), the same-POOL cell is equivalent at matched yield by TOST (post hoc), and the fixed margin chosen post hoc does not differ significantly either.

The rules accept different numbers of updates, so the comparison is also made at matched acceptance counts on (A) and (B) (Fig.~\ref{fig:yield}). BRIDGE has an equal or lower false-improvement rate than the corrected $t$-test up to about $450$ accepted updates, and a slightly higher proportion beyond that point. The matched-count claim therefore covers that range (matched-yield details in Appendix~\ref{app:sens}).

A last comparator is an honest holdout on the same labeled data. This rule reserves part of POOL to test the chosen update, whereas BRIDGE reuses all of POOL through cross-fitting. On (A) to (D), BRIDGE accepts $1.1$ to $1.9$ times as many updates as the two holdout splits. BRIDGE deploys $1.1$ to $1.9$ times as many updates that gain more than the margin (post hoc, Table~\ref{tab:falseimp}, Holm $p\le0.01$), and the share of accepted updates that gain at most the margin does not differ significantly (accepted counts in Table~\ref{tab:holdout}). The false-improvement rate of BRIDGE is close to that of the holdouts. In (A), the false-improvement rate is $1.1\%$ for BRIDGE against $1.3\%$ and $3.2\%$ for the holdouts. In (B), it is $1.7\%$ against $1.9\%$ and $2.1\%$ (descriptive, without a paired test). In (C), neither BRIDGE nor the holdouts accept a false improvement. In (D), the difference between BRIDGE and either holdout is not significant (Table~\ref{tab:rulesall}).

\paragraph{Updates below the margin.} Two complementary measures use the margin $\varepsilon=0.005$, which was fixed during development (Table~\ref{tab:hyper}), and both are post hoc (Table~\ref{tab:falseimp}). The first measure is the share of accepted updates whose gain on the reference is at most $\varepsilon$, and the second is the number that gain more than $\varepsilon$ per $100$ decisions. Among the updates that the naive rule accepts, $33.9\%$ to $57.4\%$ gain at most $\varepsilon$, against $0.2\%$ to $6.3\%$ for BRIDGE, and the difference is significant in every collection (Holm $p\le0.005$). The two holdouts reach $13.4$ to $23.3$ in (A), (B), and (D) and $41.0$ to $46.0$ in (C), against $24.9$, $32.8$, $29.1$, and $49.9$ for BRIDGE, so BRIDGE deploys significantly more updates above the margin. No rule has a significantly smaller share below the margin than BRIDGE. In (A) and (B), the corrected $t$-test with per-split re-selection does not differ significantly from BRIDGE ($6.0\%$ and $5.3\%$ against $3.5\%$ and $3.4\%$), and neither do the two holdouts. BRIDGE keeps $34\%$ to $80\%$ of the above-margin updates that the naive rule accepts, from $24.9$ against $59.5$ per $100$ decisions in (A) to $49.9$ against $62.4$ in (C). The rest are rejected because the POOL evidence is not yet sufficient.

\paragraph{Strongest comparators.} Table~\ref{tab:strongest} lists the rules that also guard against selection. In (A) to (D), BRIDGE and the corrected $t$-test with re-selection meet the criterion, and the two holdouts deploy fewer updates above the margin than BRIDGE. A fixed margin selected without the evaluated dataset misses the criterion in (A), and its false-improvement rate in (B) is higher ($4.2\%$ against $1.7\%$); the comparison is post hoc and descriptive.

\subsection{Empirical Consistency with the Finite-Sample Bound (Synthetic Population $\mathcal{Q}$)}
\label{sec:exp-prop1}

\begin{table}[!t]
\centering
\scriptsize
\providecolor{bghead}{HTML}{DBEFF8}\providecolor{bgzebra}{HTML}{F2F9FC}\providecolor{bgpass}{HTML}{D7F0D2}\providecolor{bgfail}{HTML}{FCDCD7}
\setlength{\tabcolsep}{2.5pt}
\caption{Finite-sample bound on the synthetic tasks (B) for each acceptance rule.}
\label{tab:synthetic}
\setlength{\tymin}{28pt}\begin{tabulary}{\columnwidth}{@{}lccLcc@{}}
\toprule
Rule & Accepted & False & Bound \% $\downarrow$ & Yield/100 $\uparrow$ & Criterion \\
\midrule
Naive (validation improves) & 1,323 & 274 & 23.52 & 94.50 & $\times$ \\
McNemar on same POOL & 135 & 47 & 45.19 & 9.64 & $\times$ \\
McNemar on cross-fit & 50 & 13 & 42.65 & 3.57 & $\times$ \\
$5\times2$~cv paired $t$ & 610 & 34 & 8.14 & 43.57 & $\checkmark$ \\
Corrected resampled $t$ & 774 & 44 & 7.95 & 55.29 & $\checkmark$ \\
Corrected $t$, re-selection & 531 & 16 & 5.19 & 37.93 & $\checkmark$ \\
BRIDGE without margin & 598 & 21 & 5.67 & 42.71 & $\checkmark$ \\
\textbf{BRIDGE} & 475 & 8 & \textbf{3.56} & 33.93 & $\checkmark$ \\
\bottomrule
\end{tabulary}

\vspace{2pt}
\parbox{\linewidth}{\justifying\scriptsize *Setting (B): 1,400 synthetic tasks drawn independently from one population, one decision per task. Bound: Proposition 1. Conservative bounds: Table~\ref{tab:chronology}.  $\downarrow$/$\uparrow$ = smaller/larger is better; \textbf{bold} = best value per column, naive acceptance excluded; bound: one-sided 95\% upper bound on the false-improvement rate, the uncertainty measure of each rule; criterion: Table~\ref{tab:criterion}.}
\end{table}

The finite-sample bound of Proposition~\ref{prop:bound} is checked empirically on (B). This collection is a synthetic population of $1{,}400$ binary tasks drawn independently from one generative population, with reference-measured gain. Each task has one trial and a $100{,}000$-row reference, which is large enough that the error of the reference truth is about a quarter of the margin. The generator and the basis of these sizes are in Appendix~\ref{app:datasets}. The population is designed so that the naive rule is wrong about as often as on the real data: $20.7\%$ of its accepted updates are false improvements, against $19.1\%$ in (A) (Table~\ref{tab:synthetic}). BRIDGE accepts $475$ updates, of which $8$ are false improvements, so $\hat r=1.7\%$. The Proposition~\ref{prop:bound} bound on $r_{\mathcal{Q}}$ \eqref{eq:rq} is $U=3.56\%$ (Table~\ref{tab:synthetic}).

The bound is also computed on a fresh batch of $1{,}400$ tasks that no earlier run used. With the same frozen gate, BRIDGE accepts $462$ updates with $7$ false improvements and gives $U=3.35\%$ (base-learner results on $\mathcal{Q}$ in Table~\ref{tab:basefull}, row HGB).

A consistency check of the bound needs two batches, because within one batch $U\ge\hat r=S/M$ holds by construction \eqref{eq:U}. The bound of a batch therefore lies above the estimate from the same batch, and that comparison carries no evidence about the bound. The check is therefore made across the two independent batches. Proposition~\ref{prop:bound} gives $\Pr(r_{\mathcal{Q}}\le U^{(b)})\ge 1-\delta$ for each batch $b$, so the bound of one batch should lie above the false-improvement rate estimated on the other, $\hat r^{(2)}\le U^{(1)}$ and $\hat r^{(1)}\le U^{(2)}$. Both inequalities hold. The bound of the primary batch ($3.56\%$) lies above the estimate of the fresh batch ($1.5\%$). The bound of the fresh batch ($3.35\%$) lies above the estimate of the primary batch ($1.7\%$; Tables~\ref{tab:synthetic} and~\ref{tab:basefull}). Both bounds also lie above the upper end of the $95\%$ Clopper--Pearson interval \cite{clopper1934} for $r_{\mathcal{Q}}$ from the two batches together. That interval is $0.9\%$ to $2.6\%$ around an estimate of $1.6\%$ ($15$ of $937$). The $95\%$ coverage itself rests on Proposition~\ref{prop:bound}; the cross-batch comparison is a consistency check, not a validation of coverage.

\subsection{Clinical Application (Retrospective)}
\label{sec:exp-clinical}

The clinical application is retrospective: the data come from existing records, not from a prospective trial, so the results show what would happen if the gate were used on that data (prospective validation: Section~\ref{sec:limitations}). It is an illustration with one hypothetical action and one cost ratio at $p_t=0.2$ for all tasks, not evidence of practical clinical benefit. Table~\ref{tab:tasks} describes the $12$ clinical tasks, which come from $6$ independent sources. The sources are screening mammography from OpenML \cite{vanschoren2013}, an HIV-1 trial cohort from the UCI repository \cite{dua2019}, and self-reported diagnoses from four NHANES \cite{cdcnhanes2018} and four BRFSS \cite{cdcbrfss2015} surveys. Two further sources are a national COVID-19 cohort from the open data of the Mexican Ministry of Health \cite{covidmx2020} and an intensive care unit (ICU) admissions cohort from the WiDS Datathon 2020 \cite{wids2020}.

Survey tasks are health-related cross-sectional classification tasks rather than prognosis tasks, and the predictors are restricted to information available at decision time. One hypothetical clinical action, such as a referral, is paired with each task, and the net-benefit calculation uses this action. The patient-identity audit of the sources (de-duplication and per-person filtering before any split) is reported in Table~\ref{tab:identity}.

\begin{table}[!t]
\centering
\scriptsize
\providecolor{bghead}{HTML}{DBEFF8}\providecolor{bgzebra}{HTML}{F2F9FC}\providecolor{bgpass}{HTML}{D7F0D2}\providecolor{bgfail}{HTML}{FCDCD7}
\setlength{\tabcolsep}{1.5pt}
\caption{Clinical application (C): results by metric.}
\label{tab:clinical}
\setlength{\tymin}{28pt}\begin{tabulary}{\columnwidth}{@{}lccccc@{}}
\toprule
Metric & \begin{tabular}[b]{@{}c@{}}Naive\\false impr. \% $\downarrow$\end{tabular} & \begin{tabular}[b]{@{}c@{}}BRIDGE\\accepted\end{tabular} & \begin{tabular}[b]{@{}c@{}}BRIDGE\\false impr. \% $\downarrow$\end{tabular} & Bound \% $\downarrow$ & \begin{tabular}[b]{@{}c@{}}LOCO-\\source \% $\downarrow$\end{tabular} \\
\midrule
\multicolumn{6}{@{}l}{\emph{Co-primary (pre-registered)}} \\
Net benefit, $p_t=0.2$ & 20.0 & 1,199 & 0.0 & 0.33 & 0.50 \\
Balanced accuracy & 0.3 & 2,273 & 0.4 & 0.82 & 1.25 \\
\multicolumn{6}{@{}l}{\emph{Secondary}} \\
Weighted $F_1$ & 23.6 & 435 & 2.1 & 5.12 & 7.84 \\
$F_2$ & 0.0 & 2,379 & 0.0 & 0.16 & 0.24 \\
Net benefit, $p_t=0.1$ & 3.9 & 1,660 & 0.0 & 0.23 & 0.37 \\
\bottomrule
\end{tabulary}

\vspace{2pt}
\parbox{\linewidth}{\justifying\scriptsize *Setting: $12$ tasks from $6$ independent sources, $200$ trials per task (2,400 decisions per metric). Co-primary metrics, fixed in the pre-registration: net benefit at the clinical threshold $p_t=0.2$ and balanced accuracy. Net benefit at the lower threshold $p_t=0.1$ is a secondary sensitivity check. Bound and LOCO-source of BRIDGE. Clinical consequences, in extra true and false positives: Table~\ref{tab:clinconseq}. Criterion: Table~\ref{tab:criterion}.  $\downarrow$/$\uparrow$ = smaller/larger is better; bound: one-sided 95\% upper bound on the false-improvement rate, the uncertainty measure of each rule; criterion: Table~\ref{tab:criterion}.}
\end{table}

The confirmatory clinical run uses $200$ trials per task and two co-primary metrics fixed in the pre-registration. The first is net benefit at $p_t=0.2$ \eqref{eq:netbenefit}, under which one true positive is worth four false positives because the odds $0.2/0.8$ equal $1/4$ \cite{vickers2006}, and the second is balanced accuracy. Weighted-$F_1$, $F_2$, and net benefit at the lower threshold $p_t=0.1$ are secondary metrics. Table~\ref{tab:clinical} reports these metrics, and the pre-registration chronology is in Appendix~\ref{app:prereg}. In the clinical application and the binary collections, the baseline is the pre-registered HGB at threshold $0.5$, and the candidate space contains thresholds from $0.05$ to $0.95$. A shift toward the clinical threshold is therefore a legitimate update, and part of the gain comes from correcting a baseline that was not set for the clinical objective.

Under net benefit at $p_t=0.2$, $20.0\%$ of the updates accepted by the naive rule are false improvements, whereas none of the $1{,}199$ updates accepted by BRIDGE is. The reduction is $20.0$ percentage points, with a $95\%$ interval of $4.5$ to $40.3$ over tasks and $0.0$ to $31.9$ over the six sources. The task is the pre-registered unit of analysis; the four BRFSS tasks and the four NHANES tasks each share respondents, so the per-source analysis is reported as a sensitivity and is the conservative reading. Over the six sources, the reduction does not differ significantly ($p=0.17$, interval up to $0.0$). The false improvements of the naive rule lie only in BRFSS ($282$) and NHANES ($171$), whereas BRIDGE accepts none of its $1{,}199$ updates as false improvements. The share of updates below the margin still differs significantly over sources ($p=0.03$, post hoc, one unadjusted test). Three quantities answer different questions: the KL-ratio bound ($0.33\%$) is conditional on the collection with the trial as unit, LOCO-source ($0.50\%$) leaves one source out at a time, and the source-level test ($p=0.17$) is the inference across sources. Under balanced accuracy, $F_2$, and net benefit at $p_t=0.1$, the naive rule rarely accepts a false improvement ($0.0\%$ to $3.9\%$), so the gate has little effect on the result. Weighted-$F_1$ again separates the rules ($23.6\%$ against $2.1\%$; Table~\ref{tab:clinical}). On BRFSS kidney disease, all $133$ updates accepted by the naive rule are false improvements, whereas BRIDGE accepts none, and the same pattern holds on BRFSS stroke ($144$ of $151$) and NHANES coronary heart disease ($97$ of $182$; Appendix~\ref{app:identity}, Table~\ref{tab:clinsource}).

Per $1{,}000$ patients, the updates that BRIDGE accepts add $90.9$ true positives and $197.0$ false positives, so sensitivity rises and specificity falls, and the workload increases from $98$ to $386$ patients flagged per $1{,}000$ (Table~\ref{tab:clinconseq}). The net-benefit run was pre-registered after an earlier weighted-$F_1$ run on the same tasks failed the LOCO criterion, so it is a re-test with fresh seeds, not an independent clinical confirmation (Appendix~\ref{app:prereg}).

\subsection{Component Analysis: Ablation and Candidate-Space Size}
\label{sec:exp-ablation}

\begin{table*}[!t]
\centering
\scriptsize
\providecolor{bghead}{HTML}{DBEFF8}\providecolor{bgzebra}{HTML}{F2F9FC}\providecolor{bgpass}{HTML}{D7F0D2}\providecolor{bgfail}{HTML}{FCDCD7}
\setlength{\tabcolsep}{3pt}
\caption{Component analysis on (A)--(E): ablation of the three components of the gate.}
\label{tab:ablation}
\begin{tabular}{@{}llllrrrrrrrrrr@{}}
\toprule
\multirow{2}{*}{Known as} & \multirow{2}{*}{Cross-fit} & \multirow{2}{*}{\begin{tabular}{@{}c@{}}Bootstrap\\bound\end{tabular}} & \multirow{2}{*}{Margin} & \multicolumn{2}{c}{(A)} & \multicolumn{2}{c}{(B)} & \multicolumn{2}{c}{(C)} & \multicolumn{2}{c}{(D)} & \multicolumn{2}{c}{(E)} \\
\cmidrule(lr){5-6} \cmidrule(lr){7-8} \cmidrule(lr){9-10} \cmidrule(lr){11-12} \cmidrule(lr){13-14}
 & & & & \begin{tabular}[b]{@{}c@{}}False\\\% $\downarrow$\end{tabular} & \begin{tabular}[b]{@{}c@{}}Bound\\\% $\downarrow$\end{tabular} & \begin{tabular}[b]{@{}c@{}}False\\\% $\downarrow$\end{tabular} & \begin{tabular}[b]{@{}c@{}}Bound\\\% $\downarrow$\end{tabular} & \begin{tabular}[b]{@{}c@{}}False\\\% $\downarrow$\end{tabular} & \begin{tabular}[b]{@{}c@{}}Bound\\\% $\downarrow$\end{tabular} & \begin{tabular}[b]{@{}c@{}}False\\\% $\downarrow$\end{tabular} & \begin{tabular}[b]{@{}c@{}}Bound\\\% $\downarrow$\end{tabular} & \begin{tabular}[b]{@{}c@{}}False\\\% $\downarrow$\end{tabular} & \begin{tabular}[b]{@{}c@{}}Bound\\\% $\downarrow$\end{tabular} \\
\midrule
Naive & -- & -- & -- & 19.08 & 22.90\,$\times$ & 20.71 & 23.52\,$\times$ & 19.99 & 22.68\,$\times$ & 12.49 & 14.19\,$\times$ & 27.65 & 32.13\,$\times$ \\
Fixed margin & -- & -- & $\checkmark$ & 10.20 & 14.17\,$\times$ & 16.42 & 19.19\,$\times$ & 2.65 & 4.05\,$\checkmark$ & 6.18 & 7.68\,$\checkmark$ & 16.74 & 22.48\,$\times$ \\
Same-POOL bootstrap & -- & $\checkmark$ & -- & 6.96 & 10.70\,$\times$ & 6.61 & 8.95\,$\checkmark$ & 2.67 & 4.06\,$\checkmark$ & 4.87 & 6.33\,$\checkmark$ & 13.73 & 19.16\,$\times$ \\
Same-POOL cell & -- & $\checkmark$ & $\checkmark$ & 2.65 & 5.85\,$\checkmark$ & 3.18 & 5.14\,$\checkmark$ & 0.00 & 0.31\,$\checkmark$ & 1.08 & 2.06\,$\checkmark$ & 3.37 & 8.17\,$\times^{\ddagger}$ \\
Cross-fit, point & $\checkmark$ & -- & -- & 10.81 & 15.00\,$\times$ & 12.07 & 14.72\,$\times$ & 7.69 & 9.84\,$\times$ & 7.28 & 8.89\,$\checkmark$ & 18.98 & 24.82\,$\times$ \\
Cross-fit, point, margin & $\checkmark$ & -- & $\checkmark$ & 5.23 & 9.00\,$\times$ & 9.49 & 12.06\,$\times$ & 0.91 & 1.86\,$\checkmark$ & 3.76 & 5.19\,$\checkmark$ & 9.11 & 14.86\,$\times$ \\
BRIDGE without margin & $\checkmark$ & $\checkmark$ & -- & 2.78 & 6.16\,$\checkmark$ & 3.51 & 5.67\,$\checkmark$ & 0.56 & 1.34\,$\checkmark$ & 1.64 & 2.75\,$\checkmark$ & 8.86 & 15.63\,$\times$ \\
\textbf{BRIDGE} & $\checkmark$ & $\checkmark$ & $\checkmark$ & 1.06 & 4.23\,$\checkmark$ & 1.68 & 3.56\,$\checkmark$ & 0.00 & 0.33\,$\checkmark$ & 0.54 & 1.47\,$\checkmark$ & 1.17 & 6.13\,$\times^{\ddagger}$ \\
\bottomrule
\end{tabular}

\vspace{2pt}
\parbox{\linewidth}{\justifying\scriptsize *$\checkmark$/$\times$ after the bound: pre-registered criterion met or not met. (E): the $11$ binary datasets of the OpenML collection, with all eight cells from MINI-52 (run list in Table~\ref{tab:chronology2}). $^{\ddagger}$: bound below $10\%$ but LOCO above it (same-POOL cell $11.95\%$, BRIDGE $10.30\%$). On the full $14$ datasets of (E), the four cells computable from the released program give $1.04\%$ for BRIDGE (bound $5.40\%$, met) and $7.89\%$ for BRIDGE without the margin (bound $13.84\%$, not met). Each component switched on ($\checkmark$) or off (--): cross-fitted evidence instead of the same POOL; the 5\% bootstrap lower bound instead of the point estimate; the margin 0.005 instead of 0. Same-POOL cell: the rule that chooses and measures on the same POOL with the bootstrap bound and the margin. Every cell judges the same deployed configuration; current model HGB with the default candidate space. Bound: KL-ratio bound in (A), (C), and (D), and Proposition 1 bound in (B). Paired intervals for the same-POOL cell against BRIDGE: Table~\ref{tab:ci}. Counts, LOCO, and yield of each cell are recorded in the code repository (runs MINI-32 and MINI-48, Table~\ref{tab:chronology}). Criterion: Table~\ref{tab:criterion}. An arrow $\downarrow$ or $\uparrow$ marks a column where a smaller or a larger value is better.}
\end{table*}

\paragraph{Ablation.} The ablation switches each of the three gate components on and off while the deployed configuration stays fixed. The components are the evidence (cross-fitted counts or the same POOL), the decision statistic (bootstrap lower bound or point estimate), and the margin ($\varepsilon=0.005$ or $0$). Table~\ref{tab:ablation} lists the eight resulting cells for (A) to (E). For (E), the cells cover the $11$ binary datasets, which were rerun with every cell recorded (run MINI-52 in Table~\ref{tab:chronology2}).

The table first shows which component lowers the false-improvement rate most when it is added to the naive rule alone, with the other two components switched off. The bootstrap lower bound alone gives the largest reduction in (A), (B), (D), and (E), at $7.6$ to $14.1$ percentage points below the naive rule. In (C), it ties with the margin alone ($17.3$ points each). The margin alone comes second in (A), (D), and (E), whereas cross-fitting alone comes second in (B). The bootstrap lower bound and the margin each lower the false-improvement rate at a cost in yield. Each of the two rejects updates that the naive rule would accept.

The binary datasets of (E), where every cell was recorded, follow the same ordering at a higher false-improvement rate. There, $27.7\%$ of the updates accepted by the naive rule are false improvements. Of the eight cells, BRIDGE has the lowest false-improvement rate ($1.2\%$, bound $6.1\%$) and comes closest to the criterion. Its LOCO bound is $10.3\%$, slightly above the $10\%$ limit. BRIDGE without the margin reaches $8.9\%$ (bound $15.6\%$), so no cell meets the criterion on these datasets (Table~\ref{tab:ablation}). On all $14$ datasets of (E), BRIDGE meets the criterion, with a false-improvement rate of $1.0\%$ and a bound of $5.4\%$ (Table~\ref{tab:summary}).

Cross-fitting is added on top of the bootstrap lower bound and the margin, and it has a small effect on the false-improvement rate in (A) and (B). In (A), BRIDGE and the same-POOL cell with the bootstrap lower bound and the margin do not differ significantly. The difference is $-1.6$ pp ($95\%$ CI $[-5.1;\,0.2]$). At matched yield, the two rules are equivalent by two one-sided tests (TOST) \cite{schuirmann1987} of \eqref{eq:hequiv}: $-0.4$ pp, $95\%$ CI $[-1.9;\,1.0]$, inside $\pm2$ pp (post hoc). In (B), BRIDGE is $1.5$ pp lower than the same-POOL cell at the number of updates that each rule accepts ($95\%$ CI $[-2.7;\,-0.5]$). At matched yield, BRIDGE is $1.1$ pp higher ($95\%$ CI $[0.2;\,2.1]$; Table~\ref{tab:ci}). The two comparisons disagree because the same-POOL cell accepts $660$ updates and BRIDGE accepts $475$ (paired differences in Table~\ref{tab:ci}). The updates accepted by the same-POOL cell beyond that count carry most of its extra false improvements. In (E), by contrast, the same-POOL cell reaches a false-improvement rate of $3.4\%$ against $1.2\%$ for BRIDGE. The same-POOL cell also fails on LOCO there ($12.0\%$; Table~\ref{tab:kspace}).

BRIDGE keeps cross-fitting because cross-fitting separates the data that select an update from the data that judge it. Through this separation, the decision statistic is out-of-sample with respect to selection (Section~\ref{sec:method-gate}). The naive estimate of the gain on the full POOL is optimistic by $+0.0034$ on average over the $1{,}064$ decisions of (A) with a selected candidate. By contrast, the cross-fitted estimate is slightly pessimistic ($-0.0034$). Both errors amount to about two-thirds of the margin (Fig.~\ref{fig:crossfit}; descriptive run MINI-27 in Table~\ref{tab:citedruns}). Replacing cross-fitting with a cell found from these results would be selecting on the outcome, the practice this paper targets. The bound of Proposition~\ref{prop:bound} holds for any frozen gate, whether or not that gate uses cross-fitting. Two pre-registered predictions differ from the results. The bootstrap lower bound lowers the false-improvement rate more than cross-fitting does. On (A) to (D), the same-POOL cell also meets the criterion at the large candidate space (predictions in Appendix~\ref{app:prereg}).

\begin{figure}[t]
\centering
\xfigwdcol
\includegraphics[width=0.95\linewidth]{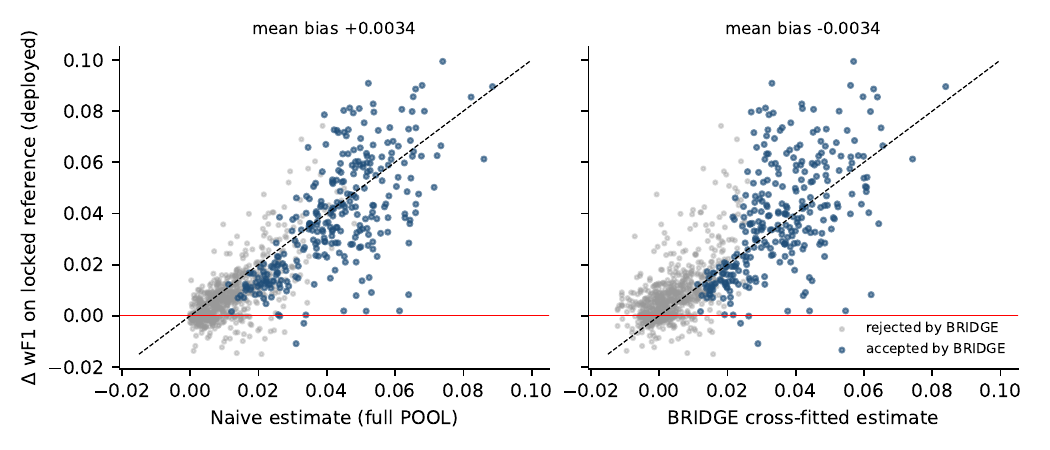}
\caption{Component analysis: naive and cross-fitted gain estimates against the reference truth on (A).}
\label{fig:crossfit}
\vspace{2pt}
\parbox{\linewidth}{\justifying\scriptsize \emph{Note.} Each point is one of the $1{,}064$ decisions of (A) with a selected candidate. The dashed line marks equality of estimate and reference truth. The naive estimate is measured on the same POOL that chose the update, and it is optimistic on average ($+0.0034$). The cross-fitted estimate, measured on folds not used to choose the update, is slightly pessimistic ($-0.0034$). Both errors amount to about two-thirds of the margin $\varepsilon=0.005$. Blue marks updates accepted by BRIDGE, and gray marks rejected updates.}
\end{figure}

\begin{figure*}[t]
\centering
\xfigwdtext
\includegraphics[width=0.95\textwidth]{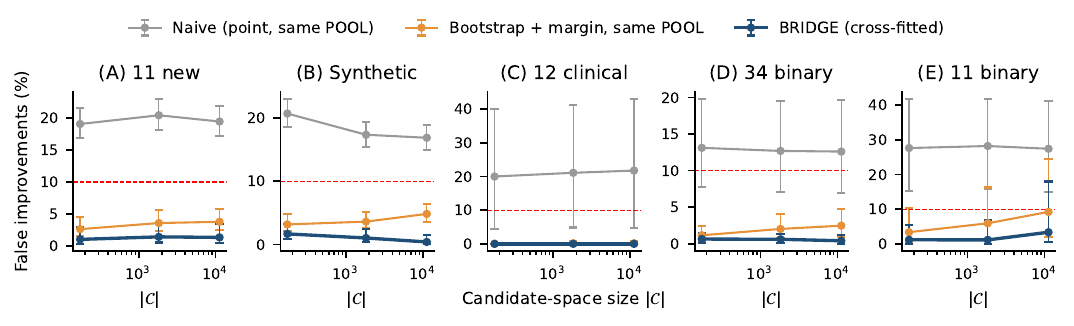}
\caption{Component analysis: false-improvement rate against candidate-space size for (A)--(E).}
\label{fig:kspace}
\vspace{2pt}
\parbox{\textwidth}{\justifying\scriptsize \emph{Note.} Error bars show $95\%$ Wilson intervals \cite{wilson1927}; the dashed line marks the $10\%$ criterion. (E): the $11$ binary datasets. Table~\ref{tab:kspace} gives the grids, the counts, and the datasets of each panel (Appendix~\ref{app:results}).}
\end{figure*}

\paragraph{Candidate-space size.} The second part of the analysis tests whether a larger candidate space increases the false-improvement rate, as the selection bias of \eqref{eq:maxbias} predicts. On (A) to (E), the same grids are refined from $173$ to $1{,}825$ and $11{,}249$ binary configurations. HGB serves as the current model, with HGB, DT, and SVC as the members (results by candidate-space size in Table~\ref{tab:kspace}, Fig.~\ref{fig:kspace}). The multiclass datasets of (D) and (E) are left out of this analysis, because no equally dense multiclass grid is pre-registered. As the candidate space grows, the naive rule fails the criterion in every cell. Its false-improvement rate lies between $12.6\%$ and $28.2\%$ (Table~\ref{tab:kspace}). On (A) to (D), BRIDGE meets the criterion at every candidate-space size. Its false-improvement rate stays at or below $1.7\%$ (Fig.~\ref{fig:kspace}). The estimates for BRIDGE do not fall steadily with size (in (A), $1.06\%$, $1.45\%$, and $1.37\%$; Table~\ref{tab:kspace}). On (A) to (D), therefore, a finer grid brings no material increase in false improvements, and the results do not show that a larger space lowers the false-improvement rate.

The binary datasets of (E) mark the range within which this finding holds. On those datasets, BRIDGE meets the criterion at $1{,}825$ configurations. As reported for the ablation, BRIDGE narrowly misses the LOCO limit at $173$. At $11{,}249$ configurations the false-improvement rate rises to $3.4\%$ ($6$ of $179$ accepted updates; bound $10.3\%$). This value lies $2.2$ percentage points above that of the default space ($95\%$ CI $0.04$ to $12.8$; Table~\ref{tab:kspace}). On the reference, each of the six false improvements loses at most $0.006$ in weighted-$F_1$. The same-POOL cell passes at every size on (A) to (D). Its false-improvement rate nevertheless grows in (D) from $1.2\%$ to $2.5\%$ and in (E) from $3.4\%$ to $9.2\%$. This growth is consistent with the selection bias of \eqref{eq:maxbias}, which grows with the number of candidates. The default space of $173$ configurations therefore remains the tested condition of use, confirmed on (A) to (D) and on all $14$ datasets of (E). A denser grid needs its own confirmation before use (Section~\ref{sec:limitations}).

\subsection{Scope and Conditions of Use}
\label{sec:exp-base}

\paragraph{Conditions of use.} A condition of use is a condition, such as the metric, the training size, the number of classes, the learner, the tuning, or the imbalance, that has been tested and meets the criterion. The claim holds for each condition tested, one at a time, as listed in Table~\ref{tab:scope}; combinations of conditions that were not tested together lie outside the claim.

\paragraph{Metrics, training sizes, multiclass, and feature selection.} These four conditions are tested on a separate candidate space of thresholds for KNN and SVC. Each condition is tested with and without a mutual-information feature-selection filter \cite{kraskov2004} fitted on SEARCH alone (protocol in Appendix~\ref{app:prereg}). The gate meets the criterion for every metric tested (weighted-$F_1$, balanced accuracy, macro-$F_1$) and every training size ($1{,}000$, $2{,}430$, and $6{,}000$ rows). The gate also meets the criterion for the multiclass tasks (results by metric, size, and class count in Table~\ref{tab:general}). The gate also meets the criterion on the $39$ supporting datasets of (D) taken together (detailed results in Table~\ref{tab:summary}). Feature selection with SVC is the one exception, and Table~\ref{tab:scope} marks it. The OpenML collection (E) meets the criterion jointly with its extension (E2) (subgroup results in Appendix~\ref{app:settingE}).

\paragraph{Ensemble members.} Besides DT and SVC, two further pairs of ensemble members were tested. The first pairs logistic regression with naive Bayes, and the second pairs a multilayer perceptron (MLP) with KNN. With both pairs the gate meets the criterion in (A) to (D). A joint space of seven families, built with the same grids, meets the criterion in all five collections ($711$ binary configurations; composition in Appendix~\ref{app:sens}; Table~\ref{tab:members}). On the OpenML collection (E), both pairs miss the criterion. Logistic regression with naive Bayes gives a bound of $7.35\%$ and a LOCO of $10.16\%$, and the MLP with KNN gives a bound of $12.94\%$ and a LOCO of $20.77\%$ (Tables~\ref{tab:members} and~\ref{tab:citedruns}). A post hoc check identifies the cause. Each pair accepts $6$ false improvements, against $2$ for DT and SVC, and the MLP with KNN accepts fewer updates in total ($147$ against $192$). Of the $12$ false improvements, $11$ lie on datasets with references of $1{,}125$ to $2{,}112$ rows, where the reference interval of the gain contains zero (Table~\ref{tab:citedruns}). On (D) the joint space is practically equivalent to the default space, while on (A) it is not. The gate meets the criterion with every member set tested in (A) to (D), so the result does not rest on the HGB--DT--SVC members alone.

The default space is retained for three reasons. First, it was frozen during development, before the confirmatory runs; adopting the joint space after these results would be selecting on the outcome, the practice this paper targets. Second, a larger space increases the selection bias of \eqref{eq:maxbias}: in (A), the false-improvement rate of the naive rule rises from $19.1\%$ to $22.4\%$, while BRIDGE meets the criterion with the joint space ($1.6\%$ in (A), $0.7\%$ in (D); Tables~\ref{tab:members} and~\ref{tab:citedruns}). Third, the joint space is a candidate for the next development cycle, to be developed first and confirmed on data not used before (Section~\ref{sec:limitations}). DT and SVC were fixed during development, before this study.

\paragraph{Base learners.} The base learner was varied one factor at a time on fresh synthetic tasks (base-learner results on $\mathcal{Q}$ in Table~\ref{tab:basefull}) and on the real data (base-learner results on real data in Table~\ref{tab:basereal}). XGBoost \cite{chen2016} and an MLP under early stopping \cite{prechelt1998} meet the criterion in all five collections (Table~\ref{tab:basereal}). The condition of use is therefore a base learner with iteration control, and Table~\ref{tab:scope} marks the variant without iteration control. The gate criterion and the ranking of absolute final quality are different questions: an MLP with early stopping has a lower final weighted-$F_1$ than HGB in (D) and yet meets the criterion (final model quality in Table~\ref{tab:absolutequality}).

The base-learner results can also be read through the skill of a base learner. Skill is $1$ minus the ratio of the log-loss of the learner to the log-loss of a forecast based on the class prevalence \cite{gneitingraftery2007}. Skill is computed on POOL, without reference data, so it is known before the gate is used. Across all $11$ base-learner settings tested on $\mathcal{Q}$, higher skill is associated with a lower gate bound. The Spearman rank correlation \cite{spearman1904}, which runs from $-1$ to $+1$, is $\rho=-0.74$ ($p=0.010$) over the full set and $\rho=-0.82$ for the nine HGB and XGBoost settings (run-level numbers in Table~\ref{tab:citedruns}). The association has an exception and is not a decision rule: the MLP trained for $300$ epochs without early stopping has a skill of $-0.67$, below every other setting, yet its bound of $4.1\%$ meets the criterion (Tables~\ref{tab:basefull} and~\ref{tab:audit}). The association is descriptive rather than causal, because the settings differ in more than skill.

\paragraph{Tuning as an update.} Tuning is treated as an update proposed to the gate rather than as a way to set the baseline. During development, HGB was tuned with Optuna \cite{akiba2019} using its tree-structured Parzen estimator sampler \cite{bergstra2011}. In the confirmatory runs, four learners (HGB, DT, SVC, KNN) were tuned in (A), in (B), and on three multiclass datasets of (D) (tuning results in Table~\ref{tab:tuning}). BRIDGE lowers the false-improvement rate in all $12$ cells tested, and it meets the formal criterion in two of the $12$ cells (tuning results in Table~\ref{tab:tuning}; status in Table~\ref{tab:scope}). Tuning can therefore be proposed to the gate as an update. The default procedure nevertheless retains the fixed configuration.

Tuning the baseline in SEARCH raises its weighted-$F_1$ by about $0.3$ points (Table~\ref{tab:tuningtests}), so keeping the fixed configuration has a small cost, and the final model quality after the gate is practically unchanged. The naive rule accepts false improvements more often ($23.2\%$ in (A), $15.3\%$ in (D), against $19.1\%$ and $12.5\%$ without tuning), and BRIDGE again lowers the false-improvement rate by $14$--$21$ pp (Table~\ref{tab:tuningtests}).

\paragraph{Imbalanced classes.} Five large binary datasets were re-sampled to minority prevalences of $20\%$, $10\%$, and $5\%$ before the outer split. POOL then holds about $360$, $180$, and $90$ minority cases, which are $20\%$, $10\%$, and $5\%$ of the $1{,}800$ POOL rows (imbalance results in Table~\ref{tab:imbalance}). Under balanced accuracy and macro-$F_1$ the gate meets the criterion at every level down to $5\%$ (imbalance results in Table~\ref{tab:imbalance}). On imbalanced data, the condition of use is therefore a minority-sensitive metric, as Table~\ref{tab:scope} marks.

\paragraph{Parameter sensitivity and cost.} In the parameter sensitivity analysis, the criterion is met on $\mathcal{Q}$ when each of the fold count, the bootstrap count, and the bag count is varied in turn. The combined $q\times\varepsilon$ grid meets the criterion in all $24$ real-data cells ($3$ quantile levels $\times$ $4$ margins in (A) and in (D); quantile-margin grid in Table~\ref{tab:qepsreal}). Other margins give a false-improvement rate equivalent to the default margin $0.005$ in (A) to (D) (margin comparison in Table~\ref{tab:marginpaired}). Variants of the bag fraction and of the SEARCH size pass in (D) and partly in (A), as Table~\ref{tab:scope} marks (details in Appendix~\ref{app:sens}). The gate's own cost is about $0.03$~s; the total cost is dominated by fitting the bagged base models, and the $16$-bag setting cuts this fitting cost by about half (cost details in Appendix~\ref{app:sens}).

\FloatBarrier
\section{Conclusion}
\label{sec:conclusion}

This paper asked how often a model update that raises the validation score is a false improvement on data that did not choose it. In response, the paper proposed BRIDGE, an acceptance gate that decides each update from the labels already available, with settings fixed once. BRIDGE lets one part of the labeled data choose the update and another part judge it (cross-fitting). The gate accepts the update when the lower bootstrap bound of its gain exceeds a margin.

Across every collection tested, BRIDGE lowers the false-improvement rate among the updates it accepts. In the four main settings, the false-improvement rate falls from $19.1\%$ to $1.1\%$ in (A), from $20.7\%$ to $1.7\%$ in (B), from $20.0\%$ to $0\%$ in (C), and from $12.5\%$ to $0.5\%$ in (D) (Table~\ref{tab:headline}). The OpenML collection (E) shows the same pattern, with the false-improvement rate falling from $25.6\%$ to $1.0\%$ (Table~\ref{tab:headline}). The hypothesis tests reject equal false-improvement rates against the naive rule in all five collections, with (C) significant at the task level and not at the coarser source level. A random draw of the same size from the naive acceptances also contains more false improvements than the updates BRIDGE accepts (Section~\ref{sec:exp-tests}). On the synthetic population, the finite-sample bound gives a $95\%$ upper bound of $3.56\%$ on the false-improvement rate of BRIDGE. A second batch of $1{,}400$ tasks gives $3.35\%$ (Section~\ref{sec:exp-prop1}).

BRIDGE removes $97\%$ to $100\%$ of the false improvements that the naive rule accepts, and its mean reference gain per deployed update is $1.8$ to $3.1$ times that of the naive rule (Table~\ref{tab:selectivity}). Against an honest holdout on the same labeled data, BRIDGE deploys $1.1$ to $1.9$ times as many updates that gain more than the margin, and the share that gains at most the margin does not differ significantly (post hoc, Table~\ref{tab:falseimp}; Table~\ref{tab:holdout}).

The finding also holds beyond the default learner, metric, and ensemble members. With XGBoost, with a multilayer perceptron under early stopping, and with the joint seven-family candidate space, the pre-registered criterion is met in all five collections. The criterion is also met in (A) to (D) with both alternative member pairs, and on imbalanced data with balanced accuracy and macro-$F_1$ down to a minority prevalence of $5\%$ (Table~\ref{tab:scope}). Among the three components of the gate, the bootstrap lower bound reduces the false-improvement rate more than the other two components, and it ties with the margin in (C) (Section~\ref{sec:exp-ablation}). For practice, these results support BRIDGE as an acceptance rule for tabular model updates within the tested conditions. In the retrospective clinical tasks, the gate kept updates that merely appeared better from reaching the deployed model. The open-source package \texttt{bridge-gate} (contribution C4) runs the procedure with the same settings and reproduces the stored results.

\subsection{Limitations and Future Work} \label{sec:limitations} The claims hold within the tested conditions listed in Table~\ref{tab:scope}. Every collection as a whole meets the pre-registered criterion. The box states the population to which each bound applies. The paragraphs that follow give each limit with its reason and the next step.

\noindent\fbox{\parbox{0.96\columnwidth}{\small \textbf{What is bounded, and for which population?}\\[2pt]
(a) The formal bound (Proposition~\ref{prop:bound}) on $r_{\mathcal{Q}}$ \eqref{eq:rq} applies to the specified i.i.d.\ synthetic task population $\mathcal{Q}$ and refers to that population, not to a single deployment.
(b) On real data the primary bound is the KL-ratio bound, conditional on the observed collection and the frozen outer partition (the inner SEARCH/POOL split and the bag draw are re-drawn per trial).
(c) The dataset-level bootstrap bound and LOCO are sensitivity analyses across datasets and sources, not the primary bound.
(d) LOCO measures sensitivity across sources, not external validation in a new institution.
(e) ``Ground truth'' means the locked reference, which is itself finite and imperfect; $H_j$ is the false-improvement indicator measured on that finite reference of task $j$.
(f) The Proposition~\ref{prop:bound} bound, like the empirical risk $\hat r$, concerns the false-improvement rate among the \emph{deployed} configurations measured on the locked reference (Definition~\ref{def:risk}). The bootstrap quantile inside the gate is a decision rule, not a confidence interval, so no $95\%$ coverage is claimed for it.}}\label{box:scope}

\paragraph{Comparators.} The best fixed validation margin was chosen after the results were seen. Its false-improvement rate does not differ significantly from that of BRIDGE, except at the smallest training size. At matched acceptance counts, BRIDGE retains its advantage between about $125$ and $450$ accepted updates (Appendix~\ref{app:sens}). BRIDGE selects among updates across datasets in all five collections, and also within single datasets in (D) and (E) (Table~\ref{tab:selectivity}). Learn then Test is compared with BRIDGE at the conceptual level (Section~\ref{sec:rel-gates}). An empirical comparison at an equal label budget is the next step. In (A) and (B), the corrected $t$-test with per-split re-selection deploys more updates above the margin, with a share below the margin that does not differ significantly from that of BRIDGE (post hoc, Table~\ref{tab:falseimp}).

\paragraph{Settings of the gate.} During development, the seven parameters of the gate and the partition were frozen (Table~\ref{tab:hyper}). In (D) every tested variant of these parameters meets the criterion. In (A) two variants fail, namely a bag drawn from all of SEARCH and a SEARCH of $1{,}000$ rows, so the frozen values are the tested condition of use (Appendix~\ref{app:sens}). Likewise, the default candidate space of $173$ configurations is the tested condition. A denser grid meets the criterion at every tested size in (A) to (D). On the binary datasets of (E) the bound exceeds $10\%$ at $11{,}249$ configurations (Table~\ref{tab:kspace}). The next step is to develop a larger space, such as the joint seven-family space, through a new development cycle and a fresh confirmation.

\paragraph{Learners, metrics, and tuning.} The gate is confirmed for base learners that have iteration control. XGBoost with $300$ boosting rounds and no early stopping fails in (A), (B), and (E). The failure lies in the gate statistic rather than in the final model, whose weighted-$F_1$ difference from HGB has an interval containing zero in (A) and (D) (Tables~\ref{tab:basefull} and~\ref{tab:absolutequality}). On rare classes, the metrics of use are balanced accuracy and macro-$F_1$, because weighted-$F_1$ exceeds the $10\%$ LOCO limit (Table~\ref{tab:imbalance}). Two further conditions meet the criterion narrowly or in part. Feature selection with SVC misses the LOCO limit ($11.4\%$), and tuning as an update meets the criterion in two of $12$ cells because tuned updates are accepted rarely (Tables~\ref{tab:general} and~\ref{tab:tuning}). The two alternative member pairs miss the criterion on the OpenML collection (E). Few updates are accepted there, and $11$ of the $12$ false improvements fall on small references (Tables~\ref{tab:members} and~\ref{tab:citedruns}). The next steps are a more stable feature-selection rule, a tuning path through development and confirmation, and larger references or more datasets for the member pairs in (E). The evidence covers tabular learners and label-based metrics. Deep and foundation models, raw images, and probability-level or ranking metrics are therefore directions for future work.

\paragraph{Data and clinical scope.} The clinical evaluation is retrospective and covers six sources, two of which are unweighted surveys. Net benefit therefore describes the sample rather than the population, and prospective validation in further institutions is the next step. The locked reference is finite and drawn once per task. On small datasets, near-zero changes are noisy in both directions, and two further reference seeds leave the conclusion unchanged (Table~\ref{tab:citedruns}). The ACTG 175 task classifies the recorded event status (CD4 decline of at least $50\%$, AIDS, or death) without handling censoring, so it does not estimate a risk at a fixed horizon. Without ACTG 175, BRIDGE accepts $1{,}001$ updates with no false improvement (bound $0.39\%$; Table~\ref{tab:clinsource}). The mammography source carries no patient identifier, and the main result is unchanged when that source is removed (Table~\ref{tab:clinsource}). Two kinds of subgroup miss the criterion: the binary datasets of (E), narrowly on LOCO ($10.3\%$), and multiclass subgroups of the OpenML collections, whose result varies with the number of classes (Appendix~\ref{app:settingE}); more multiclass datasets are the next step. The clinical baseline is the pre-registered threshold $0.5$, not $p_t$; a baseline at $t=p_t$ and a full decision curve across thresholds are a further step \cite{vickers2006}. Dataset shift and calibration drift after deployment lie outside the scope of the gate \cite{finlayson2021,davis2017}.

\paragraph{Future work.} The evidence would be extended in five directions: (i) prospective validation, (ii) settings with explicit distribution shift, (iii) deep and foundation models, (iv) new datasets from the OpenML suite, and (v) an empirical comparison with Learn then Test at an equal label budget. Three further directions would extend the method itself rather than the evidence. The first direction is a series of updates on the same model, with a bound on the false-improvement rate that covers the whole series rather than each decision separately. The second direction is a count-based bootstrap for probability-level and ranking metrics. The third is the same-POOL cell as a separately pre-registered rule, together with per-task threshold selection. The quantile and the margin set the operating point between false improvements and yield, and choosing an operating point for a target yield, through a new development cycle on (A) and (B) and confirmation on new data, is a further step; such a result would be development, not a post hoc claim.

\appendices
\counterwithin{table}{section}\renewcommand{\thetable}{\thesection\arabic{table}}
\counterwithin{figure}{section}\renewcommand{\thefigure}{\thesection\arabic{figure}}
\renewcommand{\topfraction}{0.95}\renewcommand{\dbltopfraction}{0.95}\renewcommand{\bottomfraction}{0.6}\renewcommand{\textfraction}{0.05}
\renewcommand{\floatpagefraction}{0.9}\renewcommand{\dblfloatpagefraction}{0.9}
\setcounter{topnumber}{4}\setcounter{dbltopnumber}{6}\setcounter{bottomnumber}{2}\setcounter{totalnumber}{6}

\noindent\textit{Appendices A to H follow: proofs, data, and pre-registration (A to C); data separation, frozen parameters, sensitivity, and cost (D); per-dataset results (E); the patient-identity audit (F); the OpenML collection (E) (G); and the detailed result tables (H).}

\FloatBarrier
\section{Proofs}
\label{app:proofs}

The proofs below cover Lemma~\ref{lem:binom}, Proposition~\ref{prop:bound}, and Corollary~\ref{cor:cons} of Section~\ref{sec:method-risk}. A final paragraph states the estimand used on the real data. The proofs rest on Assumption~\ref{ass:pop} and on a standard binomial tail bound alone.

\begin{proof}[Proof of Lemma~\ref{lem:binom}]
Under Assumption~\ref{ass:pop}(i)--(iii), the tasks are i.i.d.\ and the gate is frozen before any data are seen (condition (ii) makes the batch construction and the training/reference split part of the task draw), so the pairs $(A_j,H_j)$ are i.i.d.\ across tasks. Conditioning on $M=\sum_{j=1}^{T} A_j=m$ selects $m$ accepted tasks, each with false-improvement indicator $H_j\sim\mathrm{Bernoulli}(r_{\mathcal{Q}})$ independently; hence $S=\sum_{j=1}^{T} A_j H_j\sim\mathrm{Binomial}(m,r_{\mathcal{Q}})$.
\end{proof}

\begin{proof}[Proof of Proposition~\ref{prop:bound}]
Fix $M=m\ge1$ and write $r=r_{\mathcal{Q}}$. For $X\sim\mathrm{Binomial}(m,r)$ the lower Chernoff--Kullback--Leibler bound gives, for a fixed threshold $x\le r$, $\Pr\!\left(X \le m x\right)\le \exp\!\left(-m\,\mathrm{kl}(x\|r)\right)$. Let $x^{*}$ be the fixed threshold with $m\,\mathrm{kl}(x^{*}\|r)=\log(1/\delta)$ (it exists whenever $m\,\mathrm{kl}(0\|r)\ge\log(1/\delta)$, and otherwise the event $r>U_m$ is empty so the bound holds trivially); since $\mathrm{kl}(S/m\|r)$ decreases as $S/m$ grows for $S/m\le r$, the event $r>U_m$ is exactly $S/m<x^{*}$. Applying the bound with the fixed threshold $x^{*}$ gives $\Pr(r>U_m)\le\exp(-m\,\mathrm{kl}(x^{*}\|r))=\delta$, so the endpoint $U_m=\sup\{u\in[S/m,1] : m\,\mathrm{kl}(S/m\|u)\le \log(1/\delta)\}$ (with $U_m=1$ when $S=m$ or $m=0$) is a one-sided upper bound on $r$ with coverage $\ge 1-\delta$ conditionally on $M=m$. Averaging over the distribution of $M$ preserves coverage:
\[
\begin{aligned}
\Pr\!\left(r_{\mathcal{Q}}\le U\right)
&= \sum_{m\ge0}\Pr\!\left(M=m\right)\Pr\!\left(r_{\mathcal{Q}}\le U_m\mid M=m\right) \\
&\ge \sum_{m\ge0}\Pr\!\left(M=m\right)\left(1-\delta\right) \;=\; 1-\delta,
\end{aligned}
\qedhere
\]
\end{proof}

\begin{proof}[Proof of Corollary~\ref{cor:cons}]
Replace $H_j$ by $H'_j\ge H_j$ for every task whose status is reference-uncertain, so that every accepted update is counted as a false improvement unless the reference indicates improvement, and let $S'=\sum_{j=1}^{T} A_j H'_j\ge S$. Bounds computed from $(M,S')$ exceed the corresponding bounds computed from $(M,S)$; since the bounds from $(M,S)$ have coverage $\ge1-\delta$ by Proposition~\ref{prop:bound}, so do the bounds from $(M,S')$.
\end{proof}

On the real data, let $\mathcal{E}$ denote the observed collection of tasks. The estimand is $r_{\mathcal{E}}=\mathbb{E}\!\left[\sum_j W_j\,\middle|\,\mathcal{E}\right]\big/\mathbb{E}\!\left[\sum_j A_j\,\middle|\,\mathcal{E}\right]$, the ratio of the expected sums of $W_j=A_jH_j$ and of $A_j$ conditional on $\mathcal{E}$. The collection $\mathcal{E}$ includes the locked reference and the frozen outer partition, while the inner SEARCH/POOL split and the bag draw are re-drawn in every trial. The trials are independent given the fixed collection (independent seeds), with possibly different means across tasks, and the one-sided Chernoff--Kullback--Leibler bound applies to sums of independent Bernoulli variables with different means.

\FloatBarrier
\section{Data, Audit, and Generator}
\label{app:datasets}
\apptables{tab:splits}{tab:tasks}

Table~\ref{tab:splits} covers every dataset and task of (A) and (C)--(E), grouped by collection, with the domain, source, sample size, minority rate, and size of the locked reference. The table footnote defines the SEARCH/POOL split, the eligible training pool, and the number of trials.

\begin{table*}[!t]
\centering
\scriptsize
\setlength{\tabcolsep}{2.2pt}
\caption{Data partition for each dataset or task in (A) and (C)--(E).}
\label{tab:splits}
\begin{minipage}[t]{0.495\textwidth}\centering
\begin{tabular}{@{}lllrrr@{}}
\toprule
Dataset & Domain & Source & N & \begin{tabular}[b]{@{}c@{}}Min.\\\%\end{tabular} & Ref. \\
\midrule
\multicolumn{6}{@{}l}{\textit{(A) 11 new datasets}} \\
pc3 & Software & OpenML 1050 & 1,563$^{\dagger}$ & 10.24 & 468 \\
MIC & Medical & OpenML 45648 & 1,649$^{\dagger}$ & 15.59 & 494 \\
Marketing\_Campaign & Marketing & OpenML 46940 & 2,240$^{\dagger}$ & 14.91 & 672 \\
Pumpkin\_Seeds & Agriculture & OpenML 46951 & 2,500$^{\dagger}$ & 48.00 & 750 \\
seismic-bumps & Mining & OpenML 46956 & 2,584$^{\dagger}$ & 6.58 & 775 \\
Rice & Agriculture & UCI 545 & 3,810 & 42.78 & 1,143 \\
Census-Income-KDD & Social & OpenML 42750 & 199,523 & 6.21 & 100,000 \\
Santander2016 & Banking & OpenML 46859 & 76,020 & 3.96 & 56,020 \\
KDD98 & Social & OpenML 42343 & 82,318 & 11.77 & 62,318 \\
LICD & Employment & OpenML 46792 & 63,634 & 25.72 & 43,634 \\
Homesite & Insurance & OpenML 46857 & 260,753 & 18.75 & 100,000 \\
\multicolumn{6}{@{}l}{\textit{(C) 12 clinical tasks}} \\
mammography & Radiology & OpenML 310 & 7,847 & 3.22 & 2,354 \\
actg175 & Clinical & UCI 890 & 2,139$^{\dagger}$ & 24.36 & 641 \\
nhanes\_hypertension & Survey & NHANES 2017--18 & 5,559 & 38.12 & 1,667 \\
nhanes\_diabetes & Survey & NHANES 2017--18 & 5,393 & 16.24 & 1,617 \\
nhanes\_asthma & Survey & NHANES 2017--18 & 5,564 & 15.03 & 1,669 \\
nhanes\_coronary & Survey & NHANES 2017--18 & 5,553 & 4.77 & 1,665 \\
brfss\_stroke & Survey & BRFSS 2015 & 440,166 & 4.15 & 100,000 \\
brfss\_asthma & Survey & BRFSS 2015 & 439,963 & 13.50 & 100,000 \\
brfss\_kidney & Survey & BRFSS 2015 & 439,879 & 3.56 & 100,000 \\
brfss\_depression & Survey & BRFSS 2015 & 439,259 & 19.07 & 100,000 \\
covid\_mx & Public Health & Mexico COVID-19 & 218,784 & 12.29 & 100,000 \\
wids\_icu & ICU & WiDS 2020 & 91,713 & 8.63 & 71,713 \\
\multicolumn{6}{@{}l}{\textit{(D) 34 binary supporting datasets}} \\
miniboone & Physics & OpenML 41150 & 130,064 & 28.06 & 100,000 \\
compas (dev) & Justice & OpenML 42192 & 5,278 & 47.04 & 1,583 \\
credit-default-tw (dev) & Finance & OpenML 42477 & 30,000 & 22.12 & 8,000 \\
noshow & Healthcare & OpenML 43439 & 110,527 & 20.19 & 90,527 \\
cardio & Medical & OpenML 45547 & 70,000 & 49.97 & 50,000 \\
heloc (dev) & Finance & OpenML 45553 & 9,871 & 47.97 & 2,961 \\
lt-vehicle-loan & Finance & OpenML 46430 & 233,154 & 21.71 & 100,000 \\
codexglue (dev) & Software & OpenML 46682 & 27,318 & 45.61 & 8,000 \\
kickstarter & Crowdfund & OpenML 46810 & 86,502 & 31.96 & 66,502 \\
bank-churn (dev) & Banking & OpenML 46911 & 10,000 & 20.37 & 3,000 \\
ecom-shipping (dev) & E-commerce & OpenML 46924 & 10,999 & 40.33 & 3,299 \\
employee (dev) & HR & OpenML 46925 & 4,653 & 34.39 & 1,395 \\
amazon-access (dev) & HR & OpenML 4135 & 32,769 & 5.79 & 8,000 \\
kick & Automotive & OpenML 41162 & 72,983 & 12.30 & 52,983 \\
porto-seguro & Insurance & OpenML 42206 & 595,212 & 3.64 & 100,000 \\
\bottomrule
\end{tabular}
\end{minipage}\hfill
\begin{minipage}[t]{0.495\textwidth}\centering
\begin{tabular}{@{}lllrrr@{}}
\toprule
Dataset & Domain & Source & N & \begin{tabular}[b]{@{}c@{}}Min.\\\%\end{tabular} & Ref. \\
\midrule
\multicolumn{6}{@{}l}{\textit{(D) 34 binary supporting datasets (cont.)}} \\
relevant-images & Vision & OpenML 42680 & 134,731 & 4.14 & 100,000 \\
newspaper-churn (dev) & Media & OpenML 44226 & 15,855 & 19.15 & 4,756 \\
online-shoppers (dev) & E-commerce & OpenML 45060 & 12,330 & 15.47 & 3,699 \\
give-some-credit & Finance & OpenML 45577 & 150,000 & 6.68 & 100,000 \\
hmeq (dev) & Finance & OpenML 46431 & 5,960 & 19.95 & 1,788 \\
avida-hil6 & Biology & OpenML 46799 & 573,891 & 3.66 & 100,000 \\
avida-cov2 (dev) & Biology & OpenML 46800 & 3,733 & 19.85 & 1,119 \\
acs-drem & Social & OpenML 46802 & 1,138,289 & 7.90 & 100,000 \\
fake-job (dev) & Employment & OpenML 46809 & 12,725 & 4.33 & 3,817 \\
sepsis & Medical & OpenML 46853 & 110,341 & 7.37 & 90,341 \\
dota2 & Gaming & OpenML 45563 & 102,944 & 47.27 & 82,944 \\
santander & Banking & OpenML 45566 & 200,000 & 10.05 & 100,000 \\
home-credit & Finance & OpenML 45567 & 307,511 & 8.07 & 100,000 \\
coil2000 (dev) & Insurance & OpenML 46281 & 9,822 & 5.97 & 2,946 \\
bnp-paribas & Insurance & OpenML 46856 & 114,321 & 23.88 & 94,321 \\
polish-bankruptcy (dev) & Finance & OpenML 46950 & 5,910 & 6.94 & 1,773 \\
taiwan-bankruptcy (dev) & Finance & OpenML 46962 & 6,819 & 3.23 & 2,045 \\
naticusdroid (dev) & Cyber & OpenML 46969 & 7,491 & 35.03 & 2,247 \\
hr-analytics (dev) & HR & OpenML 46935 & 19,158 & 24.93 & 5,747 \\
\multicolumn{6}{@{}l}{\textit{(D) 5 multiclass datasets}} \\
Fashion-MNIST & Vision & OpenML 40996 & 70,000 & 10.00 & 50,000 \\
DryBean & Agriculture & UCI 602 & 13,611 & 3.84 & 4,083 \\
USPS & Vision & OpenML 41082 & 9,298 & 7.61 & 2,789 \\
mnist\_784 & Vision & OpenML 554 & 70,000 & 9.02 & 50,000 \\
Kuzushiji-MNIST & Vision & OpenML 41982 & 70,000 & 10.00 & 50,000 \\
\multicolumn{6}{@{}l}{\textit{(E) 11 binary OpenML datasets}} \\
airline-satisfaction & Marketing & OpenML 46920 & 129,880 & 45.27 & 100,000 \\
KDDCup09-churn & Marketing & OpenML 1112 & 50,000 & 7.34 & 30,000 \\
bank-marketing & Banking & OpenML 1461 & 45,211 & 11.70 & 25,211 \\
pulsar-HTRU2 & Physics & OpenML 45558 & 17,898 & 9.16 & 5,369 \\
jm1 & Software & OpenML 1053 & 10,885 & 19.35 & 3,265 \\
support2 & Medical & OpenML 46860 & 9,105 & 31.89 & 2,731 \\
telco-customer-churn & Marketing & OpenML 42178 & 7,043 & 26.54 & 2,112 \\
phoneme & Speech & OpenML 1489 & 5,404 & 29.35 & 1,621 \\
spambase & Cyber & OpenML 44 & 4,601 & 39.40 & 1,380 \\
sick & Medical & OpenML 38 & 3,772 & 6.12 & 1,131 \\
bioresponse & Biology & OpenML 4134 & 3,751 & 45.77 & 1,125 \\
\multicolumn{6}{@{}l}{\textit{(E) 3 multiclass OpenML datasets}} \\
ipums-la-97-small & Survey & OpenML 382 & 7,019 & 3.68 & 2,105 \\
otto-group & E-commerce & OpenML 45548 & 61,878 & 3.12 & 41,878 \\
corporate-credit-ratings & Finance & OpenML 46372 & 5,000 & 10.00 & 1,500 \\
\bottomrule
\end{tabular}
\end{minipage}

\vspace{2pt}
\parbox{\textwidth}{\justifying\scriptsize *Locked reference: size (Ref.) and source of each real dataset or task. The rows that remain once the reference is set aside form the eligible training pool, from which every trial draws its training rows. Groups. (A): the $11$ new binary datasets. (C): the $12$ clinical tasks. (D): the $39$ supporting datasets, namely the $34$ binary confirmation datasets (including the $18$ development datasets, marked (dev), on which the gate and the candidate space were frozen) and the $5$ multiclass datasets. (E): the $14$ datasets of the OpenML collection, with names shortened and identified by the OpenML ID. Setting (B), the synthetic population, is generated rather than drawn from a dataset and is described in this appendix. Min.: minority-class share; for a multiclass dataset, the smallest class. Split: $2{,}430$ training rows by default (SEARCH $630$, POOL $1{,}800$); $^{\dagger}$: $1{,}000$ training rows (SEARCH $259$, POOL $741$). Trials per task: $100$ in (A), (D), and binary (E); $200$ in (C); $60$ in multiclass (D) and (E). Domains briefly characterize the public sources, and exact definitions live in the code-repository loaders. The $28$ domain labels map to seven groups, all represented among the $64$ real datasets: Health (Healthcare, Medical, Radiology, Clinical, Survey, Public Health, ICU); Finance, banking and insurance (Finance, Banking, Insurance, Crowdfund); Agriculture and mining (Agriculture, Mining); Cyber and software (Software, Cyber); Commerce, services and HR (E-commerce, Marketing, HR, Employment, Automotive); Social, legal and media (Social, Justice, Media, Gaming); Science and vision (Physics, Biology, Vision, Speech).}
\end{table*}

\begin{table*}[!t]
\centering
\scriptsize
\providecolor{bghead}{HTML}{DBEFF8}\providecolor{bgzebra}{HTML}{F2F9FC}\providecolor{bgpass}{HTML}{D7F0D2}\providecolor{bgfail}{HTML}{FCDCD7}
\setlength{\tabcolsep}{3pt}
\caption{Clinical tasks.}
\label{tab:tasks}
\begin{tabular}{@{}lllll@{}}
\toprule
Task & Population & Outcome & Predictors / timing & Hypothetical action \\
\midrule
mammography & Screening mammograms (OpenML 310) & Malignant calcification & Image-derived features at reading & Recall for biopsy \\
actg175 & HIV-1 trial participants (UCI 890) & Recorded event: CD4 decline $\geq50\%$, AIDS, or death & Baseline covariates only & Intensify monitoring \\
nhanes\_* & US adults $\geq$20 y, NHANES 2017--18 & Self-reported: hypertension, diabetes, asthma, CHD & Exam \& questionnaire, same visit & Refer for testing \\
brfss\_* & US adults, BRFSS 2015 & Self-reported: stroke, asthma, kidney disease, depression & Survey responses, same interview & Refer for screening \\
covid\_mx & Confirmed COVID-19 patients, Mexico & In-hospital/registered death & Admission data (comorbidities, triage) & Escalate care \\
wids\_icu & ICU admissions, GOSSIS (WiDS 2020) & Hospital death & First-24 h ICU measurements & Escalate care \\
\bottomrule
\end{tabular}

\vspace{2pt}
\parbox{\linewidth}{\justifying\scriptsize *Clinical tasks: population, outcome, timing of predictors and hypothetical action. The COVID-19 cohort is the open data released by the Mexican Secretaría de Salud. NHANES outcomes are physician diagnoses as reported by the respondent; the NHANES and BRFSS survey tasks are cross-sectional health-related classification (predictors from the same visit or interview), not prognosis.}
\end{table*}

\paragraph{Inclusion and exclusion rules.} A dataset qualified for the benchmark on three conditions: it was publicly available, it was tabular or image-derived in tabular form, and its labels were recorded. A dataset was excluded for any of the following reasons:
\begin{itemize}
\item The data were synthetic or simulated, or formed a time series.
\item The binary target came from thresholding a continuous variable, or its provenance could not be verified.
\item The class ratio was exactly $50\%$, which marks artificially balanced data, or more than $25\%$ of the rows were duplicates.
\item A single feature predicted the target almost without error (feature leakage, as for PhiUSIIL, AUC $0.996$).
\item The dataset family had already informed the design, or the task was a regression.
\item The rows were grouped by subject or session, or identifiers indicated leakage that de-duplication could not remove.
\end{itemize}

The $25\%$ duplicate threshold and the other rules were fixed in the pre-registration of run MINI-21 before any results were seen (Table~\ref{tab:chronology}). The benchmark expanded over successive selection rounds, so the rules above are the union of the rules of all rounds. The versioned list, with the reason for each dataset, is maintained in the code repository.

\paragraph{Synthetic population $\mathcal{Q}$: tasks and seeds.} The synthetic analysis draws $1{,}400$ tasks i.i.d.\ from a fixed generative population $\mathcal{Q}$ and runs one trial per task; the population was pre-specified in run MINI-25 (Table~\ref{tab:chronology}). Each task carries the seed of the generator that creates it (generator in the code repository). The primary batch uses seeds $0$ to $1{,}399$, which no pilot run touched; the replication batch uses $1{,}500$ to $2{,}899$; and later runs take new seeds, such as $2{,}900$ to $4{,}299$ for the tuning runs.

\paragraph{Synthetic population $\mathcal{Q}$: label mechanism.} Each task is generated in its own batch of $103{,}000$ rows, with $3{,}000$ development rows and $100{,}000$ reference rows. The development rows allow the same $2{,}430$ training rows as on the real data, and the reference measures the true gain with a small error.

The label mechanism is drawn with equal probability from a tree mechanism and a linear mechanism. Under the tree mechanism, each task draws $p\in[5,40]$ standard normal features, $4$ to $12$ pairwise-threshold interactions, weights $\sim\mathcal{N}(0,2^{2})$, thresholds $\sim\mathcal{N}(0,0.7^{2})$, and two feature indices uniformly. Under the linear mechanism, each task draws $p\in[5,40]$ standard normal features with coefficients $\beta\sim\mathcal{N}(0,1)$; a fraction of the coefficients is set to zero, and the rest are divided by the square root of the number of non-zero coefficients. Three parameters are drawn per task: the signal strength $s\in[0.05,1]$, a class-balance parameter $m\in[0.05,0.5]$, and the linear sparsity $\lambda\in[0,0.7]$.

Three pilot rounds preceded the pre-registration of MINI-25 (Table~\ref{tab:chronology}), where the ranges and constants were fixed. With a strong signal, the naive rule accepted a false improvement on $1$ of $59$ pilot tasks; after the signal range was extended down to $s=0.05$, it did so on $6$ of $58$.

Within each batch, the logit is standardized and scaled by $s$, and the labels are Bernoulli draws whose probability is the logistic transform of the scaled logit minus an intercept, the batch quantile of order $1-m$ of the scaled logit. A stratified split sets the development rows apart from the reference rows. The parameter $m$ shifts the intercept and does not fix the class proportion: the realized minority proportion averages $42\%$ and lies between $29\%$ and $50\%$ for $90\%$ of the tasks.

\paragraph{Synthetic population $\mathcal{Q}$: design and truth.} A pilot phase designed the population so that the naive false-improvement rate is realistic, at least $15\%$ and close to the $19\%$ seen on real data; the pilot tasks enter no reported result. A Gaussian mechanism was dropped because the naive rule was almost never wrong on it, and the results of BRIDGE were not used in these design choices. The true gain $\Delta$ is computed on each task's own reference of $100{,}000$ rows, and the decision pairs are i.i.d.\ across tasks (Assumption~\ref{ass:pop}). The standard error of the reference has a median of about $0.0014$, roughly a quarter of $\varepsilon=0.005$. The comparison uses the estimated rate $\hat r$, the bound $U$ of Proposition~\ref{prop:bound} \eqref{eq:U}, and a conservative variant that counts updates of uncertain reference status as false improvements (Table~\ref{tab:synthetic}).

\paragraph{Training rows and reference size.} A task has $2{,}430$ training rows when at least $2{,}480$ labeled examples remain after the reference is locked, and $1{,}000$ training rows otherwise (Table~\ref{tab:splits}). SEARCH takes $25.9\%$ of the training rows and POOL the rest \eqref{eq:partition}. The default tasks therefore have $630$ SEARCH rows and $1{,}800$ POOL rows, and Table~\ref{tab:hyperbasis} gives the basis of these sizes. When $N<40{,}000$, the reference has $\lfloor 0.3N\rfloor$ rows, capped at $8{,}000$. When $N\ge40{,}000$, the reference has $\min(N-20{,}000,\,100{,}000)$ rows (Table~\ref{tab:hyperbasis}).

\paragraph{Checks on the sizes.} Two checks examine the sizes chosen above (Table~\ref{tab:citedruns}). First, the median standard error of the gain on the locked reference is $0.0003$ to $0.0020$ on the five new datasets whose reference holds at least $40{,}000$ rows, and $0.0043$ to $0.0152$ on the six datasets with a shorter reference. Second, a SEARCH of $360$, $630$, or $1{,}000$ rows on $\mathcal{Q}$ gives bounds of $1.65\%$, $3.56\%$, and $3.11\%$, at yields of $32.1$ to $33.9$ per $100$, and all three sizes meet the criterion.

\FloatBarrier
\section{Pre-Registration Chronology}
\label{app:prereg}
\apptables{tab:citedruns}{tab:chronology2}

Nineteen development runs (MINI-01 to MINI-19) fixed the seven parameters and the candidate space, and the gate was then frozen (Table~\ref{tab:chronology}). Every later run tests that frozen gate on new data, in sensitivity analyses, and in the conditions of use; variants are reported beside the default, and no run changes the default.

The numbers quoted in the text from development, confirmatory, and sensitivity runs that appear in no other table are recomputed from the stored outputs and listed in Table~\ref{tab:citedruns}, sorted by run number.

\begin{table*}[!t]
\centering
\scriptsize\renewcommand{\arraystretch}{0.92}
\setlength{\tabcolsep}{2.2pt}
\caption{Runs cited in the text, by run number.}
\label{tab:citedruns}\label{tab:devruns}
\begin{minipage}[t]{0.494\textwidth}\centering
\begin{tabular}[t]{@{}lrrrr@{}}
\toprule
Check & \begin{tabular}[b]{@{}c@{}}Acc./\\false\end{tabular} & \begin{tabular}[b]{@{}c@{}}Bound\\\% $\downarrow$\end{tabular} & \begin{tabular}[b]{@{}c@{}}LOCO\\\% $\downarrow$\end{tabular} & Yield \\
\midrule
\multicolumn{5}{@{}l}{\textit{MINI-02 $\cdot$ 34 (D) $\times$ 100}} \\
\quad frozen gate, threshold-only HGB & 586 / 1 & 1.24 & 1.45 & 17.2$^{a}$ \\
\addlinespace[1pt]
\multicolumn{5}{@{}l}{\textit{MINI-04 $\cdot$ 34 (D) $\times$ 100}} \\
\quad threshold-only KNN, cross-fit gate & 150 / 7 & 14.04 & 19.95 & 4.4 \\
\addlinespace[1pt]
\multicolumn{5}{@{}l}{\textit{MINI-09 $\cdot$ 18 dev. $\times$ 100}} \\
\quad third member KNN & 218 / 3 & 5.71 & 8.70 & 12.1 \\
\quad third member SVC & 269 / 4 & 5.27 & 8.41 & 14.9 \\
\addlinespace[1pt]
\multicolumn{5}{@{}l}{\textit{MINI-12 $\cdot$ 18 dev. $\times$ 100}} \\
\quad tuned HGB as the update & 193 / 5 & 8.61 & 12.32 & 10.7 \\
\quad tuned HGB, baseline and member & 356 / 13 & 8.12 & 9.63 & 19.8 \\
\quad default HGB, baseline and member & 455 / 10 & 5.20 & 6.41 & 25.3 \\
\addlinespace[1pt]
\multicolumn{5}{@{}l}{\textit{MINI-14 $\cdot$ 27 (D) $\times$ 40 per size}} \\
\quad tuning on 630 rows & 87 / 1 & 10.12 & 14.27 & 8.1 \\
\quad tuning on 2,000 rows & 40 / 0 & 14.52 & 20.08 & 3.7 \\
\quad tuning on 5,000 rows & 30 / 0 & 20.95 & 59.82 & 2.8 \\
\quad all three sizes (3,240 dec.) & 157 / 1 & 5.21 & 6.67 & 4.8 \\
\addlinespace[1pt]
\multicolumn{5}{@{}l}{\textit{MINI-16 $\cdot$ 18 dev. $\times$ 100}} \\
\quad accuracy & 234 / 6 & 7.75 & 11.94 & 13.0 \\
\quad balanced accuracy & 1,352 / 4 & 0.93 & 1.01 & 75.1 \\
\addlinespace[1pt]
\multicolumn{5}{@{}l}{\textit{MINI-17 $\cdot$ size study, 3,240 trials}} \\
\multicolumn{5}{@{}l}{\quad skipped at 1,000 training rows: avida-hil6 2, porto-seguro 40} \\
\addlinespace[1pt]
\multicolumn{5}{@{}l}{\textit{MINI-19 $\cdot$ 8 wide dev.\ datasets}} \\
\quad union policy, first 50 trials (400 dec.) & 61 / 1 & 15.10 & 32.37 & 15.2$^{b}$ \\
\quad extended to 150 (MINI-19b; 1,200 dec.) & 179 / 5 & 9.30 & 20.58 & 14.9$^{c}$ \\
\addlinespace[1pt]
\multicolumn{5}{@{}l}{\textit{MINI-20 $\cdot$ size study, 8,100 trials}} \\
\multicolumn{5}{@{}l}{\quad skipped at 1,000 training rows: avida-hil6 3, porto-seguro 100} \\
\addlinespace[1pt]
\multicolumn{5}{@{}l}{\textit{MINI-24 $\cdot$ 11 new, dataset-level bootstrap bound$^{d}$}} \\
\quad naive &  & 29.4 &  &  \\
\quad $5\times2$ cv &  & 13.0 &  &  \\
\quad corrected $t$ &  & 12.5 &  &  \\
\quad BRIDGE &  & 3.1 &  &  \\
\bottomrule
\end{tabular}
\end{minipage}\hfill
\begin{minipage}[t]{0.494\textwidth}\centering
\begin{tabular}[t]{@{}lrrrr@{}}
\toprule
Check & \begin{tabular}[b]{@{}c@{}}Acc./\\false\end{tabular} & \begin{tabular}[b]{@{}c@{}}Bound\\\% $\downarrow$\end{tabular} & \begin{tabular}[b]{@{}c@{}}LOCO\\\% $\downarrow$\end{tabular} & Yield \\
\midrule
\multicolumn{5}{@{}l}{\textit{MINI-27 $\cdot$ 11 new, 1,064 dec.}} \\
\multicolumn{5}{@{}l}{\quad error of the gain estimate$^{e}$: naive $+0.0034$; cross-fitted $-0.0034$} \\
\addlinespace[1pt]
\multicolumn{5}{@{}l}{\textit{MINI-28 $\cdot$ 12 clin. $\times$ 100, weighted-$F_1$ stage}} \\
\quad BRIDGE & 222 / 6 (2.7) & 8.15 & 12.80$^{f}$ &  \\
\quad naive & 1,145 / 254 (22.2) & 26.16 & 28.66$^{f}$ &  \\
\addlinespace[1pt]
\multicolumn{5}{@{}l}{\textit{MINI-31 $\cdot$ $\mathcal{Q}$, 1,400}} \\
\quad XGBoost project default as base & 138 / 16 (11.6) & 19.33$^{g}$ &  & 9.9 \\
\addlinespace[1pt]
\multicolumn{5}{@{}l}{\textit{MINI-33b $\cdot$ $\mathcal{Q}$, 1,400}} \\
\multicolumn{5}{@{}l}{\quad regularized XGBoost below HGB in base weighted-$F_1$: 94.1\% of tasks} \\
\addlinespace[1pt]
\multicolumn{5}{@{}l}{\textit{MINI-34 $\cdot$ $\mathcal{Q}$, 1,400, 9 configurations}} \\
\quad accepted, skill-positive base & 2,161 / 1 (0.05) &  &  &  \\
\quad accepted, skill-negative base & 1,142 / 77 (6.74) &  &  &  \\
\multicolumn{5}{@{}l}{\quad Spearman $\rho$, skill vs bound, 9 configurations: $-0.82$ ($p=0.007$)} \\
\multicolumn{5}{@{}l}{\quad Spearman $\rho$, 11 configurations incl.\ 2 MLPs: $-0.74$ ($p=0.010$)} \\
\addlinespace[1pt]
\multicolumn{5}{@{}l}{\textit{MINI-37 $\cdot$ $\mathcal{Q}$, 1,400}} \\
\quad 16 bags, SEARCH 630 & 459 / 9 & 3.98$^{g}$ &  & 32.8 \\
\quad 32 bags, SEARCH 630 & 475 / 8 & 3.56$^{g}$ &  & 33.9 \\
\quad 64 bags, SEARCH 630 & 465 / 5 & 2.69$^{g}$ &  & 33.2 \\
\quad 32 bags, SEARCH 360 & 464 / 2 & 1.65$^{g}$ &  & 33.1 \\
\quad 32 bags, SEARCH 1,000 & 450 / 6 & 3.11$^{g}$ &  & 32.1 \\
\multicolumn{5}{@{}l}{\quad Monte Carlo error$^{h}$, $\mathcal{Q}$: SD 0.00014; 0.3\%; 1.0\% (14 of 1,400)} \\
\multicolumn{5}{@{}l}{\quad Monte Carlo error$^{h}$, 11 new $\times$ 10: SD 0.00009; 0.7\%; 1.8\% (2 of 110)} \\
\multicolumn{5}{@{}l}{\quad reference SE$^{i}$, 5 new datasets, reference $\ge$ 40,000: 0.0003--0.0020} \\
\multicolumn{5}{@{}l}{\quad reference SE$^{i}$, 6 new datasets, reference $<$ 40,000: 0.0043--0.0152} \\
\addlinespace[1pt]
\multicolumn{5}{@{}l}{\textit{MINI-38 $\cdot$ 11 new $\times$ 100, three reference seeds}} \\
\quad original seed: naive & 1,064 / 203 (19.1) & 22.90 & 25.27 & 96.7 \\
\quad original seed: BRIDGE & 284 / 3 (1.1) & 4.23 & 5.95 & 25.8 \\
\quad seed 1: naive & 1,078 / 244 (22.6) & 26.60 & 29.29 & 98.0 \\
\quad seed 1: BRIDGE & 340 / 8 (2.4) & 6.09 & 8.12 & 30.9 \\
\quad seed 2: naive & 1,067 / 236 (22.1) & 26.15 & 28.83 & 97.0 \\
\quad seed 2: BRIDGE & 289 / 2 (0.7) & 3.42 & 5.05 & 26.3 \\
\addlinespace[1pt]
\multicolumn{5}{@{}l}{\textit{MINI-50 $\cdot$ joint seven-family space}} \\
\multicolumn{5}{@{}l}{\quad (A): naive false impr. 22.4\%; BRIDGE 1.6\%} \\
\multicolumn{5}{@{}l}{\quad (D): naive false impr. 13.2\%; BRIDGE 0.7\%} \\
\addlinespace[1pt]
\multicolumn{5}{@{}l}{\textit{MINI-53 $\cdot$ (E), 1,280 dec., post hoc check$^{j}$}} \\
\quad default (DT and SVC) & 192 / 2 (1.0) & 5.40 & 8.43 & 15.0 \\
\quad members LR and NB & 244 / 6 (2.5) & 7.35 & 10.16 & 19.1 \\
\quad members MLP and KNN & 147 / 6 (4.1) & 12.94 & 20.77 & 11.5 \\
\multicolumn{5}{@{}l}{\quad false improvements of both pairs: 12; reference-uncertain 11, on references of 1,125--2,112 rows} \\
\multicolumn{5}{@{}l}{\quad clearly worse: 1 (airline, reference 100,000 rows, gain $-0.0017$)} \\
\bottomrule
\end{tabular}
\end{minipage}

\vspace{2pt}
\parbox{\textwidth}{\justifying\scriptsize *Runs MINI-02 to MINI-19: development runs (Table~\ref{tab:chronology}). The later runs are confirmatory, sensitivity, or descriptive. Data: datasets $\times$ trials per dataset, or tasks. (D): the $34$ confirmation datasets of the supporting group, of which $27$ are in MINI-14; dev.: the $18$ development datasets. new: the $11$ new datasets of (A). clin.: the $12$ clinical tasks. dec.: decisions. Each run header gives the run number and the data of the run; a row that spans the columns states a result in words. Acc./false: accepted / false improvements, with the false-improvement share in parentheses (\%). Bound: KL-ratio bound ($b=1$), \%. LOCO: leave-one-dataset-out bound, \%. Yield: accepted updates per $100$ decisions. Superscripts. Power check LOCO$_0$, the LOCO bound recomputed with zero false improvements: a, $0.81\%$; b, $18.22\%$; c, $5.54\%$. With $50$ trials per dataset the bound exceeds $10\%$ even when no decision is a false improvement, and $100$ or more trials bring it below $10\%$. d: $95\%$ upper bound of the false-improvement rate, from a bootstrap over the $11$ datasets with 10,000 resamples, \%. e: estimate minus reference truth for the deployed configuration, against the margin $0.005$. f: LOCO by source. g: Proposition~\ref{prop:bound} bound on $\mathcal{Q}$. h: three Monte Carlo checks, namely the standard deviation of the 5\% bootstrap quantile over 20 bootstrap seeds per trial, the share of re-runs whose decision one alternative seed changes, and the share of trials in which at least one of the 20 seeds changes the decision. i: range of the per-dataset median bootstrap standard error of the reference gain, separately for references of 43,634 to 100,000 rows ($\ge$ 40,000) and 468 to 1,143 rows ($<$ 40,000). MINI-12 and MINI-14 tune HGB with Optuna (tree-structured Parzen estimator, 25 search iterations) on the 630-row SEARCH split; MINI-14 repeats this at 630, 2,000 and 5,000 tuning rows. Mean reference weighted-$F_1$ in MINI-12: $0.8180$ for the default HGB and $0.8206$ for the tuned learner, a mean gain of $+0.0026$. MINI-14, the mean reference gain of the tuned learner at $630$, $2{,}000$, and $5{,}000$ rows and over all sizes is $+0.0027$, $+0.0021$, $+0.0020$, and $+0.0023$. Share of trials whose $95\%$ reference interval lies below zero: $12.3\%$, $9.0\%$, $6.7\%$, and $9.3\%$. MINI-17/MINI-20: trials skipped because a bag had fewer than 10 examples of a class. MINI-28: the clinical stage with weighted-$F_1$ as the primary endpoint, before the net-benefit stage (MINI-30). MINI-31: XGBoost with the project defaults (300 rounds, no early stopping) during development. MINI-33b: share of tasks on which the regularized XGBoost has a lower base weighted-$F_1$ than HGB on the locked reference. MINI-34: accepted updates pooled over the nine pre-registered HGB and XGBoost configurations, split by the sign of the base learner's log-loss skill on POOL; Spearman $\rho$ between median skill and bound (Table~\ref{tab:basefull}). MINI-37: the bag count and the SEARCH size, varied one at a time. MINI-38: the outer split is repeated with two fresh reference seeds (original 20261125; seed 1: 1038000001; seed 2: 1038000002). j: post hoc, descriptive. reference-uncertain: the 95\% Bayesian-bootstrap interval of the reference gain contains zero. clearly worse: the interval lies below zero.}
\end{table*}

Table~\ref{tab:chronology} summarizes the path from development to confirmation. For each confirmatory run, the table gives the primary endpoint and its outcome, including predictions that were \emph{not} confirmed. The pre-registered predictions of the component analysis (Section~\ref{sec:exp-ablation}) appear as P1, P2, and P3 in Table~\ref{tab:chronology}. The pre-registration documents record the labels H1, H2, and H35 for the same predictions. Run identifiers of the form MINI-$n$ label the pre-registered runs \cite{nosek2018} and match the scripts and CSV archives in the code repository.

\begin{table*}[!t]
\centering
\scriptsize\renewcommand{\arraystretch}{0.93}
\setlength{\tabcolsep}{2.6pt}
\caption{Chronology of the pre-registered runs, ordered by run number.}
\label{tab:chronology}\label{tab:chronology2}
\begin{tabular}{@{}lllp{262pt}@{}}
\toprule
Run & Fixed beforehand & End. & Outcome \\
\midrule
01--19 & mechanism studies; thresholded HGB, HGB--DT ensembles, the $173$ space & Beh & validation--test independence hypothesis \emph{not} supported \\
04 & threshold-only KNN with the cross-fit gate & Crt & fails (bound $14.0\%$) \\
09 & third ensemble member: KNN or SVC & Yld & SVC chosen: higher yield at a comparable false-improvement rate \\
12, 14 & Optuna-tuned HGB on SEARCH & Dsc & small gains; clearly worse in up to $12\%$ of trials, so the fixed configuration is frozen \\
17, 20 & $1{,}000$/$2{,}430$/$6{,}000$ training rows (17 dev.); fresh seeds & Crt & pass (bound $3.04$/$1.05$/$1.37\%$) \\
21 & $11$ new binary datasets, weighted-$F_1$ & Crt & pass (bound $4.23\%$, LOCO $5.95\%$) \\
22 & $5$ new multiclass datasets of (D) & Crt & pass (bound $2.11\%$, LOCO $3.16\%$) \\
23, 24 & comparators on identical decisions & Bnd & test-statistic-based rules fail; BRIDGE passes (Table~\ref{tab:rules}) \\
25 & $1{,}400$ i.i.d.\ tasks of $\mathcal{Q}$, one trial each & Prp & pass ($3.56\%$; conservative variant $4.70\%$) \\
26 & honest holdout on the same labeled data & Yld & BRIDGE $1.40$--$1.86\times$ as many updates; false improvements $1.1\%$ against $1.3$--$3.2\%$ \\
27 & seeds and models of MINI-21; $173$ candidates & Dsc & naive estimate $+0.0034$, cross-fitted $-0.0034$ (Table~\ref{tab:citedruns}) \\
28 & $12$ clinical tasks, weighted-$F_1$ primary; LOCO by source & Crt & LOCO fails at $100$ trials per task \\
29 & fixed margin $t^*$ from development, at $6{,}000$ rows and on new data & Prd & failure confirmed: near $t^*{=}0.05$ yields $1.4$, $0.1$; bounds $16.1\%$, $100\%$ \\
29 & fixed margins chosen post hoc in sample & Dsc & $t\in\{0.020,\dots,0.035\}$ pass $7/7$ settings; leave-one-out: BRIDGE $7/7$, margin $6/7$ \\
30 & the same $12$ tasks; net benefit $p_t{=}0.2$, balanced accuracy & Crt & pass (bound $0.33\%$) \\
31--36 & XGBoost and MLP as base learners & Crt & XGBoost $\ge100$ rounds without early stopping fails; early-stopped learners pass \\
33--36 & skill diagnostic and two MLP predictions & Prd & skill rule and both MLP predictions \emph{not} confirmed \\
32, 35 & ablation $2\times2\times2$; $|\mathcal{C}|$ up to $11{,}249$ & Mec & same-POOL cell also meets the criterion; predictions \textbf{P1 and P3 not confirmed} \\
37 & bag count, SEARCH size; bootstrap and reference precision & Crt & all variants pass (bags $2.7$--$4.0\%$, SEARCH $1.6$--$3.6\%$; Table~\ref{tab:citedruns}) \\
38 & $11$ new datasets; two fresh reference seeds & Crt & pass at both seeds (false improvements $0.7$--$2.4\%$, bounds $3.4$--$6.1\%$, worst LOCO $8.1\%$) \\
39, 40 & HGB, DT, SVC, KNN tuned in SEARCH; (A), (B), multiclass (D) & Crt & false improvements lower in all $12$ cells; criterion met in $2$ (Table~\ref{tab:tuning}) \\
41 & minority prevalence $20$/$10$/$5\%$; three metrics & Crt & balanced accuracy, macro-$F_1$ pass; weighted-$F_1$ exceeds LOCO (Table~\ref{tab:imbalance}) \\
43 & $39$ datasets of (D); fresh seeds & Crt & pass (bound $1.47\%$, LOCO $1.62\%$, LOCO$_0$ $0.39\%$, yield $30.1$) \\
44 & stored seeds of (A), (D); tuned baseline; $q$--$\varepsilon$ grid; $\eta=2$ pp & Tst & tuning raises the baseline by about $0.3$ points (Tables~\ref{tab:tuningtests}, \ref{tab:qepsreal}) \\
45, 50 & LR--NB and MLP--KNN pairs; joint seven-family space & Crt & met in (A) and (D) with every pair and the joint space (Table~\ref{tab:members}) \\
46, 49 & XGBoost ($300$ rounds; early stopping), MLP; final quality & Crt & XGBoost with $300$ rounds fails in (A); early-stopped learners pass (Table~\ref{tab:basereal}) \\
47 & model seed, bag fraction, bag count, SEARCH size & Crt & (D) passes every variant; (A) two variants fail (Appendix~\ref{app:sens}) \\
48 & comparators, holdouts, ablation cells on (C), (D) & Tst & BRIDGE lower than or not different from every comparator (Table~\ref{tab:rulesCD}) \\
51, 52 & grids of $1{,}825$ and $11{,}249$ on (C), (D), binary (E) & Crt & (C), (D) pass at every size; binary (E) at $1{,}825$ (Table~\ref{tab:kspace}) \\
53 & members, base learners, joint space on (B), (C) & Crt & every design passes except XGBoost with $300$ rounds on (B) (Table~\ref{tab:members}) \\
53 & the same designs on (E) & Crt & early-stopped learners and joint space pass; member pairs and $300$-round XGBoost miss \\
(E) & $14$ new OpenML datasets; frozen gate & Crt & pass jointly ($192/2$; bound $5.40\%$, LOCO $8.43\%$); subgroups do not meet every criterion \\
(E2) & $19$ new OpenML datasets; frozen gate & Crt & pass jointly ($452/7$; bound $4.16\%$, LOCO $5.44\%$) \\
(E2) & subgroups by class count & Crt & binary passes only with the flagged \texttt{churn}; multiclass fails LOCO \\
(E2) & amendment: \texttt{pendigits} replaces \texttt{shuttle} ($60$ trials stopped) & Dsc & replaced before any metric was read; deviates from the grouped-data rule (Appendix~\ref{app:settingE}) \\
post hoc & stored per-trial outputs, reanalyzed & Dsc & random-selection control and gain--loss; \emph{not} pre-registered (Table~\ref{tab:selectivity}) \\
\bottomrule
\end{tabular}

\vspace{2pt}
\parbox{\linewidth}{\justifying\scriptsize *Run: MINI-$n$ run number. Rows are ordered by the first run number, runs that answer the same question share a row, and a run with two findings has two rows. MINI-01 to MINI-19 are development runs (dev.). Endpoint codes: Crt, the pre-registered criterion (bound $<10\%$, LOCO $<10\%$, LOCO$_0$ $<10\%$, yield $\ge5$; Table~\ref{tab:criterion}); Bnd, a bound below $10\%$; Prp, the bound of Proposition~\ref{prop:bound}; Tst, paired differences or TOST; Prd, a pre-registered prediction; Beh, the behavior of the gate by metric, space, and size; Yld, the yield and the false-improvement rate; Mec, the mechanism and the stability in $|\mathcal{C}|$; Dsc, descriptive, with no endpoint. The OpenML collections (E) and (E2) and the post hoc reanalysis close the table, each with its own pre-registration or with none. The run number MINI-42 labels the one-pass analysis of the OpenML datasets (Appendix~\ref{app:settingE}). Code manifests (SHA-256) recorded before each confirmatory seed; the runs are distributed across three machines, per dataset or per block of tasks on the synthetic population, and the cross-machine checks are described in the reproducibility paragraph of this appendix; $0$ trials stopped with an error on confirmatory runs. The pre-registration is internal: the times are recorded by the author. The pre-registration index in the code repository lists, per run, the pre-registration file, its recorded time, and the SHA-256 of the file and of the manifest, and the release tag at submission fixes the package contents; the hashes and the cross-machine checks verify the contents, not the times.}
\end{table*}

\paragraph{Execution and documented bugs.} Every trial of the confirmatory runs MINI-20 to MINI-38 completed without an error (Table~\ref{tab:chronology}).

Four documented bugs appear in the chronology with the corrections (Table~\ref{tab:chronology}). A single-class bag error in the early-stopping split (MINI-18) was fixed by requiring at least $2$ examples of each occurring class per bag. The first launch of MINI-29 stopped before producing data, because a helper function could not be passed to the parallel workers and the queue script reported completion despite the failure. The function was moved to module level, and the script now checks the exit status, with trial logic, data, and seeds unchanged. The first cross-machine verification of MINI-20 reported $83$ differences, caused by a seed argument in the wrong position; after the fix, the second verification found none. A concatenation bug misreported the naive count, which was corrected in MINI-21.

\paragraph{Tuning search spaces.} The tuning runs (MINI-12, MINI-14, MINI-39, MINI-40) used the tree-structured Parzen estimator of Optuna with $25$ iterations and three-fold cross-validation in SEARCH, scoring weighted-$F_1$ (Table~\ref{tab:chronology}). The budget was fixed in MINI-12 and used unchanged, and the search ranges were fixed in the pre-registration of each run.

MINI-12 and MINI-14 tuned HGB with the same sampler and space: learning rate $0.01$ to $0.3$ (log scale), $50$ to $300$ boosting iterations, $3$ to $63$ leaf nodes, $1$ to $50$ minimum samples per leaf, and L2 regularization $10^{-8}$ to $10$ (log scale). MINI-39 and MINI-40 tuned four learners, with spaces in the scikit-learn parameter names. For DT: \texttt{max\_depth} $2$ to $30$, \texttt{min\_samples\_split} $2$ to $20$, \texttt{min\_samples\_leaf} $1$ to $20$, and Gini or entropy. For SVC: $C$ from $10^{-3}$ to $10^{3}$ (log scale), a radial basis function (RBF), linear, or polynomial kernel, $\gamma$ set to \texttt{scale} or \texttt{auto}, degree $2$ to $4$, and \texttt{coef0} $0$ to $1$. For KNN: \texttt{n\_neighbors} $1$ to $31$, uniform or distance weights, and the Minkowski exponent $p\in\{1,2\}$.

\paragraph{Development check.} On the development data, the gate met its bound on accuracy ($7.75\%$) and exceeded the LOCO limit ($11.9\%$). The default threshold is already close to optimal for accuracy, and candidates are seldom better; the yield is $13.0$ per $100$ decisions on accuracy, against $75.1$ on balanced accuracy (MINI-16, Table~\ref{tab:devruns}). Rare acceptance widens the KL-ratio and LOCO bounds, as in the tuning limit of Section~\ref{sec:limitations}.

\paragraph{Clinical runs.} The clinical evidence comes from two pre-registered runs on the same $12$ tasks and $6$ sources. The first fixed weighted-$F_1$ as its primary endpoint and failed the LOCO criterion at $100$ trials per task (KL bound $8.15\%$, per-source LOCO $12.8\%$; MINI-28 in Table~\ref{tab:citedruns}). The second was pre-registered after that failure, with net benefit at $p_t=0.2$ and balanced accuracy as co-primary on the same tasks with fresh seeds.

\paragraph{Reproducibility.} Experiments used Python 3.12, scikit-learn 1.9.1, NumPy 2.5.3, pandas 3.0.6, SciPy 1.18.1, and XGBoost 3.4.1. On the tested Intel and AMD machines, all trials reproduce bit-identically, except the $17$ trials that use Optuna tuning or mutual-information feature selection, which are identical on the same CPU family (Intel or AMD), and one MLP--KNN trial of MINI-53, whose reference gain differs between CPU families without changing the decision. The released software package (C4, Section~\ref{sec:intro}) runs the paper's procedure as its default behavior and reproduces every stored result row of the $15$ development and confirmatory runs, including all $3{,}300$ binary decisions of MINI-21 and the $300$ multiclass decisions of MINI-22 (Table~\ref{tab:chronology}). The identity checks of every run are stored with its manifest in the code repository.

\setcounter{topnumber}{4}\setcounter{dbltopnumber}{6}\setcounter{bottomnumber}{2}\setcounter{totalnumber}{6}

\section{Data Separation, Frozen Parameters, Candidate Spaces, Sensitivity, Fixed Margin, and Cost}
\label{app:sens}
\apptables{tab:leakage}{tab:complexity}

\paragraph{Data used for each choice.} For each choice in the procedure, Table~\ref{tab:leakage} records the data with which the choice is made and the place where the choice is assessed. The table states the separation of Section~\ref{sec:method-gate} explicitly: within the cross-fit, each per-fold candidate is assessed on the fold that did not select it.

\begin{table}[!htb]
\centering
\scriptsize
\setlength{\tabcolsep}{2.5pt}
\caption{What each choice is made with, and where it is assessed.}
\label{tab:leakage}
\begin{tabular}{@{}>{\raggedright\arraybackslash}p{2.75cm}>{\raggedright\arraybackslash}p{2.35cm}>{\raggedright\arraybackslash}p{2.95cm}@{}}
\toprule
Choice & Made with & Assessed on \\
\midrule
Model type and hyperparameters & fixed \emph{a priori}, no tuning (rationale: Section~\ref{sec:method-partition}) & --- \\
Candidate space ($173$ binary default) and gate parameters & $18$ development datasets, frozen & families of new data never seen in development (new binary, new multiclass, clinical, fresh synthetic); the $34$ confirmation datasets of (D) with fresh seeds \\
Preprocessing (imputation, scaling, encoding) & fitted per bag, SEARCH only; WiDS category codes are listed from all rows before the split (covariates only, no labels) & POOL and reference \\
Base models (up to $32$ bag draws, kept bags averaged); early stopping & SEARCH (internal bag validation) & POOL and reference \\
Candidate selection inside the gate & $2/3$ of POOL & the held-out $1/3$ of POOL \\
Accept / reject & lower bootstrap quantile of the cross-fit counts & --- \\
Ground truth (false improvement or not) & --- & locked reference, never used for any decision \\
\bottomrule
\end{tabular}

\vspace{2pt}
\parbox{\linewidth}{\justifying\scriptsize *Within the cross-fit, every per-fold candidate is assessed on the fold that was not used to select it; the final configuration is re-selected on all of POOL (Section~\ref{sec:method-gate}).}
\end{table}

\rowcolors{0}{}{}

\begin{table*}[!t]
\centering
\scriptsize\renewcommand{\arraystretch}{0.95}
\setlength{\tabcolsep}{2.15pt}
\caption{Frozen parameters of BRIDGE and of the evaluation protocol: value and basis of each.}
\label{tab:hyperbasis}
\begin{tabularx}{\textwidth}{@{}llc>{\raggedright\arraybackslash}X@{}}
\toprule
Parameter & Value & Set & Tested in \\
\midrule
Cross-fit folds $F$ & $3$ & Dev & $F\in\{2,3,5\}$ on $\mathcal{Q}$ (Table~\ref{tab:sens31}) \\
Replicates $N_{\mathrm{boot}}$ & $20{,}000$, multinomial over the triple counts & Dev & MC error, $20$ seeds; $2{,}000$ (Table~\ref{tab:citedruns}) \\
Lower quantile $q$ & $5\%$, decision statistic (not the error level $\delta=0.05$) & Dev & $2.5$/$5$/$10\%$ (Tables~\ref{tab:sens31}, \ref{tab:qepsreal}) \\
Margin $\varepsilon$ & $0.005$ & Dev & $0.0025$--$0.01$ (Tables~\ref{tab:marginsweep}, \ref{tab:marginpaired}) \\
HGB base learner & lr $0.1$, $\le300$ iter., $\le31$ leaves, $\ge20$/leaf, no L2; early stop: $10$ iter., $10\%$ of bag & Lib & tuned HGB (Tables~\ref{tab:tuning}, \ref{tab:tuningtests}) \\
Bags & $\le32$ draws of $80\%$ of SEARCH; kept if every class $\ge10$ (multiclass $\ge2$); averaged & Dev & $16$, $32$, $64$ bags on $\mathcal{Q}$ (Table~\ref{tab:citedruns}) \\
DT and SVC members & DT: library defaults; SVC: $C=1$, RBF, $\gamma=$ scale, Platt probabilities & Dev & SVC over KNN (Table~\ref{tab:citedruns}) \\
Preprocessing & median imputation, standardization; mode imputation, ordinal codes ($-1$ unseen) & Std & fitted per bag on SEARCH only, with the WiDS exception in Table~\ref{tab:leakage} \\
Candidate space & $173$ binary configurations; $14+8K$ for $K$ classes (composition below) & Dev & $1{,}825$ and $11{,}249$ configurations (Table~\ref{tab:kspace}) \\
Training rows & $2{,}430$ = SEARCH $630$ + POOL $1{,}800$; $1{,}000$ below $2{,}480$ rows & Dev & $1{,}000$/$2{,}430$/$6{,}000$ (Table~\ref{tab:general}); SEARCH $360$--$1{,}000$ \\
Locked reference & $0.3N\le8{,}000$ if $N<40{,}000$, else $\min(N{-}20{,}000,\,100{,}000)$; $\mathcal{Q}$: $10^5$ & Dev & reference SE; two more seeds (Table~\ref{tab:citedruns}) \\
Trials & $100$ per dataset; $200$ per clinical task; $60$ per multiclass dataset; $1$ per $\mathcal{Q}$ task & Pow, Pre & LOCO$_0$ $18.2\%$/$0.8\%$ at $50$/$100$ (Table~\ref{tab:citedruns})$^{a}$ \\
Seeds & fresh per trial (rows, split, bags, folds, bootstrap); learners $42$; reference fixed & Pre & $\mathcal{Q}$ seeds unused by pilots (Appendix~\ref{app:datasets}) \\
\bottomrule
\end{tabularx}

\vspace{2pt}
\parbox{\textwidth}{\justifying\scriptsize *Set. Dev: set in development, on the $18$ development datasets and the development runs (Table~\ref{tab:devruns}), and frozen before any confirmatory run (Table~\ref{tab:chronology}). Lib: library defaults, with the changes stated in the row, fixed before any evaluation. Std: standard practice. Pow: power check. Pre: pre-registration. Tested in: the run or table in which each value was varied or checked. The margin criterion is met at every tested margin wherever it is met at $0.005$, and SEARCH sizes $360$, $630$, and $1{,}000$ were tested on $\mathcal{Q}$. a: $200$ trials per clinical task, adopted after the LOCO-source condition failed at $100$ (Table~\ref{tab:chronology}); $60$ per multiclass dataset as pre-registered, against $40$ in development. MC: Monte Carlo. SE: standard error. iter.: boosting iterations. lr: learning rate.}
\end{table*}

\paragraph{Default binary space.} The default binary space combines three families of candidates on coarse grids. The thresholds run from $0.05$ to $0.95$ in steps of $0.05$, and the mixture weights move in steps of $0.25$. The three-model family uses the narrower threshold range $0.30$ to $0.60$. The grids are coarse by design, to keep the selection bias small (Section~\ref{sec:method-partition}). For HGB, the grids give $18$ threshold variants, which are the $19$ thresholds without the baseline value $0.5$. The grids also give $57$ HGB--DT mixtures, which are $3$ weights $\times$ $19$ thresholds. The grids also give $98$ HGB--DT--SVC mixtures, which are $14$ weight triples $\times$ $7$ thresholds. The $14$ triples are the $15$ weight combinations in steps of $0.25$ without pure HGB. The three families of the default space together give $18+57+98=173$ configurations. Run MINI-10, which introduced the three-model thresholds, preceded the pre-registration of MINI-11 that froze the space before the confirmatory runs (Table~\ref{tab:chronology}).

\paragraph{Refined and joint spaces.} The refined spaces of Table~\ref{tab:kspace} keep the same three families with finer steps. The refined spaces also widen the three-model threshold range to $0.30$ to $0.70$. The space of $1{,}825$ configurations has $46$ HGB thresholds ($0.05$ to $0.95$, step $0.02$) and $414$ HGB--DT pairs (weights $0.1$ to $0.9$, step $0.1$, $\times$ $46$ thresholds). The space adds $1{,}365$ three-model points ($65$ weight vectors in steps of $0.1$ without pure HGB, $\times$ $21$ thresholds, step $0.02$). The space of $11{,}249$ configurations has $90$ HGB thresholds (step $0.01$, without $0.50$) and $1{,}729$ pairs (weights $0.05$ to $0.95$, step $0.05$, $\times$ $91$ thresholds). The space adds $9{,}430$ three-model points ($230$ weight vectors in steps of $0.05$, $\times$ $41$ thresholds, step $0.01$). The joint space of the member study applies the same grids to every tested member. The joint space holds $18$ HGB thresholds, $57$ two-model mixtures for each of the seven members, and $98$ three-model mixtures for each of the three tested pairs, for $18+7\times57+3\times98=711$ configurations (Table~\ref{tab:members}).

The members of the alternative pairs keep fixed settings and are not tuned. Logistic regression uses $C=1$ and at most $1{,}000$ iterations. Gaussian naive Bayes uses the library defaults. The MLP has hidden layers of $64$ and $32$ units and at most $300$ iterations. The early-stopping rule is that of the HGB configuration, with a $10\%$ validation split and $10$ iterations without improvement. These settings were fixed in the pre-registration of MINI-36 (Table~\ref{tab:chronology}) and not tuned. KNN uses $k=5$. XGBoost in the joint space uses early stopping, as in the base-learner study (Table~\ref{tab:basefull}).

\paragraph{Multiclass spaces.} For $K$ classes, the multiclass space has $14+8K$ candidates. The candidates are $14$ simplex weights over HGB, DT, and SVC, plus $4$ class scales applied to each of the $K$ classes for two base scores, HGB and the HGB--DT average. The four scales, $0.5$, $0.7$, $1.4$, and $2$, are two down-weights and the approximate reciprocals of the two. Development run MINI-18 fixed these scales (Table~\ref{tab:chronology}). In the joint space, the multiclass candidates are the simplex weights of the three tested pairs ($3\times14$) and the HGB class scales ($4K$). The space adds the equal average of HGB with each of the seven members under the class scales ($7\times4K$), for $42+32K$ candidates.

\begin{table*}[!t]
\centering
\scriptsize
\providecolor{bghead}{HTML}{DBEFF8}\providecolor{bgzebra}{HTML}{F2F9FC}\providecolor{bgpass}{HTML}{D7F0D2}\providecolor{bgfail}{HTML}{FCDCD7}
\setlength{\tabcolsep}{3pt}
\caption{Parameter sensitivity: margin sweep.}
\label{tab:marginsweep}
\begin{tabular}{@{}lllll@{}}
\toprule
Setting & $\varepsilon=0.0025$ & $\varepsilon=0.005$ (default) & $\varepsilon=0.0075$ & $\varepsilon=0.01$ \\
\midrule
(A), weighted-$F_1$ & 357/4: 3.83 / 5.10 $\checkmark$ & 284/3: 4.23 / 5.95 $\checkmark$ & 234/2: 4.32 / 6.36 $\checkmark$ & 209/1: 3.74 / 5.67 $\checkmark$ \\
(A), balanced accuracy & 856/1: 0.80 / 0.91 $\checkmark$ & 848/1: 0.81 / 0.92 $\checkmark$ & 836/0: 0.46 / 0.53 $\checkmark$ & 825/0: 0.47 / 0.54 $\checkmark$ \\
(A), macro-$F_1$ & 852/0: 0.45 / 0.52 $\checkmark$ & 829/0: 0.47 / 0.53 $\checkmark$ & 813/0: 0.48 / 0.55 $\checkmark$ & 794/0: 0.49 / 0.56 $\checkmark$ \\
(A), feature selection, KNN & 263/5: 6.06 / 9.75 $\checkmark$ & 225/1: 3.45 / 6.00 $\checkmark$ & 193/0: 2.30 / 4.20 $\checkmark$ & 169/0: 2.66 / 5.32 $\checkmark$ \\
(A), feature selection, SVC & 296/10: 8.20 / 11.52 $\times$ & 253/7: 7.74 / 11.41 $\times$ & 211/5: 7.73 / 12.22 $\times$ & 171/3: 7.42 / 12.61 $\times$ \\
(D) 26, 1,000 rows & 648/12: 4.06 / 4.71 $\checkmark$ & 558/6: 3.04 / 3.59 $\checkmark$ & 468/4: 2.91 / 3.52 $\checkmark$ & 404/3: 2.93 / 3.64 $\checkmark$ \\
(D) 27, 2,430 rows & 843/7: 2.16 / 2.43 $\checkmark$ & 688/1: 1.05 / 1.20 $\checkmark$ & 554/1: 1.32 / 1.54 $\checkmark$ & 440/1: 1.69 / 2.01 $\checkmark$ \\
(D) 27, 6,000 rows & 1088/10: 2.08 / 2.30 $\checkmark$ & 825/3: 1.37 / 1.57 $\checkmark$ & 641/1: 1.13 / 1.35 $\checkmark$ & 498/1: 1.48 / 1.88 $\checkmark$ \\
(D) 5 multiclass & 209/0: 1.97 / 2.86 $\checkmark$ & 197/0: 2.11 / 3.16 $\checkmark$ & 180/0: 2.34 / 3.69 $\checkmark$ & 165/0: 2.59 / 4.23 $\checkmark$ \\
(C), net benefit $p_t=0.2$ & 1301/0: 0.30 / 0.44 $\checkmark$ & 1199/0: 0.33 / 0.50 $\checkmark$ & 1146/0: 0.34 / 0.53 $\checkmark$ & 1080/0: 0.36 / 0.59 $\checkmark$ \\
(B) $\mathcal{Q}$ & 529/10: 3.72 / -- $\checkmark$ & 475/8: 3.56 / -- $\checkmark$ & 419/4: 2.62 / -- $\checkmark$ & 373/3: 2.51 / -- $\checkmark$ \\
\bottomrule
\end{tabular}

\vspace{2pt}
\parbox{\linewidth}{\justifying\scriptsize *Setting: collection, metric, or design. (D) 26/27: the datasets of (D) in the training-size study; rows = training rows. Cells: accepted/false: risk bound \% / LOCO \% and the pre-registered criterion of the setting, recomputed from the stored 5\% bootstrap lower bound of every trial (the deployed configuration does not depend on the margin). Risk bound: KL-ratio bound (real data) or Proposition 1 (population Q); LOCO is not defined on Q. Trials, yield, the LOCO$_0$ power check, and the number of units with a yield of at least $5$ per $100$ are recorded in the code repository.}
\end{table*}

\paragraph{Parameter sweeps.} On the synthetic population $\mathcal{Q}$, ten parameter variants were swept. The variants are the nine combinations of the fold count $F$ and the quantile level $q$ at $N_{\mathrm{boot}}=20{,}000$, and $N_{\mathrm{boot}}=2{,}000$ at the standard configuration $F{=}3$, $q{=}5\%$. All ten pass the pre-registered criterion, with Proposition~\ref{prop:bound} bounds from $2.0\%$ to $4.9\%$ (Table~\ref{tab:sens31}). The quantile level governs the trade-off between the bound and the yield. With a Bonferroni correction \cite{dunn1961} over the ten variants ($\delta=0.05/10=0.005$ per variant), every variant stays below $10\%$: the highest bound is $5.73\%$ ($F{=}3$, $q{=}10\%$) and the lowest $2.75\%$, computed from Table~\ref{tab:sens31}.

The margin sweep over $0.0025$, $0.005$, $0.0075$, and $0.01$ is reported in Table~\ref{tab:marginsweep}. Every setting that passes with the default margin also passes with the three other margins. Feature selection with SVC fails with every margin, as it does with the default. With $16$, $32$, and $64$ bags on $\mathcal{Q}$, the bound is $3.98\%$, $3.56\%$, and $2.69\%$, at yields of $32.8$, $33.9$, and $33.2$ per $100$. All three bag counts meet the criterion (Table~\ref{tab:citedruns}).

\paragraph{Monte Carlo error.} The Monte Carlo error of the bootstrap lower bound at $N_{\mathrm{boot}}=20{,}000$ is small (MINI-37, Table~\ref{tab:chronology}). The check re-runs the bootstrap with $20$ seeds per trial on all $1{,}400$ tasks of $\mathcal{Q}$ and on $110$ real-data trials. The median standard deviation of the decision statistic is $0.00014$ on $\mathcal{Q}$ and $0.00009$ on the real data, about $3\%$ and $2\%$ of $\varepsilon$. A single alternative seed changes the decision in $0.3\%$ and $0.7\%$ of the re-runs, and at least one of the $20$ seeds does so in $1.0\%$ ($14$ of $1{,}400$) and $1.8\%$ ($2$ of $110$). At $N_{\mathrm{boot}}=2{,}000$, BRIDGE accepts $481$ updates with $7$ false improvements, against $475$ with $8$ at $20{,}000$.

\begin{table}[!t]
\centering
\scriptsize
\setlength{\tabcolsep}{2.4pt}
\caption{Parameter sensitivity: each margin tested against the default $0.005$.}
\label{tab:marginpaired}
\begin{tabular}{@{}rllrr@{}}
\toprule
Margin & Acc./false & $\Delta$, pp [95\% CI] & TOST $p$ & Acc.\ ratio \\
\midrule
\multicolumn{5}{@{}l}{\textit{(A)}} \\
0.0025 & 357 / 4 & +0.06 [-0.29; +0.35] & 0.011 & 1.26 \\
0.0075 & 234 / 2 & -0.20 [-0.88; +1.78] & 0.0498 & 0.82 \\
0.01 & 209 / 1 & -0.58 [-1.40; +0.04] & 0.0498 & 0.74 \\
\addlinespace[1pt]
\multicolumn{5}{@{}l}{\textit{(B)}} \\
0.0025 & 529 / 10 & +0.21 [-0.22; +0.81] & $<10^{-4}$ & 1.11 \\
0.0075 & 419 / 4 & -0.73 [-1.64; -0.02] & 0.018 & 0.88 \\
0.01 & 373 / 3 & -0.88 [-1.89; -0.04] & 0.0498 & 0.79 \\
\addlinespace[1pt]
\multicolumn{5}{@{}l}{\textit{(C)}} \\
0.0025 & 1,301 / 0 & +0.00 [+0.00; +0.00] & $<10^{-4}$ & 1.09 \\
0.0075 & 1,146 / 0 & +0.00 [+0.00; +0.00] & $<10^{-4}$ & 0.96 \\
0.01 & 1,080 / 0 & +0.00 [+0.00; +0.00] & $<10^{-4}$ & 0.90 \\
\addlinespace[1pt]
\multicolumn{5}{@{}l}{\textit{(D)}} \\
0.0025 & 1,375 / 13 & +0.41 [+0.09; +0.86] & $<10^{-4}$ & 1.24 \\
0.0075 & 948 / 3 & -0.22 [-0.63; +0.02] & $<10^{-4}$ & 0.85 \\
0.01 & 778 / 2 & -0.28 [-0.69; +0.01] & $<10^{-4}$ & 0.70 \\
\bottomrule
\end{tabular}

\vspace{2pt}
\parbox{\linewidth}{\justifying\scriptsize *Acc./false: accepted / false improvements. $\Delta$: false-improvement rate minus the false-improvement rate at the default margin, in percentage points [95\% CI]. TOST $p$: two one-sided tests of equivalence within $\pm2$ percentage points, Holm-adjusted. Each of the 12 comparisons is equivalent. Acc.\ ratio: accepted updates, relative to the default margin. For margin $\varepsilon$, the decision compares the stored 5\% lower bound with $\varepsilon$; all margins are therefore judged on the same trials. Default accepted / false: (A) 284 / 3, (B) 475 / 8, (C) 1,199 / 0, (D) 1,113 / 6. Brackets: cluster bootstrap over datasets, tasks or clinical tasks, 10,000 resamples; TOST with equivalence margin $\pm2$ percentage points, Holm over the 12 comparisons; specification committed before the computation. In (C) no accepted update is a false improvement.}
\end{table}

\begin{table}[!t]
\centering
\scriptsize
\setlength{\tabcolsep}{2.2pt}
\caption{Parameter sensitivity: quantile level $q$ against margin $\varepsilon$ on the real data.}
\label{tab:qepsreal}
\begin{tabular}{@{}lllll@{}}
\toprule
$q$ & $\varepsilon=0.0025$ & $\varepsilon=0.005$ & $\varepsilon=0.0075$ & $\varepsilon=0.01$ \\
\midrule
\multicolumn{5}{@{}l}{\textit{(A)}} \\
2.5\% & 299/3, 3.99 eq. & 243/2, 4.15 n.e. & 202/1, 3.89 n.e. & 164/0, 2.76 n.e. \\
5\% & 357/4, 3.83 eq. & 284/3, 4.23 def. & 234/2, 4.32 n.e. & 209/1, 3.74 n.e. \\
10\% & 422/11, 5.94 n.e. & 357/4, 3.83 eq. & 297/3, 4.02 eq. & 251/2, 4.00 n.e. \\
\addlinespace[1pt]
\multicolumn{5}{@{}l}{\textit{(D)}} \\
2.5\% & 1,180/7, 1.52 eq. & 972/5, 1.52 eq. & 799/2, 1.17 eq. & 669/0, 0.61 eq. \\
5\% & 1,375/13, 1.96 eq. & 1,113/6, 1.47 def. & 948/3, 1.19 eq. & 778/2, 1.21 eq. \\
10\% & 1,611/25, 2.68 eq. & 1,326/12, 1.92 eq. & 1,107/6, 1.48 eq. & 934/3, 1.21 eq. \\
\bottomrule
\end{tabular}

\vspace{2pt}
\parbox{\linewidth}{\justifying\scriptsize *Cells: accepted / false improvements, KL-ratio bound (\%), and equivalence of the false-improvement rate with the default ($q=5\%$, $\varepsilon=0.005$) by a paired TOST with margin $\pm2$ percentage points and Holm correction over the 11 variants of a setting (eq.: equivalent; n.e.: not shown equivalent; def.: the default; which with 11 datasets reflects limited power, not a difference). All cells also meet LOCO, LOCO$_0$ and yield. Decisions at every $(q,\varepsilon)$ are taken from the same cross-fit counts and bootstrap draws as the default. Pre-registration MINI-44 (run list: Table~\ref{tab:chronology2}).}
\end{table}

\begin{table}[!t]
\centering
\scriptsize
\setlength{\tabcolsep}{2.0pt}
\caption{Parameter sensitivity: tests of tuning on the real data.}
\label{tab:tuningtests}
\begin{tabular}{@{}lrrrrrc@{}}
\toprule
Design & Naive & BRIDGE & Bound $\downarrow$ & LOCO $\downarrow$ & LOCO$_0$ $\downarrow$ & Crit. \\
\midrule
\multicolumn{7}{@{}l}{\textit{(A)}} \\
Paper & 1,064 / 203 (19.1) & 284 / 3 (1.1) & 4.23 & 5.95 & 2.11 & $\checkmark$ \\
Tuned HGB & 1,061 / 246 (23.2) & 211 / 5 (2.4) & 7.73 & 10.66 & 3.02 & $\times$ \\
\multicolumn{7}{@{}l}{\quad baseline change, $F_1$ points: +0.30 [+0.04; +0.61]} \\
Gate tuned & 1,064 / 203 (19.1) & 299 / 2 (0.7) & 3.30 & 4.59 & 1.97 & $\checkmark$ \\
\addlinespace[1pt]
\multicolumn{7}{@{}l}{\textit{(D)}} \\
Paper & 3,562 / 445 (12.5) & 1,113 / 6 (0.5) & 1.47 & 1.62 & 0.39 & $\checkmark$ \\
Tuned HGB & 3,523 / 540 (15.3) & 906 / 10 (1.1) & 2.53 & 2.81 & 0.49 & $\checkmark$ \\
\multicolumn{7}{@{}l}{\quad baseline change, $F_1$ points: +0.25 [+0.13; +0.39]} \\
Gate tuned & 3,562 / 445 (12.5) & 1,217 / 7 (0.6) & 1.47 & 1.61 & 0.35 & $\checkmark$ \\
\bottomrule
\end{tabular}

\vspace{2pt}
\parbox{\linewidth}{\justifying\scriptsize *Naive, BRIDGE: accepted / false improvements (false \%). Crit.: pre-registered criterion met ($\checkmark$) or not met ($\times$). Paper: the paper's design, fixed HGB and frozen gate. Same trials as the paper's runs (same seeds, splits and bags; the paper-design rows reproduce the stored decisions exactly in all 4,800 trials). Tuned HGB: tuned baseline, Optuna TPE, 25 iterations, 3-fold cross-validation inside SEARCH only, the HGB space and budget of the tuning runs; the tuned HGB replaces the fixed one as current model and member. Gate tuned: gate tuned per task, POOL split into two thirds and one third; the gate runs on the two thirds for every $q\in\{2.5,5,10\}\%$ and $\varepsilon\in\{0.0025,0.005,0.0075,0.01\}$, the chosen candidate is scored on the third, and the best pair (ties: the default) is applied on the full POOL. Brackets: cluster bootstrap over datasets, 10,000 resamples. Pre-registered criterion as in Table~\ref{tab:criterion}. Pre-registration MINI-44 (run list: Table~\ref{tab:chronology2}). Bound, LOCO, LOCO$_0$ in \%.}
\end{table}

\begin{table}[!htb]
\centering
\scriptsize
\setlength{\tabcolsep}{2.4pt}
\caption{Parameter sensitivity on the synthetic tasks (B).}
\label{tab:sens31}
\begin{tabular}{@{}lrrrrrr@{}}
\toprule
Variant & Acc. & False & $U$, \% $\downarrow$ & Cons., \% $\downarrow$ & Yield $\uparrow$ & Gain $\uparrow$ \\
\midrule
$F{=}2$, $q{=}2.5\%$ & 382 & 2 & 2.00 & 2.45 & 27.3 & 13.81 \\
$F{=}2$, $q{=}5\%$ & 445 & 4 & 2.47 & 4.41 & 31.8 & 15.09 \\
$F{=}2$, $q{=}10\%$ & 548 & 10 & 3.59 & 5.50 & 39.1 & 16.89 \\
$F{=}3$, $q{=}2.5\%$ & 406 & 4 & 2.70 & 3.08 & 29.0 & 14.34 \\
\textbf{$F{=}3$, $q{=}5\%$ (standard)} & \textbf{475} & \textbf{8} & \textbf{3.56} & \textbf{4.70} & \textbf{33.9} & \textbf{15.61} \\
$F{=}3$, $q{=}10\%$ & 563 & 16 & 4.90 & 6.24 & 40.2 & 17.12 \\
$F{=}3$, $q{=}5\%$, $N_{\mathrm{boot}}{=}2{,}000$ & 481 & 7 & 3.22 & 4.37 & 34.4 & 15.73 \\
$F{=}5$, $q{=}2.5\%$ & 414 & 5 & 3.02 & 4.07 & 29.6 & 14.39 \\
$F{=}5$, $q{=}5\%$ & 483 & 10 & 4.07 & 5.70 & 34.5 & 15.62 \\
$F{=}5$, $q{=}10\%$ & 575 & 16 & 4.80 & 6.96 & 41.1 & 17.24 \\
\bottomrule
\end{tabular}

\vspace{2pt}
\parbox{\linewidth}{\justifying\scriptsize *Acc., False: accepted and false improvements; $U$: Proposition~\ref{prop:bound} bound at $\delta{=}0.05$; Cons.: conservative variant (Corollary~\ref{cor:cons}); yield: updates accepted per $100$ decisions; gain: total reference gain. The standard configuration (bold) is mid-range.}
\end{table}

\paragraph{Matched-yield comparison.} In the matched-yield comparison on the $11$ new datasets, BRIDGE and the corrected $t$-test are tied at $0\%$ below $125$ accepted updates (Fig.~\ref{fig:yield}). Between $125$ and $450$ accepted updates, BRIDGE is $0.4$ to $1.1$ percentage points lower; the curves cross near $450$, both at $4.0\%$, and by $550$ the corrected $t$-test is $1.6$ points lower ($7.6\%$ against $6.0\%$). BRIDGE is below the $5\times2$~cv test at every counted yield, by up to $4.0$ points.

\paragraph{Fixed validation margin.} A fixed validation margin offers a simple and strong alternative to the gate. The pre-registered margin $t^*{=}0.05$ was calibrated on the development size of $2{,}430$ rows and fails under fresh seeds in the two settings that the pre-registration named: the tasks with $6{,}000$ rows and the new datasets under weighted-$F_1$. A margin selected post hoc passes all settings. These comparisons are exploratory: the claim is that BRIDGE reaches a risk close to that of a well-chosen fixed margin without choosing a margin per task, not that BRIDGE dominates the margin. The fixed-margin paragraph of Appendix~\ref{app:perdataset} gives the per-setting yields, the leave-one-setting-out analysis, and run MINI-29 (Table~\ref{tab:chronology}).

\paragraph{Parameter sensitivity on the real data.} Collection (D) passes every variant tested (Table~\ref{tab:scope}). In (A), the $16$- and $64$-bag variants pass the criterion, but the TOST does not establish equivalence with the default within $\pm2$ percentage points. Two variants fail in (A): a bag drawn from all of SEARCH, a fraction of $100\%$ (bound $13.95\%$, LOCO $21.33\%$), and a SEARCH of $1{,}000$ rows, which finds no false improvement on the six datasets tested but fails the power check ($\mathrm{LOCO}_0$ $20.77\%$). The run labels and full numbers are in Table~\ref{tab:citedruns}.

\paragraph{Computational cost.} The gate re-scores stored predictions instead of retraining, adding about $0.03$~s per run, or $0.4$--$3.8\%$ of the time to fit the bagged base models on SEARCH, and its inference cost equals that of the deployed predictor (Table~\ref{tab:complexity}). The $16$-bag setting also meets the criterion (Table~\ref{tab:citedruns}), so the fitting cost can be halved; the $16$- and $64$-bag variants are in Table~\ref{tab:scope}.

\begin{table}[!htb]
\centering
\scriptsize
\setlength{\tabcolsep}{2.4pt}
\caption{Wall-clock cost of BRIDGE.}
\label{tab:complexity}
\begin{tabular}{@{}lrrrrr@{}}
\toprule
 & MIC & Mktg & Census & Homes. & Sant. \\
\midrule
POOL rows & 741 & 741 & 1,800 & 1,800 & 1,800 \\
Features & 104 & 25 & 41 & 297 & 369 \\
Fit bags on SEARCH (s) & 1.39 & 0.73 & 1.30 & 6.19 & 4.24 \\
Naive selection (s) $\downarrow$ & 0.006 & 0.006 & 0.006 & 0.007 & 0.007 \\
Cross-fitted selection (s) $\downarrow$ & 0.017 & 0.017 & 0.019 & 0.019 & 0.019 \\
Bootstrap (s) $\downarrow$ & 0.011 & 0.011 & 0.010 & 0.009 & 0.006 \\
BRIDGE gate total (s) $\downarrow$ & 0.027 & 0.028 & 0.029 & 0.027 & 0.024 \\
Inference, baseline (ms/row) $\downarrow$ & 0.174 & 0.099 & 0.055 & 0.145 & 0.160 \\
Inference, ensemble (ms/row) $\downarrow$ & 0.63 & 0.33 & 0.27 & 1.17 & 0.78 \\
Gate / fit (\%) $\downarrow$ & 1.96 & 3.84 & 2.20 & 0.44 & 0.58 \\
\bottomrule
\end{tabular}

\vspace{2pt}
\parbox{\linewidth}{\justifying\scriptsize *Columns: MIC; Mktg, Marketing\_Campaign; Census, Census-Income-KDD; Homes., Homesite; Sant., Santander2016 (all datasets of (A)). Wall-clock cost (mean of 3 trials, one CPU core, Intel Core i9-13900K): fitting the bagged base models on SEARCH (32 bags each of HGB, DT and SVC), naive selection, the BRIDGE gate (cross-fitted selection and bootstrap) and inference latency of the deployed predictor (batch of 2,000 rows; baseline = 32 HGB bags, ensemble = 96 bagged models). Inference cost equals that of the deployed predictor, independent of BRIDGE. Gate total: cross-fitted selection plus bootstrap; percentages use unrounded times.}
\end{table}

\section{Per-Dataset Results and Paired Comparisons} \label{app:perdataset} \apptables{tab:perdataset}{tab:ci} For each rule, Table~\ref{tab:perdataset} gives the counts of accepted updates and false improvements on each of the $11$ new datasets. The naive rule accepts the most false improvements on three datasets: KDD98 ($83$ accepted, $54$ false improvements), Santander ($90/39$), and seismic ($94/35$). BRIDGE rejects every update on these three datasets.

\begin{table}[!htb]
\centering
\scriptsize
\setlength{\tabcolsep}{2.2pt}
\caption{Per-dataset results on the 11 new datasets.}
\label{tab:perdataset}
\begin{tabular}{@{}lrrrrrr@{}}
\toprule
Dataset & Naive & $5\times2$ & Corr. & Re-sel. & No-m. & BRIDGE \\
\midrule
Census-KDD & 100/0 & 60/0 & 86/0 & 40/0 & 55/0 & 17/0 \\
Homesite & 99/3 & 49/0 & 71/0 & 41/0 & 53/0 & 30/0 \\
KDD98 & 83/54 & 7/3 & 8/6 & 3/3 & 2/1 & 0/0 \\
LICD & 100/0 & 79/0 & 90/0 & 79/0 & 86/0 & 72/0 \\
MIC & 100/1 & 75/0 & 96/1 & 72/1 & 82/1 & 71/1 \\
Marketing & 100/0 & 78/0 & 98/0 & 81/0 & 87/0 & 76/0 \\
Pumpkin & 100/32 & 24/6 & 35/8 & 4/1 & 9/1 & 3/0 \\
Rice & 98/20 & 28/3 & 42/5 & 10/1 & 17/2 & 3/0 \\
Santander & 90/39 & 7/4 & 12/5 & 2/2 & 4/2 & 0/0 \\
pc3 & 100/19 & 42/8 & 62/10 & 17/0 & 33/5 & 12/2 \\
seismic & 94/35 & 12/5 & 8/3 & 0/0 & 3/0 & 0/0 \\
Total & 1064/203 & 461/29 & 608/38 & 349/8 & 431/12 & 284/3 \\
\bottomrule
\end{tabular}

\vspace{2pt}
\parbox{\linewidth}{\justifying\scriptsize *Per-dataset results on the 11 new datasets, 100 trials each (1,100 in total). Columns: the naive rule; the $5\times2$~cv paired $t$-test; the corrected resampled $t$-test (Corr.); the corrected $t$ with per-split re-selection (Re-sel.); BRIDGE without the margin (No-m.); BRIDGE. Dataset names shortened (Table~\ref{tab:splits}). Cells: accepted / false improvements per rule; false improvement = gain $\le 0$ on the locked reference (Definition~\ref{def:risk}); 0/0 = no acceptance (false-improvement rate undefined). Totals reproduce Table~\ref{tab:rules}.}
\end{table}

\paragraph{Bags and skipped trials.} A drawn bag is retained when every class contributes at least $10$ examples, so that each class is represented and the internal validation split of the base model remains useful. The multiclass run uses a minimum of $2$ examples instead (Appendix~\ref{app:prereg}). When a bag is discarded, the bagged score averages the remaining bags of the $32$ drawn. A trial is skipped when the minority class is too small for any bag to meet the rule. All skipped trials belong to the training-size study with $1{,}000$ rows. In that study porto-seguro is skipped in all trials, and avida-hil6 is skipped in some. The counts are $100$ and $3$ for the confirmatory run MINI-20 and $40$ and $2$ for the development run MINI-17 (runs in Table~\ref{tab:chronology}). Every other run has complete trials.

\paragraph{Paired comparisons.} Table~\ref{tab:ci} reports cluster-bootstrap confidence intervals for the difference in false-improvement rate between BRIDGE and each comparator. Each interval is computed at the number of accepted updates of that rule. The gain over BRIDGE without the margin is small and significant. This gain is $-1.7$ pp ($[-4.8;\,-0.3]$) on the real data and $-1.8$ pp ($[-3.1;\,-0.7]$) on $\mathcal{Q}$.

\begin{table}[!htb]
\centering
\scriptsize
\setlength{\tabcolsep}{3pt}
\caption{Paired comparisons: the difference BRIDGE minus each comparator.}
\label{tab:ci}
\begin{tabular}{@{}llrr@{}}
\toprule
Comparator & $\Delta$, pp [95\% CI] & $p$ & Holm $p$ \\
\midrule
\multicolumn{4}{@{}l}{\textit{Real data (11 new datasets)}} \\
naive & $-18.0\ [-30.2;\,-7.1]$ & $<$0.0002$^{\dagger}$ & $<$0.0012$^{\dagger}$ \\
McNemar, same POOL & $-11.0\ [-19.2;\,-1.5]$ & 0.0026 & 0.0078 \\
$5\times2$~cv $t$ & $-5.2\ [-12.2;\,-1.4]$ & 0.0006 & 0.0024 \\
corrected $t$ & $-5.2\ [-11.7;\,-1.6]$ & 0.0002 & 0.0010 \\
BRIDGE without margin & $-1.7\ [-4.8;\,-0.3]$ & 0.0046 & 0.0092 \\
corrected $t$, re-selection & $-1.2\ [-6.3;\,+1.3]$ & 0.32 & 0.32 \\
same-POOL cell, own point & $-1.6\ [-5.1;\,+0.2]$ & 0.22 & -- \\
same-POOL cell, matched yield & $-0.4\ [-1.9;\,+1.0]$ & 1.00 & -- \\
margin $0.025$, 11 new datasets & $-0.2\ [-2.2;\,+0.6]$ & 0.68 & -- \\
margin $0.025$, 1,000 rows & $-2.2\ [-4.2;\,-0.6]$ & 0.0008 & -- \\
\addlinespace[1pt]
\multicolumn{4}{@{}l}{\textit{Synthetic population $\mathcal{Q}$}} \\
naive & $-19.0\ [-21.3;\,-16.8]$ & $<$0.0002$^{\dagger}$ & $<$0.0012$^{\dagger}$ \\
McNemar, same POOL & $-33.1\ [-41.2;\,-25.4]$ & $<$0.0002$^{\dagger}$ & $<$0.0012$^{\dagger}$ \\
$5\times2$~cv $t$ & $-3.9\ [-5.8;\,-2.1]$ & $<$0.0002$^{\dagger}$ & $<$0.0012$^{\dagger}$ \\
corrected $t$ & $-4.0\ [-5.7;\,-2.4]$ & $<$0.0002$^{\dagger}$ & $<$0.0012$^{\dagger}$ \\
BRIDGE without margin & $-1.8\ [-3.1;\,-0.7]$ & $<$0.0002$^{\dagger}$ & $<$0.0012$^{\dagger}$ \\
corrected $t$, re-selection & $-1.3\ [-2.8;\,0.0]$ & 0.051 & 0.051 \\
same-POOL cell, own point & $-1.5\ [-2.7;\,-0.5]$ & 0.0028 & -- \\
same-POOL cell, matched yield & $+1.1\ [+0.2;\,+2.1]$ & 0.012 & -- \\
margin $0.025$, 11 new datasets & $-0.3\ [-1.4;\,+0.9]$ & 0.64 & -- \\
margin $0.02$ & $-2.5\ [-4.1;\,-1.0]$ & 0.0018 & -- \\
\bottomrule
\end{tabular}

\vspace{2pt}
\parbox{\linewidth}{\justifying\scriptsize *$\Delta$: false-improvement rate of BRIDGE minus that of the comparator, in percentage points [95\% CI]. $p$: two-sided. Holm $p$: adjusted over the six comparators. The fixed margins are selected post hoc. Paired cluster bootstrap at each rule's own operating point: in every resample both rules are evaluated on the same resampled decisions; clusters are the 11 datasets (real data) or the 1,400 tasks ($\mathcal{Q}$); 10,000 resamples, seed 20261012; percentile 95\% intervals. Two-sided p = 2 $\times$ min(share of resampled differences $\geq 0$, share $\leq 0$), resolution 2/10,000 = 0.0002.$^{\dagger}$ No resampled difference fell on the opposite side of zero (0 of 10,000): the empirical p is below its resolution, which is not an upper confidence bound on the true tail probability (likewise for the Holm-adjusted value); resamples in which a rule accepts nothing are left out for that rule. Holm adjustment over the six comparators; each test has the null hypothesis of equal false-improvement rates against the alternative of a difference. The same-POOL cell rows (at each rule's own operating point, and at matched yield, i.e., the same number of top-ranked accepted updates) and the fixed-margin rows are post-hoc, exploratory comparisons with nominal intervals, outside the Holm family.}
\end{table}

The corrected $t$-test with per-split re-selection and the best fixed margin chosen post hoc reach false-improvement rates that do not differ significantly from the false-improvement rate of BRIDGE. Two fixed margins are exceptions: $t=0.025$ at $1{,}000$ training rows and $t=0.02$ on $\mathcal{Q}$. In these two settings BRIDGE is lower by $2.2$ and $2.5$ percentage points. The table note gives the test construction, the Holm correction over the six comparisons, and the resolution of the $p$-values; the fixed-margin rows are exploratory and have nominal intervals.

The difference between BRIDGE and the same-POOL cell with the bootstrap lower bound and the margin, at $173$ candidates, is also in Table~\ref{tab:ci}. The difference is computed twice, at the number of accepted updates of each rule and at matched numbers. Section~\ref{sec:exp-ablation} discusses both results.

\paragraph{Fixed-margin detail.} The pre-registered grid of fixed margins spans $0.005$ to $0.050$ in steps of $0.005$. Chosen post hoc in sample, the values $0.020$ to $0.035$ pass all $7$ settings. The seven settings are the tasks with $1{,}000$, $2{,}430$, and $6{,}000$ training rows, the new datasets under weighted-$F_1$, balanced accuracy, and macro-$F_1$, and the population $\mathcal{Q}$.

The yield of a fixed margin depends strongly on the chosen value. Across the weighted-$F_1$ settings the yield ranges from $5.4$ to $49.4$ per $100$ decisions, against $21.5$ to $33.9$ for BRIDGE (computed from the stored runs of Table~\ref{tab:marginsweep}). On the tasks with $6{,}000$ rows, the fixed margin reaches a yield of $5.4$ to $21.0$, below the $30.6$ of BRIDGE. Near $t{=}0.05$ the fixed margin accepts few updates at the larger sizes. The yields are $1.4$ and $0.1$, and the bounds are $16.1\%$ and $100\%$ (MINI-29, Table~\ref{tab:chronology}). By contrast, the yield of BRIDGE grows with the sample size: $21$, $25$, and $31$ at $1{,}000$, $2{,}430$, and $6{,}000$ rows (Table~\ref{tab:general}).

Against the best fixed margin, the false-improvement rate of BRIDGE does not differ significantly, either on the $11$ new datasets or on $\mathcal{Q}$. At the smallest training size, BRIDGE has a lower false-improvement rate than the best fixed margin ($-2.2$ pp, nominal $95\%$ CI $[-4.2;\,-0.6]$). On large data BRIDGE accepts more improvements, with a yield of $31$ against $14$ per $100$. The difference is $+16.8$ ($[+10.2;\,+23.7]$), computed with the paired bootstrap of Table~\ref{tab:ci}. An exploratory leave-one-setting-out analysis, added after the pre-registered results, chooses the best fixed margin from six settings and tests it on the seventh, in turn. The fixed margin passes six of the seven settings and fails on the new datasets under weighted-$F_1$ (LOCO $10.99\%$, bound $8.50\%$), whereas BRIDGE passes all seven, with a bound of $4.23\%$ on that setting.

\paragraph{Base-learner audit.} The base-learner audit (Table~\ref{tab:audit}) shows that AUC ranges from $0.566$ to $0.579$ across learners, while the number of boosting rounds, log-loss, extreme-probability rate (below $0.01$ or above $0.99$), and fit time differ. The differences relevant to the gate therefore lie in the predicted probabilities and in iteration control, not in the ranking of cases (Section~\ref{sec:exp-base}).

\begin{table}[!htb]
\centering
\scriptsize
\setlength{\tabcolsep}{2.2pt}
\caption{Audit of the base learners.}
\label{tab:audit}
\begin{tabular}{@{}lrrrrrr@{}}
\toprule
Learner & AUC $\uparrow$ & Log-loss $\downarrow$ & Extreme, \% $\downarrow$ & wF1 $\uparrow$ & Rounds & Fit, s $\downarrow$ \\
\midrule
HGB & 0.578 & 0.675 & 0.00 & 0.552 & 16.67 & 0.028 \\
XGBoost, default & 0.566 & 0.810 & 1.68 & 0.559 & 300.00 & 0.350 \\
XGBoost, regul. & 0.579 & 0.673 & 0.00 & 0.520 & 7.50 & 0.032 \\
\bottomrule
\end{tabular}

\vspace{2pt}
\parbox{\linewidth}{\justifying\scriptsize *Extreme: share of predicted probabilities below 0.01 or above 0.99; wF1: weighted-$F_1$ at threshold 0.5; rounds: boosting rounds in use; fit: fit time per bag. Audit of the base learners on the locked reference data (median over 1,400 tasks; mean for the share of probabilities below 0.01 or above 0.99; descriptive).}
\end{table}

\paragraph{Dataset-level bootstrap.} A dataset-level bootstrap upper bound, which accounts for dataset clustering rather than conditioning on the observed collection, gives $3.1\%$ for BRIDGE and $29.4\%$ for the naive rule on the $11$ new datasets of (A), against $13.0\%$ and $12.5\%$ for the $5\times2$~cv and corrected $t$ tests (Table~\ref{tab:citedruns}). On $\mathcal{Q}$ the $5\times2$~cv and corrected $t$ tests meet the criterion (bounds $8.14\%$ and $7.95\%$, Table~\ref{tab:synthetic}) while failing on the real data. Under a stricter $5\%$ threshold, BRIDGE passes on $\mathcal{Q}$ (Proposition~\ref{prop:bound} bound $3.56\%$) and on the real-data KL-ratio bound ($4.23\%$) but not on the real-data LOCO ($5.95\%$).

\FloatBarrier
\section{Patient-Identity Audit}
\label{app:identity}
\apptables{tab:identity}{tab:clinsource}
The patient-identity audit (Table~\ref{tab:identity}) records, for each of the six clinical sources, the identity key, the row and unique-key counts, the excess duplicate rows and the rows actually removed; de-duplication and filtering are applied in the loaders before any split. On the co-primary net-benefit metric, the few accepted changes that are not better than both ``treat none'' and ``treat all'' are concentrated in the HIV-1 cohort (ACTG 175), where the prevalence exceeds $p_t$ (Table~\ref{tab:clinconseq}). The mammography source carries no patient identifier, so its inclusion cannot be verified at the patient level, and the primary result is unchanged when that source is removed (Table~\ref{tab:clinsource}).

\begin{table}[!htb]
\centering
\scriptsize
\setlength{\tabcolsep}{2.4pt}
\caption{Patient-identity audit of the clinical sources.}
\label{tab:identity}
\begin{tabular}{@{}llrrr@{}}
\toprule
Source & Identity key & Rows & Unique keys & Dup. \\
\midrule
WiDS 2020 (GOSSIS) & patient\_id & 91,713 & 91,713 & 0 \\
\multicolumn{5}{@{}l}{\quad one row per patient} \\
\multicolumn{5}{@{}l}{\quad unique encounter\_id: 91,713} \\
\multicolumn{5}{@{}l}{\quad hospitals: 147} \\
COVID-19 Mexico & id (record id) & 566,602 & 563,201 & 3,401 \\
\multicolumn{5}{@{}l}{\quad duplicates are excess rows} \\
\multicolumn{5}{@{}l}{\quad ids occurring more than once: 1,486} \\
\multicolumn{5}{@{}l}{\quad 4,887 rows with such ids removed before filtering} \\
\multicolumn{5}{@{}l}{\quad confirmed cases kept: 218,784} \\
NHANES 2017-2018 & SEQN & 9,254 & 9,254 & 0 \\
\multicolumn{5}{@{}l}{\quad one row per respondent (survey design)} \\
\multicolumn{5}{@{}l}{\quad the 4 outcomes share respondents and form one LOCO cluster} \\
ACTG 175 (UCI 890) & none in UCI copy & 2,139 &  &  \\
\multicolumn{5}{@{}l}{\quad randomized trial, one row per participant (documentation)} \\
BRFSS 2015 & SEQNO (state-level) &  &  &  \\
\multicolumn{5}{@{}l}{\quad one record per interviewed respondent (survey design)} \\
\multicolumn{5}{@{}l}{\quad the 4 outcomes form one LOCO cluster} \\
Mammography (OpenML 310) & none &  &  &  \\
\multicolumn{5}{@{}l}{\quad identical (X, y) rows kept once} \\
\multicolumn{5}{@{}l}{\quad identical X with conflicting labels removed} \\
\multicolumn{5}{@{}l}{\quad patient independence cannot be verified} \\
\bottomrule
\end{tabular}

\vspace{2pt}
\parbox{\linewidth}{\justifying\scriptsize *Dup.: excess duplicate rows. Indented rows: notes per source. Patient-identity audit of the clinical sources (de-duplication and filtering applied in the loaders before any split). Participant identifier unavailable in the UCI copy of ACTG 175 and patient identifier unavailable in the OpenML mammography data; the BRFSS record number SEQNO is unique only within a state.}
\end{table}

\begin{table}[!htb]
\centering
\scriptsize
\setlength{\tabcolsep}{2.2pt}
\caption{Patient-identity audit: clinical results by source.}
\label{tab:clinsource}
\begin{tabular}{@{}lrrrrrr@{}}
\toprule
Source or subset & BRIDGE & Unc. & Bound $\downarrow$ & Cons. $\downarrow$ & LOCO $\downarrow$ & Naive \\
\midrule
\multicolumn{7}{@{}l}{\textit{Net benefit, $p_t=0.2$ (co-primary)}} \\
ACTG 175 & 198/0 & 2 & -- & -- & -- & 200/0 \\
BRFSS (4 tasks) & 210/0 & 0 & -- & -- & -- & 684/282 \\
COVID-19 Mexico & 200/0 & 0 & -- & -- & -- & 200/0 \\
Mammography & 10/0 & 1 & -- & -- & -- & 200/0 \\
NHANES (4 tasks) & 401/0 & 1 & -- & -- & -- & 782/171 \\
WiDS ICU & 180/0 & 0 & -- & -- & -- & 200/0 \\
All sources & 1,199/0 & 4 & 0.33 & 1.07 & 0.50 & 2,266/453 \\
without ACTG 175 & 1,001/0 & 2 & 0.39 & 0.92 & 0.67 & 2,066/453 \\
without BRFSS & 989/0 & 4 & 0.39 & 1.29 & 0.67 & 1,582/171 \\
without COVID-19 Mexico & 999/0 & 4 & 0.39 & 1.29 & 0.67 & 2,066/453 \\
without Mammography & 1,189/0 & 3 & 0.33 & 0.92 & 0.50 & 2,066/453 \\
without NHANES & 798/0 & 3 & 0.50 & 1.40 & 0.67 & 1,484/282 \\
without WiDS ICU & 1,019/0 & 4 & 0.39 & 1.27 & 0.65 & 2,066/453 \\
\addlinespace[1pt]
\multicolumn{7}{@{}l}{\textit{Balanced accuracy (co-primary)}} \\
ACTG 175 & 177/0 & 23 & -- & -- & -- & 200/0 \\
BRFSS (4 tasks) & 783/0 & 0 & -- & -- & -- & 800/0 \\
COVID-19 Mexico & 200/0 & 0 & -- & -- & -- & 200/0 \\
Mammography & 200/0 & 0 & -- & -- & -- & 200/0 \\
NHANES (4 tasks) & 713/8 & 104 & -- & -- & -- & 800/8 \\
WiDS ICU & 200/0 & 0 & -- & -- & -- & 200/0 \\
All sources & 2,273/8 & 127 & 0.82 & 7.10 & 1.25 & 2,400/8 \\
without ACTG 175 & 2,096/8 & 104 & 0.88 & 6.45 & 1.42 & 2,200/8 \\
without BRFSS & 1,490/8 & 127 & 1.25 & 10.85 & 1.45 & 1,600/8 \\
without COVID-19 Mexico & 2,073/8 & 127 & 0.90 & 7.79 & 1.45 & 2,200/8 \\
without Mammography & 2,073/8 & 127 & 0.90 & 7.79 & 1.45 & 2,200/8 \\
without NHANES & 1,560/0 & 23 & 0.24 & 2.49 & 0.48 & 1,600/0 \\
without WiDS ICU & 2,073/8 & 127 & 0.90 & 7.79 & 1.45 & 2,200/8 \\
\bottomrule
\end{tabular}

\vspace{2pt}
\parbox{\linewidth}{\justifying\scriptsize *Confirmatory clinical run with $200$ trials per task (800 for the four BRFSS or NHANES tasks, 2,400 for all sources). BRIDGE and Naive: accepted / false improvements. Unc.: the accepted updates that are reference-uncertain. False improvement: gain $\le 0$ on the locked reference; reference-uncertain: the Bayesian-bootstrap interval of the reference gain contains $0$. Bound: KL-ratio bound; conservative: every accepted update that is not clearly improving counted as a false improvement; LOCO: leave-one-source-out. Single-source rows give counts only. Bound, Cons., LOCO in \%.}
\end{table}

The naive rule accepts the most false improvements on three clinical tasks: BRFSS kidney disease ($133$ false improvements of $133$ accepted), BRFSS stroke ($144$ of $151$), and NHANES coronary heart disease ($97$ of $182$). BRIDGE accepts no update on any of the three tasks. On the real data the same pattern appears: BRIDGE accepts no update on KDD98, Santander, and seismic (Table~\ref{tab:perdataset}).

\section{OpenML Collection (E): Datasets, Extension, and Sensitivities}
\label{app:settingE}
\apptables{tab:openmlE}{tab:openmlE}

\begin{table*}[!t]
\centering
\scriptsize
\setlength{\tabcolsep}{2.2pt}
\caption{Datasets of the OpenML collection (E) and its extension (E2), with per-dataset decisions and curation flags.}
\label{tab:openmlE}
\begin{minipage}[t]{0.494\textwidth}\centering
\begin{tabular}[t]{@{}llrrrrr@{}}
\toprule
Dataset & Source & N & K & Feat. & Naive & BRIDGE \\
\midrule
\multicolumn{7}{@{}l}{\textit{(E) $14$ OpenML datasets}} \\
airline-satisfaction & OML~46920 & 129,880 & 2 & 21 & 99/6 & 22/0 \\
KDDCup09-churn & OML~1112 & 50,000 & 2 & 230 & 80/62 & 0/0 \\
bank-marketing & OML~1461 & 45,211 & 2 & 16 & 100/0 & 53/0 \\
pulsar-HTRU2 & OML~45558 & 17,898 & 2 & 8 & 97/45 & 0/0 \\
jm1 & OML~1053 & 10,885 & 2 & 21 & 100/2 & 63/0 \\
support2 & OML~46860 & 9,105 & 2 & 46 & 93/26 & 5/0 \\
telco-customer-churn & OML~42178 & 7,043 & 2 & 19 & 99/22 & 15/1 \\
phoneme & OML~1489 & 5,404 & 2 & 5 & 96/26 & 3/0 \\
spambase & OML~44 & 4,601 & 2 & 57 & 98/30 & 5/1 \\
sick & OML~38 & 3,772 & 2 & 29 & 100/16 & 4/0 \\
bioresponse & OML~4134 & 3,751 & 2 & 1,776 & 94/57 & 1/0 \\
ipums-la-97-small & OML~382 & 7,019 & 8 & 60 & 60/14 & 2/0 \\
otto-group & OML~45548 & 61,878 & 9 & 93 & 60/7 & 5/0 \\
corporate-credit-ratings & OML~46372 & 5,000 & 10 & 7 & 60/4 & 14/0 \\
\multicolumn{7}{@{}l}{\textit{(E2) $20$ OpenML and UCI datasets}} \\
in-vehicle-coupon & OML~46937 & 12,684 & 2 & 24 & 86/32 & 4/2 \\
wilt & OML~40983 & 4,839 & 2 & 5 & 100/1 & 28/0 \\
\bottomrule
\end{tabular}
\end{minipage}\hfill
\begin{minipage}[t]{0.494\textwidth}\centering
\begin{tabular}[t]{@{}llrrrrr@{}}
\toprule
Dataset & Source & N & K & Feat. & Naive & BRIDGE \\
\midrule
APSFailure & OML~41138 & 76,000 & 2 & 170 & 99/2 & 1/0 \\
iranian-churn & UCI~563 & 3,150 & 2 & 13 & 100/0 & 92/0 \\
SECOM & UCI~179 & 1,567 & 2 & 590 & 74/38 & 0/0 \\
steel-plates-faults & UCI~198 & 1,941 & 2 & 27 & 100/2 & 31/0 \\
myocardial-infarction & UCI~579 & 1,700 & 2 & 111 & 99/40 & 10/1 \\
OVA-breast & OML~1128 & 1,545 & 2 & 10,935 & 100/33 & 12/1 \\
lending-club & OML~43729 & 9,578 & 2 & 13 & 100/4 & 24/0 \\
TVS-loan-default & OML~43743 & 119,528 & 2 & 30 & 83/47 & 0/0 \\
personal-loan & OML~43826 & 5,000 & 2 & 13 & 97/31 & 5/0 \\
fitness-club$^{c}$ & OML~46927 & 1,500 & 2 & 6 & 96/30 & 3/0 \\
good-customer$^{c}$ & OML~46938 & 1,723 & 2 & 13 & 99/22 & 5/1 \\
churn$^{b}$ & OML~46915 & 5,000 & 2 & 19 & 100/0 & 100/0 \\
wine-white & OML~40498 & 4,898 & 7 & 11 & 60/8 & 5/1 \\
shuttle$^{f}$ & OML~40685 & 58,000 & 7 & 9 & -- & -- \\
covertype & OML~180 & 110,393 & 7 & 54 & 60/3 & 11/0 \\
CIFAR-10-small$^{e}$ & OML~40926 & 20,000 & 10 & 3,072 & 60/10 & 6/1 \\
SVHN-small$^{d}$ & OML~42140 & 9,927 & 10 & 3,072 & 60/0 & 58/0 \\
pendigits (amend.)$^{a}$ & OML~32 & 10,992 & 10 & 16 & 60/0 & 57/0 \\
\bottomrule
\end{tabular}
\end{minipage}

\vspace{2pt}
\parbox{\textwidth}{\justifying\scriptsize *(E): pre-registered OpenML collection; (E2): its pre-registered extension; both run through the released program with the same protocol (HGB, default candidate space, weighted-$F_1$; $100$ trials per binary and $60$ per multiclass dataset). Naive, BRIDGE: accepted / false improvements. Decisions per dataset: $100$ (binary) or $60$ (multiclass). K: classes. Feat.: features. OML: OpenML ID. Names shortened. Flags: $^{a}$rows grouped by $44$ writers, against the grouped-data rule (amendment for \texttt{shuttle}); $^{b}$original source lost, same data as OpenML 40701, described there as artificial; $^{c}$origin of the data not verifiable (Kaggle practice or unattributed data); $^{d}$flattened image whose source-photograph grouping cannot be audited; $^{e}$flattened image (stress test); $^{f}$all $60$ trials stopped with a structural error (a class too small after the reference split), replaced by the amendment. Flagged datasets are kept as pre-registered and reported; without the datasets flagged $a$--$c$, (E)+(E2) covers $29$ datasets and still meets the criterion (naive $22.6\%$ false, BRIDGE $479$ accepted / $8$ false improvements, $1.67\%$; bound $4.28\%$, LOCO $5.38\%$), and so do its $22$ binary datasets (naive $24.9\%$, BRIDGE $1.59\%$; bound $4.61\%$, LOCO $6.23\%$); the binary subgroup of (E2) alone does not ($11$ datasets, LOCO $13.49\%$).}
\end{table*}

The datasets of the OpenML collection (E) and its extension (E2), with the decisions and the curation flags of each, are listed in Table~\ref{tab:openmlE}. The OpenML collection (E) is the $14$-dataset collection of Section~\ref{sec:exp-settings}, evaluated with the frozen gate and reported in Table~\ref{tab:summary}. BRIDGE accepts $192$ of $1{,}280$ decisions with $2$ false improvements ($1.0\%$), bound $5.40\%$, LOCO $8.43\%$, $\mathrm{LOCO}_0$ $3.62\%$, yield $15$, and $8$ datasets with a yield of at least $5$. The pass is not uniform by class count. The binary subgroup of (E) ($11$ datasets) fails narrowly on LOCO ($10.30\%$), a limitation found after the results. The multiclass subgroup ($3$ datasets) has no false improvement, but its bound is too wide to pass ($32.59\%$, with LOCO and $\mathrm{LOCO}_0$ at $100\%$). The row is retained as first reported, not replaced by the extension.

A second pre-registered collection, E2, adds $19$ datasets and $1{,}700$ decisions under the same protocol and the same frozen gate; the amendment of its dataset list deviates from the curation rule and is recorded below. E2 passes jointly: BRIDGE accepts $452$ updates with $7$ false improvements ($1.55\%$), bound $4.16\%$, LOCO $5.44\%$, $\mathrm{LOCO}_0$ $1.19\%$, yield $26.6$, and $14$ datasets with a yield of at least $5$. The naive rule accepts $1{,}633$ updates with $303$ false improvements ($18.6\%$). The difference in false-improvement rate, BRIDGE minus naive, is $-17.0$ pp ($95\%$ CI $[-24.3,-9.6]$).

The joint result is read by class count, in addition to the pooled row. The binary subgroup of E2 ($14$ datasets, $1{,}400$ decisions) meets every criterion: $315$ accepted with $5$ false improvements ($1.59\%$), bound $5.00\%$, LOCO $7.60\%$, $\mathrm{LOCO}_0$ $2.04\%$, yield $22.5$, and $9$ datasets with a yield of at least $5$. The BRIDGE-minus-naive difference is $-19.6$ pp ($95\%$ CI $[-28.1,-10.7]$). This pass depends on \texttt{churn}, a dataset flagged $^{b}$ as artificial, which contributes $100$ accepted updates and no false improvement; the dependence was found after the results (a post hoc limitation). Without the flagged datasets, the combined binary subgroup of (E) and (E2) meets the LOCO condition ($22$ datasets; naive $24.9\%$, BRIDGE $1.59\%$; Table~\ref{tab:openmlE}). The binary subgroup of (E2) alone, without the datasets flagged $a$--$c$, exceeds the LOCO limit ($13.49\%$).

The multiclass subgroup of E2 ($5$ datasets, $300$ decisions) meets the requirements on the false-improvement rate and yield. This subgroup has $137$ accepted updates, $2$ false improvements ($1.46\%$), and a yield of $45.7$. The LOCO bound of the subgroup is $14.10\%$, above the $10\%$ condition, so the subgroup fails the criterion. The BRIDGE-minus-naive difference is $-5.5$ pp ($95\%$ CI $[-7.7,0.0]$).

Pooling the two pre-registered collections, E+E2 covers $33$ datasets and $2{,}980$ decisions and meets every criterion: BRIDGE accepts $644$ with $9$ false improvements ($1.40\%$), bound $3.38\%$, LOCO $4.04\%$, $\mathrm{LOCO}_0$ $0.75\%$, yield $21.6$, and $22$ datasets with a yield of at least $5$. The BRIDGE-minus-naive difference is $-20.2$ pp ($95\%$ CI $[-26.7,-13.8]$). The binary subgroup ($25$ datasets, $2{,}500$ decisions) also meets every criterion: $486$ accepted with $7$ false improvements ($1.44\%$), bound $3.87\%$, LOCO $4.96\%$, $\mathrm{LOCO}_0$ $1.09\%$, yield $19.4$, and $15$ datasets with a yield of at least $5$. The BRIDGE-minus-naive difference is $-22.6$ pp ($95\%$ CI $[-30.1,-15.3]$).

\paragraph{Deviation in the E2 amendment.} The frozen E2 list contains \texttt{shuttle}, whose $60$ trials all stop with a structural error because a class is too small after the reference split. An amendment, frozen before any replacement metric was read, replaced \texttt{shuttle} with \texttt{pendigits} (OpenML 32). The post-run provenance audit then found a conflict with the frozen curation rule. The $10{,}992$ rows of \texttt{pendigits} come from $44$ writers, and the rule excludes data grouped by writer, subject, or session, whereas the runner splits rows because the OpenML table carries no writer identifier. \texttt{pendigits} is therefore not treated as passing the primary curation. The deviation is recorded, the run is retained as executed, and no third dataset is chosen after the results are seen.

\paragraph{Provenance flags.} The same provenance audit flags three further datasets of (E2). These are \texttt{churn}, whose original source is lost and whose data are described as artificial; \texttt{Fitness\_Club}; and \texttt{Is-this-a-good-customer}, whose origin cannot be verified. All three remain reported as pre-registered, and no replacement dataset is introduced. Table~\ref{tab:openmlE} also reports the results without the datasets flagged $a$--$c$. With these datasets removed, the collection meets the criterion ($29$ datasets; naive $22.6\%$ false, BRIDGE $1.67\%$).

\paragraph{Sensitivity without \texttt{pendigits}.} Because of this deviation, the pooled claim rests on the sensitivity analysis without \texttt{pendigits}. This analysis covers $32$ datasets and $2{,}920$ decisions and meets every criterion: $587$ accepted, $9$ false improvements ($1.53\%$), bound $3.73\%$, LOCO $4.55\%$, $\mathrm{LOCO}_0$ $0.85\%$, yield $20.1$, and $21$ datasets with a yield of at least $5$. The BRIDGE-minus-naive difference is $-20.5$ pp ($95\%$ CI $[-27.1,-14.2]$). The multiclass subgroup of this analysis ($7$ datasets, $420$ decisions) meets the requirements on the false-improvement rate and yield. This subgroup has $101$ accepted updates, $2$ false improvements ($1.98\%$), and a yield of $24.0$. The subgroup fails all three checks: the bound ($10.87\%$), the LOCO condition ($30.27\%$), and the power check ($\mathrm{LOCO}_0$ $13.02\%$).

\paragraph{Image stress tests and the group structure of SVHN.} Two multiclass datasets of the extension, CIFAR-10 and the Street View House Numbers (SVHN), are flattened images, not tabular application data: each image contributes $3{,}072$ pixel features ($32\times32$ pixels in three colour channels), and both serve as high-dimensional stress tests. OpenML 42140 holds a $10\%$ stratified subsample of SVHN without an identifier of the source photograph, so a group split cannot audit the independence of crops from the same photograph. This gap is reported as a limitation, not as a confirmed violation. Without \texttt{pendigits} and SVHN, the multiclass subgroup ($6$ datasets, $360$ decisions) has $43$ accepted updates, $2$ false improvements ($4.65\%$), and a yield of $11.9$, and fails the bound ($30.27\%$), the LOCO condition ($49.85\%$), and the power check ($\mathrm{LOCO}_0$ $21.47\%$).

\paragraph{Volkert supplement.} A later supplement adds Volkert (OpenML 41166), an anonymous numerical tabular benchmark with $58{,}310$ rows, $180$ features, $10$ classes, and $60$ complete trials. The dataset was frozen after the confirmatory runs, so its analysis is exploratory. On Volkert alone, BRIDGE accepts $4$ updates and no false improvement, at a yield of $6.7$. The rule of at least three datasets with a yield of at least $5$ needs more than one dataset, so this result is reported separately. Pooled with E+E2, the supplement gives $34$ datasets and $3{,}040$ decisions that meet every criterion: BRIDGE accepts $648$ updates with $9$ false improvements ($1.39\%$), at a bound of $3.36\%$, LOCO $4.01\%$, and $\mathrm{LOCO}_0$ $0.75\%$. The yield is $21.3$, $23$ datasets reach a yield of at least $5$, and the BRIDGE-minus-naive difference is $-19.9$ pp ($95\%$ CI $[-26.4,-13.6]$).

Because the supplement is exploratory, the pooled row leaves the confirmatory claim unchanged. The multiclass subgroup of the pooled row ($9$ datasets, $540$ decisions) has a LOCO bound of $10.56\%$, slightly above the $10\%$ condition.

\paragraph{Collection outside the validated class range.} Under the same program and protocol as (E), $11$ further multiclass OpenML datasets were run as an exploratory collection with $60$ trials each. These datasets have $3$, $4$, $6$, $11$, or $26$ classes, outside the validated range of $7$ to $10$ classes. The collection meets the criterion jointly ($175$ accepted, $2$ false improvements; bound $5.89\%$, LOCO $8.02\%$). Its class-count groups with $K=3$ and $K=6$ fail the criterion. Each of the remaining groups holds one dataset and is described without a criterion. The result therefore varies with the number of classes (Table~\ref{tab:scope}).

\paragraph{One-pass analysis.} A one-pass analysis repeats the comparison with one decision per dataset instead of the $60$ or $100$ repeated seeds. On E2, BRIDGE accepts $7$ of $19$ decisions with no false improvement, against $2$ false improvements among $17$ naive acceptances ($11.8\%$); pooled over E+E2, BRIDGE accepts $8$ of $33$ with no false improvement, against $7$ among $31$ naive acceptances ($22.6\%$). These numbers are reported for completeness, because the main analysis rests on the full repetitions and a single decision per dataset carries no cluster interval.

\FloatBarrier
\section{Detailed Result Tables}
\label{app:results}
\apptables{tab:summary}{tab:clinconseq}

This appendix gives the detailed result tables behind Sections~\ref{sec:method} and~\ref{sec:experiments}, numbered in print order. Each of the five groups below lists its tables and the main-text subsection that uses each table. In all appendix tables, $\downarrow$ and $\uparrow$ mark columns where a smaller or a larger value is better. A bound is the one-sided $95\%$ upper bound on the false-improvement rate. The criterion is the pre-registered criterion of Table~\ref{tab:criterion}.

\begin{table*}[!t]
\centering
\scriptsize
\providecolor{bghead}{HTML}{DBEFF8}\providecolor{bgzebra}{HTML}{F2F9FC}\providecolor{bgpass}{HTML}{D7F0D2}\providecolor{bgfail}{HTML}{FCDCD7}
\setlength{\tabcolsep}{3pt}
\caption{Main results in detail, (A)--(E).}
\label{tab:summary}
\begin{tabular}{@{}lrrlrllrl@{}}
\toprule
Setting & Decisions & \begin{tabular}[b]{@{}c@{}}Naive:\\accepted\end{tabular} & \begin{tabular}[b]{@{}c@{}}Naive \% false\\{[}95\% CI{]} $\downarrow$\end{tabular} & \begin{tabular}[b]{@{}c@{}}BRIDGE:\\accepted\end{tabular} & \begin{tabular}[b]{@{}c@{}}BRIDGE \% false\\{[}95\% CI{]} $\downarrow$\end{tabular} & \begin{tabular}[b]{@{}c@{}}Reduction, pp\\{[}95\% CI{]} $\uparrow$\end{tabular} & \begin{tabular}[b]{@{}c@{}}Upper bound\\$\hat U$ or $U$ (\%) $\downarrow$\end{tabular} & \begin{tabular}[b]{@{}c@{}}False removed \%\\{[}95\% CI{]} $\uparrow$\end{tabular} \\
\midrule
(A) 11 new datasets & 1,100 & 1,064 & $19.1\ [8.1;\,31.6]$ & 284 & $1.1\ [0.0;\,4.0]$ & $18.0\ [7.1;\,30.2]$ & 4.23 & 99 [95; 100] \\
(B) Synthetic tasks & 1,400 & 1,323 & $20.7\ [18.5;\,22.9]$ & 475 & $1.7\ [0.6;\,2.9]$ & $19.0\ [16.8;\,21.3]$ & 3.56 & 97 [95; 99] \\
(C) Clinical tasks, net benefit & 2,400 & 2,266 & $20.0\ [4.3;\,40.4]$ & 1,199 & 0.0 & $20.0\ [4.5;\,40.3]$ & 0.33 & 100 [100; 100] \\
(D) 39 supporting datasets & 3,700 & 3,562 & $12.5\ [7.4;\,18.6]$ & 1,113 & $0.5\ [0.1;\,1.2]$ & $12.0\ [7.2;\,17.9]$ & 1.47 & 99 [97; 100] \\
(E) OpenML collection, 14 datasets & 1,280 & 1,236 & $25.6\ [15.0;\,37.8]$ & 192 & $1.0\ [0.0;\,4.0]$ & $24.6\ [14.0;\,36.0]$ & 5.40 & 99 [98; 100] \\
\bottomrule
\end{tabular}

\vspace{2pt}
\parbox{\linewidth}{\justifying\scriptsize CI = confidence interval. Reduction: naive minus BRIDGE, in percentage points. Intervals: cluster bootstrap of the random-selection analysis (seed 20261014); the paired intervals of Table~\ref{tab:rules} come from the comparator runs (seed 20261022), so the interval limits can differ in the first decimal. Over the $6$ clinical sources, the interval is $0.0$ to $31.9$ pp. Upper bound: the KL-ratio bound $\hat U$ \eqref{eq:Uhat} on the real data, and the Proposition~\ref{prop:bound} bound $U$ \eqref{eq:U} on (B). False removed: Table~\ref{tab:selectivity}.}
\end{table*}

\begin{table*}[!t]
\centering
\scriptsize
\setlength{\tabcolsep}{1.8pt}
\renewcommand{\arraystretch}{1.15}
\caption{Hypothesis tests: estimates and $p$-values, collection by collection.}
\label{tab:testsdetail}
\newcommand{\cell}[2]{\begin{tabular}[t]{@{}c@{}}#1\\#2\end{tabular}}
\begin{tabular}{@{}llccccc@{}}
\toprule
Test & \begin{tabular}[b]{@{}l@{}}Comparison\\(T1: BRIDGE minus)\end{tabular} & (A) & (B) & (C) & (D) & (E) \\
\midrule
T1 \eqref{eq:hdiff} & Naive & $-18.0$ ($p<0.0012^{*}$) & $-19.0$ ($p<0.0012^{*}$) & $-20.0$ ($p=0.034^{*}$) & $-12.0$ ($p=0.0009^{*}$) & $-24.6$ ($p<0.0002^{*}$) \\
 & McNemar, same POOL & $-11.0$ ($p=0.0078^{*}$) & $-33.1$ ($p<0.0012^{*}$) & $0.0$ ($p=1.00$) & $-6.1$ ($p=0.0009^{*}$) & -- \\
 & McNemar, cross-fit & -- & -- & $0.0$ ($p=1.00$) & $-2.0$ ($p=0.018^{*}$) & -- \\
 & $5\times2$~cv paired $t$ & $-5.2$ ($p=0.0024^{*}$) & $-3.9$ ($p<0.0012^{*}$) & $-2.5$ ($p=0.034^{*}$) & $-3.1$ ($p=0.0009^{*}$) & -- \\
 & Corrected resampled $t$ & $-5.2$ ($p=0.0010^{*}$) & $-4.0$ ($p<0.0012^{*}$) & $-2.6$ ($p=0.10$) & $-3.6$ ($p=0.0009^{*}$) & -- \\
 & Corrected $t$, re-selection & $-1.2$ ($p=0.32$) & $-1.3$ ($p=0.051$) & $-0.4$ ($p=0.34$) & $-1.0$ ($p=0.0009^{*}$) & -- \\
 & BRIDGE without margin & $-1.7$ ($p=0.0092^{*}$) & $-1.8$ ($p<0.0012^{*}$) & $-0.6$ ($p=0.34$) & $-1.1$ ($p=0.0016^{*}$) & -- \\
 & Holdout 2/3 + 1/3 & -- & -- & $0.0$ ($p=1.00$) & $+0.3$ ($p=0.47$) & -- \\
 & Holdout 1/2 + 1/2 & -- & -- & $0.0$ ($p=1.00$) & $-0.3$ ($p=0.47$) & -- \\
\midrule
T2 \eqref{eq:hequiv} & Margins vs $0.005$ & equivalent ($p\le0.0498^{*}$) & equivalent ($p\le0.0498^{*}$) & equivalent ($p<10^{-4\,*}$) & equivalent ($p<10^{-4\,*}$) & -- \\
\midrule
T3 \eqref{eq:hrandom} & Random draw, pooled & $54.2$ vs $3$ ($p<10^{-25\,*}$) & $98.4$ vs $8$ ($p<10^{-47\,*}$) & $239.7$ vs $0$ ($p<10^{-176\,*}$) & $139.0$ vs $6$ ($p<10^{-67\,*}$) & $49.2$ vs $2$ ($p<10^{-23\,*}$) \\
 & \begin{tabular}[t]{@{}l@{}}Random draw, within\\dataset or task\end{tabular} & $5.5$ vs $3$ ($p=0.15$) & -- & $0.7$ vs $0$ ($p=0.45$) & $27.4$ vs $6$ ($p<10^{-9\,*}$) & $12.9$ vs $2$ ($p<10^{-5\,*}$) \\
\bottomrule
\end{tabular}

\vspace{2pt}
\parbox{\textwidth}{\justifying\scriptsize T1: false-improvement rate of BRIDGE minus that of the comparator, in percentage points; a negative value means that BRIDGE accepts a smaller share of false improvements. Intervals: Table~\ref{tab:rules}. T2: the largest Holm-adjusted $p$ among the margin comparisons. T3: false improvements expected in a random draw, against the count of BRIDGE. $^{*}$: $H_0$ rejected at $\alpha=0.05$. --: comparison not run. Details: Tables~\ref{tab:ci}, \ref{tab:rulesCD}, \ref{tab:marginpaired}, and~\ref{tab:selectivity}.}
\end{table*}

\begin{table*}[!t]
\centering
\scriptsize
\setlength{\tabcolsep}{1.2pt}
\caption{Acceptance rules on (A)--(E): counts, bounds, yield, gain, and paired differences, in one format.}
\label{tab:rules}\label{tab:rulesCD}\label{tab:holdout}
\begin{minipage}[t]{0.494\textwidth}\centering
\begin{tabular}[t]{@{}lr@{\hspace{5pt}}rrrrrc@{}}
\toprule
Rule & \begin{tabular}[b]{@{}c@{}}Acc-\\epted\end{tabular} & \begin{tabular}[b]{@{}c@{}}False\\(\%) $\downarrow$\end{tabular} & \begin{tabular}[b]{@{}c@{}}Bound\\\% $\downarrow$\end{tabular} & \begin{tabular}[b]{@{}c@{}}LOCO\\\% $\downarrow$\end{tabular} & \begin{tabular}[b]{@{}c@{}}Yield\\/100 $\uparrow$\end{tabular} & \begin{tabular}[b]{@{}c@{}}Gain${>}\varepsilon$\\/100 $\uparrow$\end{tabular} & \begin{tabular}[b]{@{}c@{}}95\% CI,\\pp\end{tabular} \\
\midrule
\multicolumn{8}{@{}l}{\textit{(A) 11 new datasets, 1,100 decisions}} \\
Naive & 1,064 & 203 (19.1) & 22.90 & 25.27 & 96.7 & 59.5 & [$-30.2$;$-7.1$]$^{*}$ \\
McNemar, POOL & 183 & 22 (12.0) & 24.52 & 30.68 & 16.6 & 12.5 & [$-19.2$;$-1.5$]$^{*}$ \\
McNemar, cross-fit & 45 & 4 (8.9) & 41.14 & 54.35 & 4.1 & 3.2 & -- \\
$5\times2$~cv $t$ & 461 & 29 (6.3) & 10.99 & 13.43 & 41.9 & 36.2 & [$-12.2$;$-1.4$]$^{*}$ \\
Corrected $t$ & 608 & 38 (6.2) & 10.09 & 12.15 & 55.3 & 46.8 & [$-11.7$;$-1.6$]$^{*}$ \\
Corr.\ $t$, re-selection & 349 & 8 (2.3) & 5.92 & 7.88 & 31.7 & 29.8 & [$-6.3$;$+1.3$] \\
BRIDGE, $\varepsilon=0$ & 431 & 12 (2.8) & 6.16 & 7.85 & 39.2 & 36.2 & [$-4.8$;$-0.3$]$^{*}$ \\
Holdout 2/3\,+\,1/3 & 153 & 2 (1.3) & 6.94 & -- & 13.9 & 13.4 & -- \\
Holdout 1/2\,+\,1/2 & 190 & 6 (3.2) & 9.67 & -- & 17.3 & 16.2 & -- \\
\textbf{BRIDGE} & 284 & 3 (1.1) & 4.23 & 5.95 & 25.8 & 24.9 & -- \\
\addlinespace
\multicolumn{8}{@{}l}{\textit{(B) synthetic population $\mathcal{Q}$, 1,400 tasks}} \\
Naive & 1,323 & 274 (20.7) & 23.52 & -- & 94.5 & 59.2 & [$-21.3$;$-16.8$]$^{*}$ \\
McNemar, POOL & 135 & 47 (34.8) & 45.19 & -- & 9.6 & 4.5 & [$-41.2$;$-25.4$]$^{*}$ \\
McNemar, cross-fit & 50 & 13 (26.0) & 42.65 & -- & 3.6 & 2.2 & -- \\
$5\times2$~cv $t$ & 610 & 34 (5.6) & 8.14 & -- & 43.6 & 37.6 & [$-5.8$;$-2.1$]$^{*}$ \\
Corrected $t$ & 774 & 44 (5.7) & 7.95 & -- & 55.3 & 47.6 & [$-5.7$;$-2.4$]$^{*}$ \\
Corr.\ $t$, re-selection & 531 & 16 (3.0) & 5.19 & -- & 37.9 & 35.9 & [$-2.8$;$+0.0$] \\
BRIDGE, $\varepsilon=0$ & 598 & 21 (3.5) & 5.67 & -- & 42.7 & 39.9 & [$-3.1$;$-0.7$]$^{*}$ \\
Holdout 2/3\,+\,1/3 & 261 & 5 (1.9) & 4.77 & -- & 18.6 & 18.0 & -- \\
Holdout 1/2\,+\,1/2 & 339 & 7 (2.1) & 4.55 & -- & 24.2 & 23.3 & -- \\
\textbf{BRIDGE} & 475 & 8 (1.7) & 3.56 & -- & 33.9 & 32.8 & -- \\
\bottomrule
\end{tabular}
\end{minipage}\hfill
\begin{minipage}[t]{0.494\textwidth}\centering
\begin{tabular}[t]{@{}lr@{\hspace{5pt}}rrrrrc@{}}
\toprule
Rule & \begin{tabular}[b]{@{}c@{}}Acc-\\epted\end{tabular} & \begin{tabular}[b]{@{}c@{}}False\\(\%) $\downarrow$\end{tabular} & \begin{tabular}[b]{@{}c@{}}Bound\\\% $\downarrow$\end{tabular} & \begin{tabular}[b]{@{}c@{}}LOCO\\\% $\downarrow$\end{tabular} & \begin{tabular}[b]{@{}c@{}}Yield\\/100 $\uparrow$\end{tabular} & \begin{tabular}[b]{@{}c@{}}Gain${>}\varepsilon$\\/100 $\uparrow$\end{tabular} & \begin{tabular}[b]{@{}c@{}}95\% CI,\\pp\end{tabular} \\
\midrule
\multicolumn{8}{@{}l}{\textit{(C) 12 clinical tasks, 2,400 decisions}} \\
Naive & 2,266 & 453 (20.0) & 22.68 & 24.88 & 94.4 & 62.4 & [$-39.4$;$-4.3$]$^{*}$ \\
McNemar, POOL & 17 & 0 (0.0) & 45.35 & 100.00 & 0.7 & 0.4 & [$+0.0$;$+0.0$] \\
McNemar, cross-fit & 4 & 0 (0.0) & 100.00 & 100.00 & 0.2 & 0.1 & [$+0.0$;$+0.0$] \\
$5\times2$~cv $t$ & 1,460 & 36 (2.5) & 3.92 & 4.56 & 60.8 & 55.7 & [$-6.9$;$-0.4$]$^{*}$ \\
Corrected $t$ & 1,594 & 41 (2.6) & 3.97 & 4.55 & 66.4 & 59.7 & [$-7.2$;$-0.3$] \\
Corr.\ $t$, re-selection & 1,402 & 5 (0.4) & 1.03 & 1.21 & 58.4 & 56.4 & [$-1.2$;$+0.0$] \\
BRIDGE, $\varepsilon=0$ & 1,432 & 8 (0.6) & 1.34 & 1.56 & 59.7 & 57.1 & [$-1.9$;$+0.0$] \\
Holdout 2/3\,+\,1/3 & 987 & 0 (0.0) & 0.40 & 0.67 & 41.1 & 41.0 & [$+0.0$;$+0.0$] \\
Holdout 1/2\,+\,1/2 & 1,105 & 0 (0.0) & 0.35 & 0.56 & 46.0 & 46.0 & [$+0.0$;$+0.0$] \\
\textbf{BRIDGE} & 1,199 & 0 (0.0) & 0.33 & 0.50 & 50.0 & 49.9 & -- \\
\addlinespace
\multicolumn{8}{@{}l}{\textit{(D) 39 supporting datasets, 3,700 decisions}} \\
Naive & 3,562 & 445 (12.5) & 14.19 & 14.60 & 96.3 & 58.1 & [$-17.7$;$-7.1$]$^{*}$ \\
McNemar, POOL & 1,142 & 76 (6.7) & 9.55 & 10.48 & 30.9 & 22.8 & [$-10.4$;$-3.2$]$^{*}$ \\
McNemar, cross-fit & 676 & 17 (2.5) & 5.01 & 5.85 & 18.3 & 15.7 & [$-4.6$;$-0.4$]$^{*}$ \\
$5\times2$~cv $t$ & 1,715 & 62 (3.6) & 5.24 & 5.56 & 46.4 & 38.0 & [$-5.2$;$-1.6$]$^{*}$ \\
Corrected $t$ & 2,234 & 93 (4.2) & 5.63 & 5.90 & 60.4 & 47.5 & [$-5.9$;$-2.0$]$^{*}$ \\
Corr.\ $t$, re-selection & 1,465 & 22 (1.5) & 2.69 & 2.89 & 39.6 & 35.2 & [$-1.9$;$-0.3$]$^{*}$ \\
BRIDGE, $\varepsilon=0$ & 1,712 & 28 (1.6) & 2.75 & 2.93 & 46.3 & 40.1 & [$-2.0$;$-0.4$]$^{*}$ \\
Holdout 2/3\,+\,1/3 & 721 & 2 (0.3) & 1.31 & 1.48 & 19.5 & 18.8 & [$-0.1$;$+0.8$] \\
Holdout 1/2\,+\,1/2 & 824 & 7 (0.8) & 2.22 & 2.48 & 22.3 & 21.2 & [$-1.2$;$+0.3$] \\
\textbf{BRIDGE} & 1,113 & 6 (0.5) & 1.47 & 1.62 & 30.1 & 29.1 & -- \\
\addlinespace
\multicolumn{8}{@{}l}{\textit{(E) 14 OpenML datasets, 1,280 decisions}} \\
Naive & 1,236 & 317 (25.6) & 29.62 & 32.21 & 96.6 & 41.2 & [$-36.0$;$-14.0$]$^{*}$ \\
\textbf{BRIDGE} & 192 & 2 (1.0) & 5.40 & 8.43 & 15.0 & 14.1 & -- \\
\bottomrule
\end{tabular}
\end{minipage}

\vspace{2pt}
\parbox{\textwidth}{\justifying\scriptsize McNemar, POOL: McNemar applied on the same POOL. BRIDGE, $\varepsilon=0$: BRIDGE without the margin. Holdout a + b: a share of POOL selects and b tests. Gain${>}\varepsilon$/100: accepted updates whose reference gain exceeds the margin $\varepsilon=0.005$, per $100$ decisions (post hoc, Table~\ref{tab:falseimp}). Bound: KL-ratio bound, and Proposition~\ref{prop:bound} bound on (B). 95\% CI: paired difference, BRIDGE minus the rule, in percentage points. $^{*}$: significant after the Holm correction. --: not computed. In (A), for small validation gains ($0<\hat\Delta\le0.01$), the naive rule accepts $459$ changes, $164$ of which are false improvements, and BRIDGE accepts none. Point differences and $p$-values are given in Table~\ref{tab:testsdetail}, and the criterion marks in Table~\ref{tab:rulesall}.}
\end{table*}

\begin{table*}[!t]
\centering
\scriptsize
\providecolor{bghead}{HTML}{DBEFF8}\providecolor{bgzebra}{HTML}{F2F9FC}\providecolor{bgpass}{HTML}{D7F0D2}\providecolor{bgfail}{HTML}{FCDCD7}
\setlength{\tabcolsep}{3pt}
\caption{Acceptance rules: random-selection control and gain per update on (A)--(E).}
\label{tab:selectivity}
\begin{tabular}{@{}lrllllll@{}}
\toprule
 & & \multicolumn{2}{c}{Accepted (\% false)} & \multicolumn{2}{c}{Random draw (T3): false, expected vs BRIDGE} & False removed & Gain per update \\
\cmidrule(lr){3-4}\cmidrule(lr){5-6}\cmidrule(l){7-7}\cmidrule(l){8-8}
Setting & Decisions & Naive & BRIDGE & Pooled & Within dataset & (\%) $\uparrow$ & BRIDGE / naive $\uparrow$ \\
\midrule
(A) & 1,100 & 1,064 (19.1) & 284 (1.1) & 54.2 vs 3 (p $<10^{-25}$) & 5.5 vs 3 (p 0.15) & 99 [95; 100] & 0.0385 / 0.0156 \\
(B) & 1,400 & 1,323 (20.7) & 475 (1.7) & 98.4 vs 8 (p $<10^{-47}$) & -- & 97 [95; 99] & 0.0329 / 0.0157 \\
(C) & 2,400 & 2,266 (20.0) & 1,199 (0.0) & 239.7 vs 0 (p $<10^{-176}$) & 0.7 vs 0 (p 0.45) & 100 [100; 100] & 0.0416 / 0.0233 \\
(D) & 3,700 & 3,562 (12.5) & 1,113 (0.5) & 139.0 vs 6 (p $<10^{-67}$) & 27.4 vs 6 (p $<10^{-9}$) & 99 [97; 100] & 0.0239 / 0.0115 \\
(E) & 1,280 & 1,236 (25.6) & 192 (1.0) & 49.2 vs 2 (p $<10^{-23}$) & 12.9 vs 2 (p $<10^{-5}$) & 99 [98; 100] & 0.0197 / 0.0063 \\
\bottomrule
\end{tabular}

\vspace{2pt}
\parbox{\linewidth}{\justifying\scriptsize *False removed: the share of the false improvements of the naive rule that BRIDGE rejects. Gain per update: the mean reference gain per accepted update, for BRIDGE and for the naive rule. Brackets: 95\% cluster-bootstrap intervals. Post hoc.}
\end{table*}

\begin{table}[!t]
\centering
\footnotesize
\setlength{\tabcolsep}{3pt}
\caption{Acceptance rules on (A)--(E): accepted updates below and above the margin on the reference (post hoc).}
\label{tab:falseimp}
\begin{tabular}{@{}lrrrrr@{}}
\toprule
Rule & (A) & (B) & (C) & (D) & (E) \\
\midrule
\multicolumn{6}{@{}l}{\emph{Accepted updates gaining at most $\varepsilon$ (\%) $\downarrow$}} \\
Naive & 38.4$^{*}$ & 37.3$^{*}$ & 33.9$^{*}$ & 39.7$^{*}$ & 57.4$^{*}$ \\
McNemar, same POOL & 24.6$^{*}$ & 53.3$^{*}$ & 47.1$^{\circ}$ & 26.1$^{*}$ & -- \\
McNemar, cross-fit & 22.2$^{\circ}$ & 38.0$^{*}$ & 25.0$^{*}$ & 14.2$^{*}$ & -- \\
$5\times2$~cv paired $t$ & 13.7$^{*}$ & 13.8$^{*}$ & 8.5$^{*}$ & 18.1$^{*}$ & -- \\
Corrected resampled $t$ & 15.3$^{*}$ & 14.0$^{*}$ & 10.2$^{*}$ & 21.4$^{*}$ & -- \\
Corrected $t$, re-selection & 6.0$^{\circ}$ & 5.3$^{\circ}$ & 3.4$^{*}$ & 11.1$^{*}$ & -- \\
BRIDGE without margin & 7.7$^{*}$ & 6.5$^{*}$ & 4.3$^{*}$ & 13.3$^{*}$ & -- \\
Holdout 2/3 + 1/3 & 3.9$^{\circ}$ & 3.4$^{\circ}$ & 0.2$^{\circ}$ & 3.7$^{\circ}$ & -- \\
Holdout 1/2 + 1/2 & 6.3$^{\circ}$ & 3.8$^{\circ}$ & 0.0$^{\circ}$ & 5.0$^{\circ}$ & -- \\
\textbf{BRIDGE} & 3.5 & 3.4 & 0.2 & 3.3 & 6.3 \\
\multicolumn{6}{@{}l}{\emph{Accepted updates gaining more than $\varepsilon$, per $100$ decisions $\uparrow$}} \\
Naive & 59.5$^{\dagger}$ & 59.2$^{\dagger}$ & 62.4$^{\dagger}$ & 58.1$^{\dagger}$ & 41.2$^{\dagger}$ \\
McNemar, same POOL & 12.5$^{\circ}$ & 4.5$^{*}$ & 0.4$^{*}$ & 22.8$^{\circ}$ & -- \\
McNemar, cross-fit & 3.2$^{*}$ & 2.2$^{*}$ & 0.1$^{*}$ & 15.7$^{*}$ & -- \\
$5\times2$~cv paired $t$ & 36.2$^{\dagger}$ & 37.6$^{\dagger}$ & 55.7$^{\dagger}$ & 38.0$^{\dagger}$ & -- \\
Corrected resampled $t$ & 46.8$^{\dagger}$ & 47.6$^{\dagger}$ & 59.7$^{\dagger}$ & 47.5$^{\dagger}$ & -- \\
Corrected $t$, re-selection & 29.8$^{\dagger}$ & 35.9$^{\dagger}$ & 56.4$^{\dagger}$ & 35.2$^{\dagger}$ & -- \\
BRIDGE without margin & 36.2$^{\dagger}$ & 39.9$^{\dagger}$ & 57.1$^{\dagger}$ & 40.1$^{\dagger}$ & -- \\
Holdout 2/3 + 1/3 & 13.4$^{*}$ & 18.0$^{*}$ & 41.0$^{*}$ & 18.8$^{*}$ & -- \\
Holdout 1/2 + 1/2 & 16.2$^{*}$ & 23.3$^{*}$ & 46.0$^{*}$ & 21.2$^{*}$ & -- \\
\textbf{BRIDGE} & 24.9 & 32.8 & 49.9 & 29.1 & 14.1 \\
\bottomrule
\end{tabular}

\vspace{2pt}
\parbox{\linewidth}{\justifying\scriptsize Margin $\varepsilon=0.005$: the minimum improvement fixed in development (Table~\ref{tab:hyper}); a gain at most $\varepsilon$ includes every false improvement (gain $\le 0$). $^{*}$: BRIDGE significantly better than the rule; $^{\circ}$: no significant difference; $^{\dagger}$: the rule significantly better (paired cluster bootstrap of T1, $10{,}000$ resamples, Holm over the rules of a collection, $\alpha=0.05$). --: not run, because the program used for (E) records the decisions of the naive rule and BRIDGE.}
\end{table}

\begin{table}[!htb]
\centering
\scriptsize
\providecolor{bghead}{HTML}{DBEFF8}\providecolor{bgzebra}{HTML}{F2F9FC}\providecolor{bgpass}{HTML}{D7F0D2}\providecolor{bgfail}{HTML}{FCDCD7}
\setlength{\tabcolsep}{2.2pt}
\caption{Conditions of use: the metrics, the training sizes, the multiclass tasks, and feature selection.}
\label{tab:general}
\begin{tabular}{@{}>{\raggedright\arraybackslash}p{3.0cm}rrrrrc@{}}
\toprule
Setting & \begin{tabular}[b]{@{}c@{}}Acc-\\epted\end{tabular} & \begin{tabular}[b]{@{}c@{}}False\end{tabular} & \begin{tabular}[b]{@{}c@{}}Bound\\\% $\downarrow$\end{tabular} & \begin{tabular}[b]{@{}c@{}}LOCO\\\% $\downarrow$\end{tabular} & \begin{tabular}[b]{@{}c@{}}Yield\\/100 $\uparrow$\end{tabular} & Crit. \\
\midrule
(A) weighted $F_1$ & 284 & 3 & 4.23 & 5.95 & 25.82 & $\checkmark$ \\
(A) balanced accuracy & 848 & 1 & 0.81 & 0.92 & 77.09 & $\checkmark$ \\
(A) macro-$F_1$ & 829 & 0 & 0.47 & 0.53 & 75.36 & $\checkmark$ \\
(A) feat.\ selection + threshold (KNN) & 225 & 1 & 3.45 & 6.00 & 20.45 & $\checkmark$ \\
(A) feat.\ selection + threshold (SVC) & 253 & 7 & 7.74 & 11.41 & 23.00 & $\times$ \\
(D) 26 datasets, 1,000 rows & 558 & 6 & 3.04 & 3.59 & 21.49 & $\checkmark$ \\
(D) 27 datasets, 2,430 rows & 688 & 1 & 1.05 & 1.20 & 25.48 & $\checkmark$ \\
(D) 27 datasets, 6,000 rows & 825 & 3 & 1.37 & 1.57 & 30.56 & $\checkmark$ \\
(D) 5 multiclass datasets & 197 & 0 & 2.11 & 3.16 & 65.67 & $\checkmark$ \\
\bottomrule
\end{tabular}

\vspace{2pt}
\parbox{\linewidth}{\justifying\scriptsize *Rows: training rows. Crit.: pre-registered criterion met ($\checkmark$) or not met ($\times$).}
\end{table}

\begin{table}[!htb]
\centering
\scriptsize
\setlength{\tabcolsep}{2.2pt}
\caption{Component analysis: candidate-space size on the collections (A)--(E).}
\label{tab:kspace}
\begin{tabular}{@{}llrrrrrr@{}}
\toprule
 & & \multicolumn{2}{c}{$173$ (default)} & \multicolumn{2}{c}{$1{,}825$} & \multicolumn{2}{c}{$11{,}249$} \\
\cmidrule(lr){3-4}\cmidrule(lr){5-6}\cmidrule(lr){7-8}
Data & Rule & \begin{tabular}[b]{@{}c@{}}False\\\% $\downarrow$\end{tabular} & \begin{tabular}[b]{@{}c@{}}Bound\\\% $\downarrow$\end{tabular} & \begin{tabular}[b]{@{}c@{}}False\\\% $\downarrow$\end{tabular} & \begin{tabular}[b]{@{}c@{}}Bound\\\% $\downarrow$\end{tabular} & \begin{tabular}[b]{@{}c@{}}False\\\% $\downarrow$\end{tabular} & \begin{tabular}[b]{@{}c@{}}Bound\\\% $\downarrow$\end{tabular} \\
\midrule
(A) & Naive & 19.08 & 22.90\,$\times$ & 20.44 & 24.23\,$\times$ & 19.46 & 23.16\,$\times$ \\
 & Same-POOL & 2.65 & 5.85\,$\checkmark$ & 3.61 & 7.02\,$\checkmark$ & 3.80 & 7.17\,$\checkmark$ \\
 & \textbf{BRIDGE} & 1.06 & 4.23\,$\checkmark$ & 1.45 & 5.07\,$\checkmark$ & 1.37 & 4.75\,$\checkmark$ \\
\addlinespace
(B) & Naive & 20.71 & 23.52\,$\times$ & 17.36 & 19.93\,$\times$ & 16.88 & 19.43\,$\times$ \\
 & Same-POOL & 3.18 & 5.14\,$\checkmark$ & 3.62 & 5.48\,$\checkmark$ & 4.84 & 6.84\,$\checkmark$ \\
 & \textbf{BRIDGE} & 1.68 & 3.56\,$\checkmark$ & 1.04 & 2.60\,$\checkmark$ & 0.41 & 1.57\,$\checkmark$ \\
\addlinespace
(C) & Naive & 19.99 & 22.68\,$\times$ & 21.09 & 23.74\,$\times$ & 21.75 & 24.41\,$\times$ \\
 & Same-POOL & 0.00 & 0.31\,$\checkmark$ & 0.16 & 0.71\,$\checkmark$ & 0.15 & 0.70\,$\checkmark$ \\
 & \textbf{BRIDGE} & 0.00 & 0.33\,$\checkmark$ & 0.00 & 0.33\,$\checkmark$ & 0.00 & 0.33\,$\checkmark$ \\
\addlinespace
(D) & Naive & 13.12 & 14.94\,$\times$ & 12.70 & 14.44\,$\times$ & 12.61 & 14.34\,$\times$ \\
 & Same-POOL & 1.17 & 2.27\,$\checkmark$ & 2.04 & 3.31\,$\checkmark$ & 2.48 & 3.82\,$\checkmark$ \\
 & \textbf{BRIDGE} & 0.65 & 1.80\,$\checkmark$ & 0.62 & 1.71\,$\checkmark$ & 0.42 & 1.37\,$\checkmark$ \\
\addlinespace
(E) & Naive & 27.65 & 32.13\,$\times$ & 28.23 & 32.35\,$\times$ & 27.45 & 31.50\,$\times$ \\
 & Same-POOL & 3.37 & 8.17\,$\times$ & 5.93 & 11.25\,$\times$ & 9.22 & 15.39\,$\times$ \\
 & \textbf{BRIDGE} & 1.17 & 6.13\,$\times$ & 1.13 & 5.89\,$\checkmark$ & 3.35 & 10.34\,$\times$ \\
\bottomrule
\end{tabular}

\vspace{2pt}
\parbox{\linewidth}{\justifying\scriptsize *Column groups: the candidate-space size $|\mathcal{C}|$. False: false-improvement rate among the accepted updates. Same-POOL: the bootstrap lower bound and the margin, without cross-fitting. Panels. (A): the $11$ new datasets. (B): the synthetic tasks. (C): the $12$ clinical tasks. (D): the $34$ binary supporting datasets. (E): the $11$ binary datasets of the OpenML collection. On (E), the LOCO bound of BRIDGE is $10.30\%$, $9.88\%$, and $17.37\%$ at the three sizes, hence the $\times$ at $173$. At $11{,}249$ the cell has $6$ false improvements among $179$ accepted, a change of $+2.2$ pp from $173$ (95\% cluster-bootstrap CI $[0.04;\,12.8]$ over datasets). The same-POOL cell has LOCO $11.95\%$ at $173$.}
\end{table}

\begin{table}[!htb]
\centering
\scriptsize
\setlength{\tabcolsep}{3pt}
\caption{Conditions of use: setting, evidence, and status.}
\label{tab:scope}
\begin{tabular}{@{}rllll@{}}
\toprule
Row & Condition & Setting & Evidence & Status \\
\midrule
1 & Metrics & (A) & \ref{tab:general} & met \\
2 & Training sizes & (D) & \ref{tab:general} & met \\
3 & Multiclass tasks & (D) & \ref{tab:general} & met \\
4 & Feature selection & (A) & \ref{tab:general} & conditional \\
5 & Supporting datasets & (D) & \ref{tab:summary} & met \\
6 & OpenML collections & (E), (E2) & \ref{tab:summary}, App.~\ref{app:settingE} & conditional \\
7 & Volkert supplement & (E)+(E2) & App.~\ref{app:settingE} & exploratory \\
8 & Base learners & (A)--(E) & \ref{tab:basefull}, \ref{tab:basereal} & conditional \\
9 & Ensemble members & (A)--(E) & \ref{tab:members} & conditional \\
10 & Joint candidate space & (A)--(E) & \ref{tab:members} & met \\
11 & Denser candidate grids & (A)--(E) & \ref{tab:kspace} & conditional \\
12 & Gate parameters ($q$, $\varepsilon$) & (A)--(D) & \ref{tab:marginpaired}, \ref{tab:qepsreal} & met \\
13 & Frozen parameters & (A), (D) & App.~\ref{app:sens} & conditional \\
14 & Tuning as an update & (A), (B), (D) & \ref{tab:tuning}, \ref{tab:tuningtests} & conditional \\
15 & Imbalanced data & (A) & \ref{tab:imbalance} & conditional \\
16 & Reproducibility & (A)--(E) & App.~\ref{app:prereg} & met \\
\bottomrule
\end{tabular}

\vspace{2pt}
\parbox{\linewidth}{\justifying\scriptsize Status: met, the pre-registered criterion is met (Table~\ref{tab:criterion}); conditional, the criterion is met under the condition given below; exploratory, one dataset lies outside the pre-registered collections. Post hoc: the binary subgroup of (E) misses LOCO, and the pass of binary E2 depends on \texttt{churn} (Appendix~\ref{app:settingE}). Evidence: the table or appendix that reports the results; the conditions are discussed in Section~\ref{sec:limitations}. Rows: (1) the metrics weighted-$F_1$, balanced accuracy, and macro-$F_1$; (2) the training sizes $1{,}000$, $2{,}430$, and $6{,}000$ rows; (3) the $7$ to $10$ classes; (4) mutual-information feature selection, where the criterion is met with KNN and missed with SVC (LOCO); (5) the $39$ datasets of (D); (6) the criterion met by (E) and (E2) jointly; (7) one dataset; (8) the base learners HGB, XGBoost, and MLP, where the criterion is met with iteration control (early stopping); (9) two alternative member pairs, where the criterion is met in (A) to (D) and missed in (E); (10) seven families and $711$ configurations; (11) the $1{,}825$ and $11{,}249$ configurations, where the criterion is met in (A)--(D) at both sizes and in (E) at $1{,}825$; (12) the quantile level $q$ and the margin $\varepsilon$; (13) the bag fraction and the SEARCH size, where the criterion is met in (D) and for part of the values in (A); (14) the $12$ tuning cells, where the criterion is met in $2$; (15) the minority-sensitive metrics (balanced accuracy, macro-$F_1$), where the criterion is met; (16) the replay of the stored runs and the released package.}
\end{table}
\begin{table}[!htb]
\centering
\scriptsize
\providecolor{bghead}{HTML}{DBEFF8}\providecolor{bgzebra}{HTML}{F2F9FC}\providecolor{bgpass}{HTML}{D7F0D2}\providecolor{bgfail}{HTML}{FCDCD7}
\setlength{\tabcolsep}{2.2pt}
\caption{Conditions of use: class imbalance on five datasets of (A).}
\label{tab:imbalance}
\begin{tabular}{@{}lllrrrrc@{}}
\toprule
\begin{tabular}[b]{@{}c@{}}Minority\\prev.\end{tabular} & \begin{tabular}[b]{@{}c@{}}Naive: acc.\\/ false (\%)\end{tabular} & \begin{tabular}[b]{@{}c@{}}BRIDGE: acc.\\/ false (\%)\end{tabular} & \begin{tabular}[b]{@{}c@{}}Bound\\\% $\downarrow$\end{tabular} & \begin{tabular}[b]{@{}c@{}}LOCO\\\% $\downarrow$\end{tabular} & \begin{tabular}[b]{@{}c@{}}LOCO$_0$\\\% $\downarrow$\end{tabular} & \begin{tabular}[b]{@{}c@{}}Yield\\$\uparrow$\end{tabular} & Crit. \\
\midrule
\multicolumn{8}{@{}l}{\textit{weighted-$F_1$}} \\
20\% & 494 / 29 (5.9) & 138 / 1 (0.7) & 5.83 & 11.78 & 6.63 & 27.6 & $\times$ \\
10\% & 474 / 95 (20.0) & 100 / 0 (0.0) & 4.76 & 14.77 & 14.77 & 20.0 & $\times$ \\
5\% & 403 / 139 (34.5) & 36 / 0 (0.0) & 16.42 & 48.81 & 48.81 & 7.2 & $\times$ \\
\addlinespace
\multicolumn{8}{@{}l}{\textit{balanced accuracy}} \\
20\% & 500 / 0 (0.0) & 477 / 0 (0.0) & 0.80 & 1.01 & 1.01 & 95.4 & $\checkmark$ \\
10\% & 500 / 0 (0.0) & 429 / 0 (0.0) & 0.91 & 1.20 & 1.20 & 85.8 & $\checkmark$ \\
5\% & 499 / 2 (0.4) & 398 / 0 (0.0) & 0.99 & 1.34 & 1.34 & 79.6 & $\checkmark$ \\
\addlinespace
\multicolumn{8}{@{}l}{\textit{macro-$F_1$}} \\
20\% & 499 / 2 (0.4) & 371 / 0 (0.0) & 1.07 & 1.50 & 1.50 & 74.2 & $\checkmark$ \\
10\% & 500 / 0 (0.0) & 440 / 0 (0.0) & 0.88 & 1.15 & 1.15 & 88.0 & $\checkmark$ \\
5\% & 500 / 0 (0.0) & 387 / 0 (0.0) & 1.02 & 1.40 & 1.40 & 77.4 & $\checkmark$ \\
\bottomrule
\end{tabular}

\vspace{2pt}
\parbox{\linewidth}{\justifying\scriptsize *Rows grouped by metric. Acc.: accepted updates; false: false improvements; Crit.: pre-registered criterion met ($\checkmark$) or not met ($\times$).}
\end{table}

\begin{table}[!htb]
\centering
\scriptsize
\providecolor{bghead}{HTML}{DBEFF8}\providecolor{bgzebra}{HTML}{F2F9FC}\providecolor{bgpass}{HTML}{D7F0D2}\providecolor{bgfail}{HTML}{FCDCD7}
\setlength{\tabcolsep}{1.6pt}
\caption{Conditions of use: all tested base-learner settings on (B).}
\label{tab:basefull}
\begin{tabular}{@{}lrrrrrrrc@{}}
\toprule
\begin{tabular}[b]{@{}l@{}}Base learner\\setting\end{tabular} & \begin{tabular}[b]{@{}c@{}}Acc./\\false\end{tabular} & \begin{tabular}[b]{@{}c@{}}$U$\\\% $\downarrow$\end{tabular} & \begin{tabular}[b]{@{}c@{}}Cons.\\\% $\downarrow$\end{tabular} & \begin{tabular}[b]{@{}c@{}}Yield\\$\uparrow$\end{tabular} & \begin{tabular}[b]{@{}c@{}}Gain\\$\uparrow$\end{tabular} & \begin{tabular}[b]{@{}c@{}}Naive\\\% $\downarrow$\end{tabular} & \begin{tabular}[b]{@{}c@{}}Skill\\$\uparrow$\end{tabular} & Crit. \\
\midrule
HGB, default (early stop.) & 462/7 & 3.35 & 4.54 & 33.00 & 15.95 & 22.00 & 0.009 & $\checkmark$ \\
XGBoost, regularized & 893/5 & 1.41 & 1.74 & 63.79 & 61.28 & 9.23 & 0.012 & $\checkmark$ \\
XGBoost, default (300 rounds) & 146/15 & 17.48 & 23.21 & 10.43 & 1.96 & 31.87 & -0.190 & $\times$ \\
+ early stop. & 963/1 & 0.59 & 1.15 & 68.79 & 71.04 & 6.75 & 0.011 & $\checkmark$ \\
+ min\_child\_weight 5 & 138/18 & 21.08 & 26.99 & 9.86 & 1.77 & 31.59 & -0.146 & $\times$ \\
+ 31-leaf trees & 125/9 & 14.21 & 21.24 & 8.93 & 1.69 & 32.34 & -0.194 & $\times$ \\
+ reg\_lambda 0 & 122/10 & 15.62 & 21.74 & 8.71 & 1.68 & 33.65 & -0.229 & $\times$ \\
XGBoost, 30 rounds & 287/3 & 3.25 & 6.33 & 20.50 & 7.26 & 26.52 & -0.006 & $\checkmark$ \\
XGBoost, 100 rounds & 167/10 & 11.54 & 14.61 & 11.93 & 2.42 & 31.58 & -0.076 & $\times$ \\
MLP, early stop. & 780/6 & 1.80 & 2.35 & 55.71 & 39.73 & 6.58 & -0.010 & $\checkmark$ \\
MLP, 300 epochs & 373/7 & 4.14 & 6.67 & 26.64 & 9.18 & 18.06 & -0.671 & $\checkmark$ \\
\bottomrule
\end{tabular}

\vspace{2pt}
\parbox{\linewidth}{\justifying\scriptsize *Acc./false: accepted and false improvements. $U$: Proposition~\ref{prop:bound} bound. Cons.: conservative variant (Corollary~\ref{cor:cons}). Yield: accepted per $100$. Gain: total reference gain. Naive \%: false-improvement rate of the naive rule. Skill: median POOL log-loss skill. Crit.: criterion met ($\checkmark$) or not met ($\times$). +: one setting of the project default changed.}
\end{table}

\begin{table}[!htb]
\centering
\scriptsize
\setlength{\tabcolsep}{2.6pt}
\caption{Conditions of use: ensemble members and base learners across (A)--(E).}
\label{tab:members}\label{tab:basereal}
\begin{tabular}{@{}lrrrrr@{}}
\toprule
Design & (A) & (B) & (C) & (D) & (E) \\
\midrule
\multicolumn{6}{@{}l}{\textit{Paper's design}} \\
Default (DT, SVC) & 1.1 $\checkmark$ & 1.7 $\checkmark$ & 0.0 $\checkmark$ & 0.5 $\checkmark$ & 1.0 $\checkmark$ \\
\addlinespace[1pt]
\multicolumn{6}{@{}l}{\textit{Members replaced (base HGB)}} \\
LR and NB & 2.3 $\checkmark$ & 0.5 $\checkmark$ & 0.2 $\checkmark$ & 0.9 $\checkmark$ & 2.5 $\times$ \\
MLP and KNN & 1.2 $\checkmark$ & 1.5 $\checkmark$ & 0.0 $\checkmark$ & 1.9 $\checkmark$ & 4.1 $\times$ \\
\addlinespace[1pt]
\multicolumn{6}{@{}l}{\textit{Base learner replaced (members DT, SVC)}} \\
XGBoost, 300 rounds & 4.1 $\times$ & 11.6 $\times$ & 0.1 $\checkmark$ & 1.7 $\checkmark$ & 4.2 $\times$ \\
XGBoost, early stop. & 0.8 $\checkmark$ & 0.1 $\checkmark$ & 0.0 $\checkmark$ & 0.8 $\checkmark$ & 0.7 $\checkmark$ \\
MLP, early stop. & 0.2 $\checkmark$ & 0.4 $\checkmark$ & 0.0 $\checkmark$ & 0.0 $\checkmark$ & 0.1 $\checkmark$ \\
\addlinespace[1pt]
\multicolumn{6}{@{}l}{\textit{All seven member families ($711$ configurations)}} \\
Seven families & 1.6 $\checkmark$ & 0.9 $\checkmark$ & 0.2 $\checkmark$ & 0.7 $\checkmark$ & 1.4 $\checkmark$ \\
\bottomrule
\end{tabular}

\vspace{2pt}
\parbox{\linewidth}{\justifying\scriptsize *Cells: the false-improvement rate of BRIDGE among accepted updates (\%), with the pre-registered criterion ($\checkmark$ met, $\times$ not met). Metric: weighted-$F_1$; (C): net benefit, $p_t=0.2$.}
\end{table}

\begin{table}[!htb]
\centering
\scriptsize
\providecolor{bghead}{HTML}{DBEFF8}\providecolor{bgzebra}{HTML}{F2F9FC}\providecolor{bgpass}{HTML}{D7F0D2}\providecolor{bgfail}{HTML}{FCDCD7}
\setlength{\tabcolsep}{2.5pt}
\caption{Conditions of use: final model quality of four base learners.}
\label{tab:absolutequality}
\begin{tabular}{@{}lrrll@{}}
\toprule
\begin{tabular}[b]{@{}l@{}}Base\\learner\end{tabular} & \begin{tabular}[b]{@{}c@{}}Baseline\\mean\end{tabular} & \begin{tabular}[b]{@{}c@{}}Final\\mean\end{tabular} & \begin{tabular}[b]{@{}c@{}}Final minus HGB,\\$F_1$ points [95\% CI]\end{tabular} & \begin{tabular}[b]{@{}c@{}}Better /\\equal / worse\end{tabular} \\
\midrule
\multicolumn{5}{@{}l}{\textit{(A) 11 new datasets}} \\
HGB & 0.855 & 0.864 & +0.000 [+0.000; +0.000] & 0 / 11 / 0 \\
XGB0 & 0.863 & 0.864 & -0.016 [-0.235; +0.210] & 5 / 0 / 6 \\
XGBES & 0.849 & 0.866 & +0.203 [-0.050; +0.521] & 4 / 1 / 6 \\
MLPES & 0.823 & 0.862 & -0.217 [-0.727; +0.367] & 3 / 0 / 8 \\
\addlinespace
\multicolumn{5}{@{}l}{\textit{(D) 39 supporting datasets}} \\
HGB & 0.819 & 0.827 & +0.000 [+0.000; +0.000] & 0 / 39 / 0 \\
XGB0 & 0.820 & 0.826 & -0.072 [-0.259; +0.112] & 16 / 0 / 23 \\
XGBES & 0.810 & 0.827 & +0.012 [-0.093; +0.130] & 15 / 2 / 22 \\
MLPES & 0.782 & 0.823 & -0.415 [-0.587; -0.250] & 9 / 0 / 30 \\
\bottomrule
\end{tabular}

\vspace{2pt}
\parbox{\linewidth}{\justifying\scriptsize Final mean: the mean weighted-$F_1$ on the locked reference after BRIDGE. better / equal / worse: datasets whose final weighted-$F_1$ is above, equal to, or below that of HGB. XGB0: XGBoost with 300 rounds, no early stopping. XGBES: XGBoost with early stopping. MLPES: a multilayer perceptron with early stopping.}
\end{table}

\begin{table}[!htb]
\centering
\scriptsize
\providecolor{bghead}{HTML}{DBEFF8}\providecolor{bgzebra}{HTML}{F2F9FC}\providecolor{bgpass}{HTML}{D7F0D2}\providecolor{bgfail}{HTML}{FCDCD7}
\setlength{\tabcolsep}{1.8pt}
\caption{Conditions of use: tuning treated as an update on (A), (B), and (D).}
\label{tab:tuning}
\begin{tabular}{@{}lrrrrrrc@{}}
\toprule
Learner & \begin{tabular}[b]{@{}c@{}}Naive acc./\\false (\%)\end{tabular} & \begin{tabular}[b]{@{}c@{}}BRIDGE acc./\\false (\%)\end{tabular} & \begin{tabular}[b]{@{}c@{}}Bound\\\% $\downarrow$\end{tabular} & \begin{tabular}[b]{@{}c@{}}LOCO\\\% $\downarrow$\end{tabular} & \begin{tabular}[b]{@{}c@{}}LOCO$_0$\\\% $\downarrow$\end{tabular} & \begin{tabular}[b]{@{}c@{}}Yield\\$\uparrow$\end{tabular} & Crit. \\
\midrule
\multicolumn{8}{@{}l}{\textit{(A) 11 new datasets}} \\
HGB & 526 / 97 (18.4) & 86 / 6 (7.0) & 23.86 & 43.42 & 10.49 & 7.8 & $\times$ \\
DT & 677 / 120 (17.7) & 171 / 6 (3.5) & 10.88 & 20.29 & 5.46 & 15.5 & $\times$ \\
SVC & 503 / 92 (18.3) & 163 / 6 (3.7) & 11.48 & 24.84 & 6.00 & 14.8 & $\times$ \\
KNN & 571 / 169 (29.6) & 39 / 4 (10.3) & 49.17 & 59.92 & 33.67 & 3.5 & $\times$ \\
\addlinespace
\multicolumn{8}{@{}l}{\textit{(B) Synthetic tasks}} \\
HGB & 912 / 129 (14.1) & 180 / 1 (0.6) & 4.44 & -- & -- & 12.9 & $\checkmark$ \\
DT & 816 / 146 (17.9) & 136 / 4 (2.9) & 11.21 & -- & -- & 9.7 & $\times$ \\
SVC & 679 / 65 (9.6) & 249 / 1 (0.4) & 3.10 & -- & -- & 17.8 & $\checkmark$ \\
KNN & 732 / 232 (31.7) & 96 / 4 (4.2) & 16.69 & -- & -- & 6.9 & $\times$ \\
\addlinespace
\multicolumn{8}{@{}l}{\textit{(D) 3 multiclass datasets}} \\
HGB & 91 / 27 (29.7) & 1 / 0 (0.0) & 100.00 & 100.00 & 100.00 & 0.6 & $\times$ \\
DT & 55 / 13 (23.6) & 7 / 0 (0.0) & 100.00 & 100.00 & 100.00 & 3.9 & $\times$ \\
SVC & 147 / 11 (7.5) & 73 / 0 (0.0) & 6.55 & 22.63 & 22.63 & 40.6 & $\times$ \\
KNN & 136 / 11 (8.1) & 38 / 1 (2.6) & 26.25 & 52.91 & 39.69 & 21.1 & $\times$ \\
\bottomrule
\end{tabular}

\vspace{2pt}
\parbox{\linewidth}{\justifying\scriptsize *Acc./false: accepted and false improvements (false \%). Crit.: criterion met ($\checkmark$) or not met ($\times$). Bound: KL-ratio bound on the real data, Proposition~\ref{prop:bound} bound on the synthetic population, where LOCO and LOCO$_0$ are not defined.}
\end{table}

\begin{table}[!htb]
\centering
\scriptsize
\providecolor{bghead}{HTML}{DBEFF8}\providecolor{bgzebra}{HTML}{F2F9FC}\providecolor{bgpass}{HTML}{D7F0D2}\providecolor{bgfail}{HTML}{FCDCD7}
\setlength{\tabcolsep}{2pt}
\caption{Clinical application (C): consequences for each $1{,}000$ observations.}
\label{tab:clinconseq}
\begin{tabular}{@{}lrrrrrrr@{}}
\toprule
 & & \multicolumn{3}{c}{Per $1{,}000$} & \multicolumn{3}{c}{Net benefit} \\
\cmidrule(lr){3-5}\cmidrule(l){6-8}
Task & Acc. & Flagged & TP & FP & Before & After & All \\
\midrule
BRFSS depression & 197 & 88$\to$291 & 65.3 & 137.8 & 0.0495 & 0.0803 & -0.0116 \\
NHANES diabetes & 200 & 37$\to$390 & 100.3 & 253.4 & 0.0133 & 0.0502 & -0.0467 \\
NHANES asthma & 1 & 0$\to$85 & 22.2 & 62.9 & 0.0000 & 0.0064 & -0.0620 \\
BRFSS asthma & 13 & 1$\to$124 & 31.0 & 92.2 & 0.0002 & 0.0082 & -0.0812 \\
ICU (WiDS) & 180 & 9$\to$112 & 35.8 & 67.3 & 0.0059 & 0.0248 & -0.1421 \\
Mammography & 10 & 10$\to$36 & 13.1 & 13.4 & 0.0071 & 0.0168 & -0.2096 \\
COVID-19 Mexico & 200 & 79$\to$258 & 61.6 & 117.4 & 0.0334 & 0.0656 & -0.0964 \\
ACTG 175 (HIV-1) & 198 & 40$\to$614 & 161.5 & 412.3 & 0.0124 & 0.0709 & 0.0542 \\
NHANES hypert. & 200 & 339$\to$661 & 123.7 & 198.8 & 0.1895 & 0.2635 & 0.2262 \\
All tasks & 1,199 & 98$\to$386 & 90.9 & 197.0 & 0.0505 & 0.0921 & -0.0031 \\
\bottomrule
\end{tabular}

\vspace{2pt}
\parbox{\linewidth}{\justifying\scriptsize *Acc.: accepted updates. Per $1{,}000$ observations: observations flagged before$\to$after the update (predicted positive), and extra true (TP) and false (FP) positives. Net benefit \eqref{eq:netbenefit} before the update, after the update, and under treating all (every observation flagged).}
\end{table}

\paragraph{Main results and hypothesis tests.} For every collection, Table~\ref{tab:summary} gives the counts of accepted updates and false improvements, the $95\%$ intervals, and the bounds behind the summary of Section~\ref{sec:exp-summary}. Table~\ref{tab:testsdetail} gives the estimate, the interval, and the $p$-value of every comparison in the hypothesis tests of Section~\ref{sec:exp-tests}. Every setting uses the current model HGB with the fixed configuration and the default candidate space ($173$ binary and $14+8K$ multiclass configurations; Section~\ref{sec:method-partition}). The false-improvement rate is pooled over all decisions of a setting, with one fresh seed per decision. The $95\%$ intervals are cluster-bootstrap percentile intervals over datasets, tasks, or clinical tasks ($10{,}000$ resamples). In (C) BRIDGE accepts no false improvement, so the upper bound is the informative quantity. Collection (E) combines $11$ binary and $3$ multiclass datasets. In the tests, T1 resamples datasets ((A), (D), (E)) or tasks ((B), (C)). The Holm adjustment runs over the comparators of a collection: six in (A) and (B), nine in (C) and (D), and one comparison without adjustment in (E). The intervals are not adjusted, so in (C) the interval of the corrected $t$-test excludes zero while the Holm-adjusted $p$ does not. The values $p<0.0012$ and $p<0.0002$ mark the resolution limit, where no resample fell on the other side of zero. T2 compares the margins $0.0025$, $0.0075$, and $0.01$ with $0.005$ within $\pm2$ percentage points, with a Holm adjustment over the $12$ comparisons. T3 is the one-sided hypergeometric test of a random draw, run post hoc. Its within-dataset form keeps the count of BRIDGE in each dataset or task.

\paragraph{Comparison of acceptance rules.} In one format for (A) to (E), Table~\ref{tab:rules} gives the counts, bounds, LOCO, yield, and updates above the margin of every rule, including the two holdouts. The table also gives the $95\%$ interval of the paired difference. Table~\ref{tab:selectivity} gives the random-selection control for question 6 (Section~\ref{sec:exp-tests}). For each rule, Table~\ref{tab:falseimp} gives the accepted updates that gain at most and more than the margin (post hoc). Among the updates that BRIDGE accepts, the reference gain of the final configuration reaches the cross-fit lower bound in $91\%$, $87\%$, $96\%$, $80\%$, and $82\%$ of cases in (A) to (E), and the median excess $\Delta_{\mathrm{ref}}-Q_{0.05}(\Delta^{*})$ is $0.016$, $0.013$, $0.011$, $0.007$, and $0.008$. The shortfall $\max\{Q_{0.05}(\Delta^{*})-\Delta_{\mathrm{ref}},0\}$ averages $0.0008$, $0.0008$, $0.0001$, $0.0011$, and $0.0007$ over all accepted updates, and $0.0089$, $0.0058$, $0.0026$, $0.0055$, and $0.0040$ over the updates that fall short; of those updates, $3$ of $25$, $8$ of $64$, $0$ of $49$, $6$ of $219$, and $2$ of $34$ are false improvements (post hoc, stored runs). The candidate selected in each fold was not stored, so the stability of that choice across folds cannot be assessed. Every rule judges the same decisions with the same deployed configuration, except the holdouts, which deploy the configuration selected by the holdout. In (E) the released program records the decisions of the naive rule and of BRIDGE, and of no other rule. LOCO leaves out one dataset, or one source in (C). LOCO is not defined on (B), where every task is independent, and is not computed for the holdouts in (A). The paired intervals use the cluster bootstrap of T1. The random-selection control is a post hoc, exploratory reanalysis of the stored runs. In every setting the updates BRIDGE accepts are a subset of those the naive rule accepts. The pooled draw takes the number of accepted updates of BRIDGE at random from the accepted updates of the naive rule. This draw gives the exact hypergeometric $p$ of a count of false improvements no larger than the count of BRIDGE. It compares BRIDGE with a random choice of the same size, not a probability about future tasks. The within-dataset draw keeps the count of BRIDGE in each dataset or clinical task, so that the test concerns the choice among the trials of one dataset. In (E) the $p$-value comes from $100{,}000$ draws.

\paragraph{Candidate-space size.} Table~\ref{tab:kspace} gives the results of the three rules at $173$, $1{,}825$, and $11{,}249$ configurations on (A) to (E). The table provides the data behind Fig.~\ref{fig:kspace} (Section~\ref{sec:exp-ablation}). The grids keep the HGB, DT, and SVC ensemble members (Appendix~\ref{app:sens}). The panel of (E) covers the binary datasets of the OpenML collection in run MINI-52, pre-registered after the runs of (A) to (D). The multiclass datasets of (D) and (E) are left out, because no equally dense multiclass grid is pre-registered. At $173$ configurations the rows of (E) reproduce the stored decisions of the released program in all $1{,}100$ trials. For (A) to (D), the accepted counts, LOCO, and yield at $173$ configurations are in Table~\ref{tab:rules}.

\paragraph{Conditions of use.} Table~\ref{tab:scope} summarizes the status of every condition of use tested in Section~\ref{sec:exp-base}. The tables that follow give the evidence. These tables cover the metrics, training sizes, and multiclass tasks (Table~\ref{tab:general}), imbalanced data (Table~\ref{tab:imbalance}), and the base learners on $\mathcal{Q}$ (Table~\ref{tab:basefull}). The evidence also covers the ensemble members and base learners on all five collections (Table~\ref{tab:members}), the final model quality (Table~\ref{tab:absolutequality}), and the tuning study (Table~\ref{tab:tuning}). Across metrics, training sizes, and multiclass tasks, the criterion is met everywhere except with feature selection and SVC, where LOCO exceeds $10\%$. The imbalance study uses the five large binary datasets of (A) whose locked reference keeps at least $800$ minority cases at every level (Census-Income-KDD, Homesite, KDD98, LICD, and Santander2016). Per dataset and level, the rows are subsampled once with a fixed seed to the stated prevalence before the locked outer split, so the reference has the same prevalence. Each cell has $100$ trials per dataset ($500$ decisions), with the trial code, gate, and candidate space unchanged (pre-registration MINI-41). On $\mathcal{Q}$, the base-learner study changes one setting of the project-default XGBoost ($300$ rounds without early stopping) at a time on $1{,}400$ fresh synthetic tasks. The study also adds a regularized variant with $4$ settings changed and $30$ or $100$ rounds without early stopping. The same study adds an MLP with hidden layers $(64,32)$, with and without early stopping, as a pre-registered prediction test. The POOL log-loss skill is $1$ minus the log-loss over that of a prevalence-only forecast, which is computable without reference data. In the member and learner study every design runs on the same trials, seeds, splits, and bags as the design of the paper, and no learner is tuned. The runs are MINI-45, MINI-46, and MINI-50 for (A) and (D), and MINI-53 for (B), (C), and (E). The bounds of the cells that miss the criterion on (E) are in Table~\ref{tab:citedruns}, and the bound, LOCO, LOCO$_0$, and yield of every cell are recorded in the code repository. The final model quality uses the first $20$ trials of every dataset, with paired cluster-bootstrap intervals (MINI-49). In the tuning study each learner (HGB, DT, SVC, KNN) is tuned with Optuna on SEARCH, and the tuned model is the update against the same learner with its fixed settings. The study covers the $11$ new datasets ($100$ trials each), the synthetic population (task seeds $2{,}900$ to $4{,}299$, one trial each), and three multiclass datasets (Fashion-MNIST, USPS, DryBean; $60$ trials each). The other two multiclass datasets were not run; this deviation was recorded before the results were seen (pre-registrations MINI-39, MINI-40).

\paragraph{Clinical consequences.} For each clinical task, Table~\ref{tab:clinconseq} translates the accepted updates into flagged cases, extra true and false positives per $1{,}000$, and net benefit before and after the update (Section~\ref{sec:exp-clinical}). The values are exploratory means over the accepted updates of each task (MINI-30, $200$ trials per task). Treat all is negative when the prevalence is below $p_t=0.2$, because flagging every observation then costs more in false positives than it gains in true positives. In every task the mean net benefit after the update exceeds both treat none ($0$) and treat all. In NHANES asthma, ACTG 175, and NHANES hypertension the model first reaches this level after the update.

\FloatBarrier

\section*{Data Availability Statement}
All datasets are public (OpenML, the UCI Machine Learning Repository, scikit-learn, CDC NHANES and BRFSS, and the Secretar\'ia de Salud COVID-19 release of Mexico); the WiDS 2020 cohort is available through its Kaggle competition, and the repository holds only its loader. The code, the per-trial outputs, and the package \texttt{bridge-gate} are available at \url{https://github.com/DwiSunaryono/BRIDGE-ML} and archived on Zenodo (\href{https://doi.org/10.5281/zenodo.23523984}{doi:10.5281/zenodo.23523984}); one script rebuilds the reported tables from the stored outputs without training a model.

\bibliographystyle{IEEEtran}
\bibliography{refs}

\end{document}